\documentclass{article} % For LaTeX2e
\usepackage{iclr2027_conference,times}
\iclrfinalcopy
\makeatletter
\def\@oddhead{%
  \vbox to 0pt{%
    \hbox to\textwidth{Preprint. Under review.\hfil}%
    \vskip 1pt%
    \hbox to\textwidth{\rule{\textwidth}{0.4pt}}%
    \vss
  }%
}
\def\@evenhead{\@oddhead}
\makeatother
\usepackage{amsmath,amsfonts,bm}

\def\eqref#1{equation~\ref{#1}}
\def\1{\bm{1}}

\DeclareMathAlphabet{\mathsfit}{\encodingdefault}{\sfdefault}{m}{sl}
\SetMathAlphabet{\mathsfit}{bold}{\encodingdefault}{\sfdefault}{bx}{n}

\newcommand{\R}{\mathbb{R}}

\DeclareMathOperator*{\argmax}{arg\,max}

\DeclareMathOperator{\sign}{sign}

\usepackage{url}
\usepackage{algorithm}
\usepackage{algpseudocode}
\usepackage{enumitem}

\usepackage{amsmath}
\usepackage{amssymb}
\usepackage{amsfonts}
\usepackage{mathtools}
\usepackage{amsthm}

\newtheorem{theorem}{Theorem}
\newtheorem{proposition}{Proposition}
\newtheorem{corollary}{Corollary}
\newtheorem{definition}{definition}

\usepackage{booktabs}
\usepackage{array}
\usepackage{multirow}

\usepackage{graphicx}
\usepackage{xcolor}
\usepackage{subcaption}

\usepackage{nicefrac}
\usepackage{microtype}
\usepackage{wrapfig}
\usepackage{booktabs}
\usepackage{multirow}
\usepackage{siunitx}
\usepackage{xcolor}
\usepackage{colortbl}
\usepackage{pifont}

\usepackage{etoc}
\usepackage[colorlinks=true,linkcolor=mauvetaupe, urlcolor=deepNavy, citecolor=mauvetaupe]{hyperref}

\definecolor{portlandorange}{rgb}{1.0, 0.35, 0.21}
\definecolor{persian}{rgb}{0.0, 0.65, 0.58}
\definecolor{forestgreen}{rgb}{0.13, 0.55, 0.13}
\definecolor{richlavender}{rgb}{0.67, 0.38, 0.8}
\definecolor{indigo}{rgb}{0.0, 0.25, 0.42}
\definecolor{huntergreen}{rgb}{0.21, 0.37, 0.23}
\definecolor{liver}{rgb}{0.33, 0.29, 0.31}
\definecolor{outerspace}{rgb}{0.25, 0.29, 0.3}
\definecolor{vectororange}{HTML}{FFAC1C}
\definecolor{vectorboox}{HTML}{FFE6B8}
\definecolor{vectorborder}{HTML}{FFAC1C}
\definecolor{smokyblack}{rgb}{0.06, 0.05, 0.03}

\definecolor{mauvetaupe}{rgb}{0.25, 0.41, 0.88}
\definecolor{deepNavy}{rgb}{1.0, 0.0, 0.56}
\definecolor{hiddenbox}{RGB}{232,236,241}
\definecolor{hiddenborder}{RGB}{92,101,116}

\definecolor{tiedbox}{RGB}{218,246,235}
\definecolor{tiedborder}{RGB}{0,150,120}

\definecolor{vectorbox}{RGB}{253,230,242}
\definecolor{vectorborder}{RGB}{205,78,145}

\newcommand{\hlpara}[2]{%
  \noindent\colorbox{#1}{\textbf{#2}}%
}

\usepackage{tikz}
\usepackage{pgfplots}
\pgfplotsset{compat=1.18}
\usetikzlibrary{arrows.meta, calc, positioning}
\newcommand{\ip}[2]{\langle #1,#2\rangle}
\newcommand{\norm}[1]{\left\lVert #1\right\rVert}
\newcommand{\tie}{N_{\mathrm{tie},c}}
\newcommand{\alias}{N_{\mathrm{alias}}}
\newcommand{\LC}{L_c}

\DeclareMathOperator{\RMS}{RMS}
\DeclareMathOperator{\supp}{supp}
\DeclareMathOperator{\PAPR}{PAPR}
\title{AF-Muon: An AdamW-Free Muon Optimizer for Tied-Embedding Models}
\author{
Arash Lagzian$^{1*}$,
Paniz Halvachi$^{2}$,
Junming Zhang$^{3}$,
Zhouhan Lin$^{3}$,
Dianbo Liu$^{1*}$\\[0.5em]
$^{1}$National University of Singapore, Singapore\\
$^{2}$Independent Researcher\\
$^{3}$LUMIA Lab, School of Artificial Intelligence, Shanghai Jiao Tong University, Shanghai, China\\[0.5em]
$^{*}$Corresponding authors:
\texttt{arashlagzian94@nus.edu.sg},
\texttt{dianbo@nus.edu.sg}
}

\begin{document}

\etocdepthtag.toc{main}
\maketitle

% \begin{center}
% \small
% Correspondence: \texttt{arashlagzian94@nus.edu.sg}, \texttt{dianbo@nus.edu.sg}
% \end{center}

\begin{abstract}
Muon improves large-scale training by applying a spectral-norm steepest-descent update to matrix parameters, but practical models also contain parameter blocks that do not fit dense-matrix geometry. One important case is the tied vocabulary table, which appears in language models and other token generators and can receive multiple structurally different gradient sources, from sparse input lookups to dense output-classifier updates. 
In the reference recipe these blocks are handed to an auxiliary AdamW optimizer, which restores second-moment state and updates the aliased table as a generic tensor.
We propose AF-Muon, an AdamW-free extension of Muon that keeps the Muon matrix update for hidden weight matrices while using a support-aware finite-cap linear minimization oracle for tied vocabulary tables and an RMS-normalized update for one-dimensional auxiliary parameters. AF-Muon therefore trains every parameter class with a single first-moment buffer and no second-moment state, saving around $20\%$ optimizer-state memory relative to Hybrid Muon in our benchmark. Across nine tied-token settings---decoder-only language models from 124M to 1B parameters, a fully shared T5-style encoder--decoder, and ImageGPT-style
image-token, protein, and sparse-MoE variants, spanning text, image, and protein-sequence data---AF-Muon improves mean validation loss and perplexity over both Hybrid Muon and a SCION-style Sign endpoint. Long-horizon runs and hyperparameter sensitivity studies confirm the gain is robust, and identical-momentum diagnostics attribute it to the finite cap, which preserves more within-row magnitude than Sign while bounding the coordinate concentration of row-RMS. These results identify tied vocabulary tables as a distinct optimizer geometry and yield a robust AdamW-free Muon variant across models, modalities, and architectures, with about $1\%$ step-time overhead in matched training.
% Anonymous code is provided in the supplementary material and at
% \url{https://github.com/lighstar0188/afmoun}.

% \textcolor{red}{by Dianbo Liu , the abstract pitch: 1) background: importance of Muon optimizer 2) non-matrix parameters, other optimizers are often used together with Muon. 3) (gap) however, in parameters with more than one roles , such as tied embedding, existing method cannot handle well. 4) we propose our solution to address this and briefly mentioned results  }

\end{abstract}

\noindent\textbf{Github:}
\url{https://github.com/arashlagzian/afmoun}

\section{Introduction}

Adam and AdamW remain standard for large-model training because their
coordinate-wise first- and second-moment adaptation is robust across
heterogeneous parameter types~\citep{kingma2015adam,loshchilov2019decoupled}.
Muon replaces this geometry for selected two-dimensional hidden weights:
it maps momentum through a short Newton--Schulz iteration approximating the
polar factor, yielding a spectral-norm steepest-descent direction in the
converged limit~\citep{jordan2024muon}. This matrix-aware update is effective
at language-model scale, but the reference Muon recipe is still hybrid:
embeddings, output heads, normalization parameters, and biases remain on an
auxiliary AdamW branch~\citep{liu2025muonscalable}. Thus Muon removes
coordinate-wise second-moment adaptation from only part of the model, including
neither the embedding nor output blocks that dominate token-model parameter
count. This leaves a basic question: how should the remaining parameter roles
be updated without AdamW while respecting their functional geometry?

% \vspace{-5pt}

Norm-based optimizer design emphasizes that tensor shape alone does not
determine an appropriate update rule. Parameters with identical array shapes may implement different functions and therefore require different
geometries~\citep{bernstein2024oldnorm,large2024scalable,bernstein2025modular}. Tied vocabulary tables provide a particularly sharp example. Weight tying uses one stored matrix \(E\in\mathbb{R}^{V\times d}\) for both token lookup and output classification~\citep{press2017using,inan2017tying}. 
% The lookup path
% contributes row-sparse gradients associated with tokens appearing in the input, whereas the output-softmax path generally contributes a dense gradient over vocabulary rows~\citep{press2017using,lopardo2026weight}. In a fully shared encoder--decoder architecture---used in recent lightweight models such as T5Gemma~2~\citep{t5gemma2025}---the same table can additionally receive separate encoder- and decoder-lookup contributions. Autograd accumulates these structurally different signals into one merged gradient on a single aliased tensor before the optimizer acts. The tied table is therefore neither an ordinary hidden linear map nor a pure input embedding or output classifier considered in isolation.
The lookup path gives row-sparse gradients for observed input tokens, whereas the output-softmax path usually gives dense gradients over vocabulary rows~\citep{press2017using,lopardo2026weight}. In fully shared encoder--decoder models, including recent models such as T5Gemma~2~\citep{t5gemma2025}, same table can receive separate encoder- and decoder-lookup contributions. Autograd merges these structurally different signals on one aliased tensor before the optimizer acts, so the tied table is neither an ordinary hidden linear map nor an isolated embedding or classifier.

\begin{figure}[t]
   \centering
   \includegraphics[
     width=1.0\linewidth,
     trim=0 8 0 5,
     clip
   ]{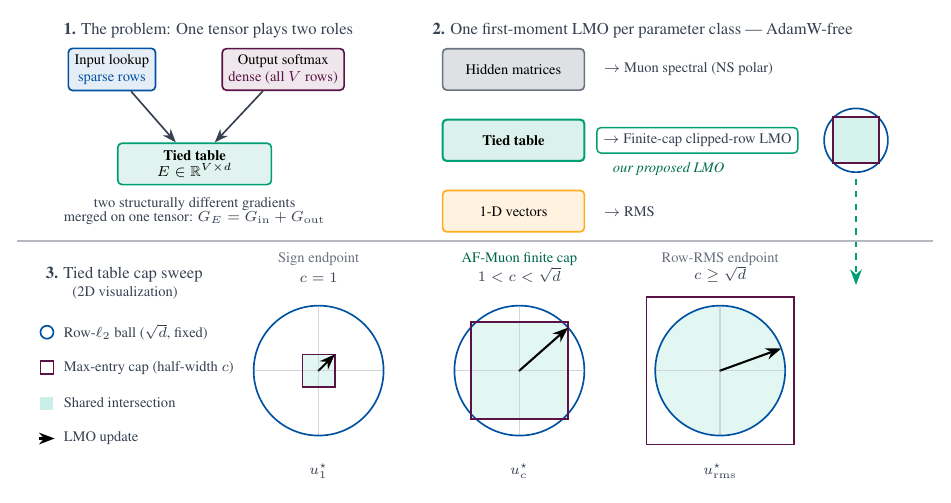}
   \vspace{-1.5em}
   \caption{
\textbf{AF-Muon tied-table geometry.}
Top: a hard-tied vocabulary table serves both sparse input lookup and dense
output-softmax roles, so their gradients merge on one tensor. AF-Muon assigns
one first-moment LMO per parameter class: Muon for hidden matrices, a finite-cap
clipped-row LMO for the tied table, and RMS normalization for 1-D vectors.
Bottom: the tied-table feasible set is the intersection of a fixed row-$\ell_2$
ball and a coordinate cap box. Varying $c$ interpolates between the Sign endpoint
($c=1$), the finite-cap AF-Muon update ($1<c<\sqrt d$), and row-RMS
($c\ge\sqrt d$).
}
   \vspace{-1.0em}
   \label{fig:overview}
\end{figure}
\vspace{-7pt}
Existing choices capture only parts of this structure. Hybrid Muon assigns the tied table to AdamW, retaining coordinate-wise second-moment state. SCION gives
the closest role-aware LMO baseline: for untied one-hot inputs it recommends ColNorm--Spectral--Sign, where ColNorm corresponds to per-token row-RMS normalization in our \(E\in\mathbb{R}^{V\times d}\) orientation; under input--output weight sharing it recommends Sign--Spectral--Sign, so the aliased table receives a Sign LMO~\citep{pethick2025scion}. We use this rule as our
SCION-style endpoint baseline. Sign controls coordinate concentration but gives weak and dominant momentum coordinates equal magnitude; row-RMS preserves within-row magnitude information but can concentrate update energy in a few coordinates. Rescaling either endpoint changes only overall magnitude, leaving the geometric trade-off unresolved.

Figure~\ref{fig:overview} summarizes the tied-table problem, the AF-Muon parameter-role assignment, and the finite-cap geometry between the Sign endpoint and the uncapped row-RMS limit. We introduce \emph{AF-Muon}, a
role-complete AdamW-free extension of Muon for models with tied token tables.
AF-Muon keeps Muon's spectral update for hidden matrices, uses RMS-normalized updates for one-dimensional auxiliary tensors, and applies a new support-aware finite-cap LMO to the merged momentum of the tied vocabulary table. The tied-table oracle intersects a row-wise RMS constraint with an entrywise cap, connecting Sign at the most restrictive cap to row-RMS when the cap is removed. Intermediate caps preserve more within-row magnitude information than Sign while limiting the coordinate concentration allowed by row-RMS. The cap \(c\) controls normalized direction geometry, while a separate scale \(s\) controls update magnitude. Thus every parameter class considered is trained with first-moment state only, with no parameter delegated to AdamW.

% We evaluate AF-Muon across decoder-only and encoder--decoder models, multiple parameter-sharing patterns, and several token domains under matched training budgets. Beyond the decoder-only language models, the fully shared T5-style setting places encoder lookup, decoder lookup, and output classification in one vocabulary tensor, giving a direct test of the aliased-gradient motivation; the evaluation further spans image-token modeling on ImageNet-32, protein language modeling, and a sparse Mixture-of-Experts architecture. In every setting AF-Muon attains lower mean validation loss than both Hybrid Muon and the SCION-style Sign endpoint; notably, the Sign endpoint alone can fall \emph{below} Hybrid Muon, whereas the finite cap does not, indicating that the cap---not merely removing AdamW---drives the gain. Long-horizon runs and sensitivity studies over batch size, learning-rate allocation, cap $c$, and scale $s$ show the advantage persists across the tested horizons and hyperparameter ranges, and diagnostics computed on identical momentum states show that finite clipping retains most of the row-RMS linear objective while sharply reducing its coordinate concentration and preserving substantially more within-row information than Sign. Our contributions are listed as follows:
We evaluate AF-Muon under matched budgets across nine tied-token settings:
decoder-only language models, a fully shared encoder--decoder model,
image-token modeling, protein language modeling, and sparse MoE. AF-Muon
achieves the lowest mean validation loss in every setting, outperforming both
Hybrid Muon and the SCION-style Sign endpoint. Long-horizon runs, sensitivity
studies, and matched-momentum diagnostics further support that the finite cap,
rather than merely removing AdamW, drives the gain. Our contributions are:

\begin{itemize}
    \item We identify and formulate optimization of a tied vocabulary table as an aliased multi-role problem: one stored tensor serves sparse lookup and dense classifier roles, so its merged gradient combines structurally different contributions.
    \item We derive an exact, support-aware finite-cap LMO that connects Sign and row-RMS while separating directional geometry, controlled by \(c\), from update magnitude, controlled by \(s\). Combined with Muon for hidden matrices and RMS-normalized updates for vectors, this yields an AdamW-free, second-moment-free optimizer that trains every parameter class with a single first-moment buffer, reducing optimizer-state memory relative to Hybrid Muon.
    \item We provide controlled evidence across nine tied-token settings, spanning
    124M--1B-parameter decoder-only models plus encoder--decoder, image-token,
    protein-sequence, and sparse-MoE variants, trained on text, image, and protein
    data with three paired seeds, up to $\sim\!10\text{B}$ tokens, and tied-geometry sensitivity studies.
\end{itemize}

% \newpage
\section{Method}
\label{sec:method}
Muon~\citep{jordan2024muon} approximately projects momentum updates for selected hidden matrices onto their polar factors, which, in the exact limit, corresponds to steepest descent over a spectral-norm ball. Common Hybrid-Muon implementations keep this matrix update for hidden weights but optimize embeddings, output heads, normalization parameters, and biases with an auxiliary AdamW
branch~\citep{loshchilov2019decoupled}. This works well in practice but mixes optimizer
geometries and carries AdamW second-moment state for the auxiliary branch.

AF-Muon is a single modular optimizer that keeps Muon's hidden-matrix update
while replacing the auxiliary AdamW branch with first-moment LMO updates
matched to each parameter's role. In the tied-head RoPE (Rotary Position Embedding) architectures considered here, learned tensors are partitioned into hidden matrices, the
shared tied token-embedding/LM-head table, and one-dimensional auxiliary
tensors. Hidden matrices use the spectral LMO, vector parameters use an RMS
LMO, and the tied table receives our new finite-cap LMO: since one stored
table is aliased between token-lookup and output-classifier roles, we propose
an intersection of a max-row-$\ell_2$ ball and a max-entry ball for its merged
update.

% ----------------------------------------------------------------------
\subsection{Steepest descent under a norm}
\label{subsec:lmo}

Let a parameter block $W\in\mathcal{W}$ carry a norm $\|\cdot\|$ with dual norm
$\|Z\|_\star := \max_{\|U\|\le 1}\langle Z,U\rangle$, where
$\langle A,B\rangle = \mathrm{tr}(A^\top B)$. 
For oracle input $Z$, a
normalized steepest-descent direction, or \emph{linear minimization oracle}
(LMO), is
\begin{equation}
  \label{eq:lmo}
  D_{\|\cdot\|}(Z)\in\arg\max_{\|U\|\le 1}\langle Z,U\rangle .
\end{equation}
For $Z\neq 0$, $\langle Z,-\eta D_{\|\cdot\|}(Z)\rangle=-\eta\|Z\|_\star<0$, so $-D_{\|\cdot\|}(Z)$ maximizes linearized decrease under $\|\cdot\|$.
The tied table and the one-dimensional auxiliary tensors use the unnormalized momentum buffer $B_t=\mu B_{t-1}+G_t$ as their oracle input. Hidden matrices use Muon's Nesterov-style
construction $M_t=\mu M_{t-1}+(1-\mu)G_t$, $Z_t=(1-\mu)G_t+\mu M_t$, with $Z_t$
passed to the Newton--Schulz map. The identity therefore refers to the oracle input, not to each minibatch gradient $G_t$.

% ----------------------------------------------------------------------
\subsection{Parameter classes}
\label{subsec:classes}

We partition stored parameters by functional role.
%Throughout, $\operatorname{RMS}(x)=\|x\|_2/\sqrt d$ for $x\in\mathbb{R}^d$.
\vspace{-0.5em}

% \paragraph{Hidden matrices.}
\hlpara{hiddenbox}{Hidden matrices.}\quad
For hidden matrices $W\in\mathbb{R}^{m\times n}$, we use Muon's spectral LMO.
If $Z_t=\Phi\Sigma\Psi^\top$, the exact spectral-norm LMO is the polar factor
$\Phi\Psi^\top$ (\cite{pethick2025scion}), approximated by $K$-step Newton--Schulz iteration, $\mathrm{NS}_K(Z_t)$. Following
Muon, we apply the rectangular shape factor $\sqrt{\max(1,m/n)}$; this is the
only branch to which we apply matrix \textbf{weight decay} in our experiments.

% \paragraph{Vector parameters.}
\hlpara{vectorboox}{Vector parameters.}\quad
% For one-dimensional parameters (normalization gains, QK-norm gains, biases),
% the ideal RMS LMO is 
% \vspace{-0.5em}
% \[
% D_{\mathrm{RMS}}(b)
% =
% \begin{cases}
% b/\operatorname{RMS}(b), & b\neq 0,\\
% 0, & b=0.
% \end{cases} \qquad \text{where}\quad \operatorname{RMS}(b)=\|b\|_2/\sqrt d \quad\text{for} \; b\in\mathbb{R}^d
% \]
% \vspace{-0.5em}
% \\
For one-dimensional parameters (normalization gains, QK-norm gains, biases),
the ideal RMS LMO is
$D_{\mathrm{RMS}}(b)=b/\operatorname{RMS}(b)$ for $b\neq0$ and $0$ otherwise,
where $\operatorname{RMS}(b)=\|b\|_2/\sqrt d$ for $b\in\mathbb{R}^d$.
% $D_{\mathrm{RMS}}(b)=b/\operatorname{RMS}(b)$ for
% $b\neq0$ (and $0$ otherwise);
We implement the stabilized form
$\widehat D_{\mathrm{RMS}}(b)=b/\max\{\operatorname{RMS}(b),\varepsilon\}$,
which equals $D_{\mathrm{RMS}}$ whenever $\operatorname{RMS}(b)\ge\varepsilon$.
This removes AdamW from these parameters while preserving a native shape
update.

% \paragraph{Tied table (aliased).}
\hlpara{tiedbox}{Tied table (aliased).}\quad
\label{subsec:tied-lmo}
The tied token-embedding/LM-head is one stored tensor $E\in\mathbb{R}^{V\times
d}$, where $V$ is the vocabulary size
and $d$ is the hidden width. It is used in two roles, $i\mapsto E_{i,:}$ and $h\mapsto Eh\in\mathbb{R}^V$;
autograd accumulates a merged gradient
$G^E=G^E_{\mathrm{in}}+G^E_{\mathrm{out}}$ on the shared tensor. Because the
tied table receives a dense output-role gradient through the softmax, its LMO
is applied to all rows. Unlike an untied head, the merged update should
\emph{not} be row-centered: shifting every row by $a^\top$ is a
logit-preserving symmetry only for the output role, but under tying, it also
shifts every token embedding, so the head-only quotient does not apply here (Appendix \ref{app:sec-hardtiedtables})
%proposition ~\ref{app:pro_remove-shift-redundancy});
; AF-Muon therefore applies its LMO to the
merged, unprojected tied momentum.

More generally, if one stored tensor $W$ is used in $R$ roles through linear
maps $A_r:\mathcal{W}\to\mathcal{X}_r$ with role norm $N_r$ and radius
$\rho_r$, pulling the product geometry back to $W$ gives the \emph{alias norm}
\begin{equation}
\label{eq:alias-norm}
N_{\mathrm{alias}}(U):=\max_{1\le r\le R}\frac{N_r(A_rU)}{\rho_r},
\end{equation}
a norm whenever $\bigcap_r\ker A_r=\{0\}$ (Appendix~\ref{app:sec_aliaspulllback}),
whose unit ball is the intersection of the pulled-back role balls and whose
LMO satisfies every role constraint at once. The tied table is the $R=2$ case
with $A_1=A_2=\mathrm{Id}$: a row-wise radius, reflecting lookup-row control,
intersected with a Sign-compatible max-entry cap on the output side. This
intersection geometry is a modeling choice, not a consequence of tying alone;
we evaluate it against both endpoints.

For $U\in\mathbb{R}^{V\times d}$ with rows $u_i$, let
$\|U\|_{\mathrm{row}}:=\max_i\|u_i\|_2/\sqrt d$ and
$\|U\|_{\max}:=\max_{i,j}|U_{ij}|$. For cap $c\ge1$, define the tied-table
norm
\begin{equation}
\label{eq:tied-norm}
\|U\|_{\mathrm{tie},c}:=\max\left\{\|U\|_{\mathrm{row}},\ \|U\|_{\max}/c\right\},
\end{equation}
whose unit ball caps every row's $\ell_2$ norm at $\sqrt d$ and every entry at
$c$. We write $\mathcal \mathcal{C}_c=\{U:\|U\|_{\mathrm{tie},c}\le1\}$ for its unit ball. This feasible set is a Cartesian product across rows, so its LMO
decomposes row-by-row. For a row momentum $b\in\mathbb{R}^d$ with support size
$k(b)$, choose $\tau_c(b)\ge0$ solving
\begin{equation}
\label{eq:tau-general}
\textstyle\sum_j\min\{c^2,\tau_c(b)^2b_j^2\}=\min\{d,c^2k(b)\},
\end{equation}
and define the \emph{clipped-row LMO}
\begin{equation}
\label{eq:clipped-row-lmo}
\mathcal{L}_c(b)_j=\operatorname{sign}(b_j)\min\{c,\tau_c(b)|b_j|\}.
\end{equation}

\begin{proposition}[Clipped-row LMO]
\label{prop:clipped-row-lmo}
For $c\ge1$, $\mathcal{L}_c(b)\in\arg\max\{\langle b,u\rangle:\|u\|_2\le\sqrt
d,\ \|u\|_\infty\le c\}$, with $\|\mathcal{L}_c(b)\|_2^2=\min\{d,c^2k(b)\}$.
The endpoints are $\mathcal{L}_1(b)=\operatorname{sign}(b)$ and, for
$c\ge\sqrt d$, $\mathcal{L}_c(b)=b/\operatorname{RMS}(b)$ (row-RMS). We write $c=\infty$ as shorthand for omitting the coordinate cap.
%We write $c=\infty$ for this unconstrained limit.
\end{proposition}
Applying $\mathcal{L}_c$ row-wise to the table momentum $B$ gives a global LMO
for $\|\cdot\|_{\mathrm{tie},c}$ (proof in Appendix~\ref{app:sec_exactClippedRow});
% for dense rows, $c^2k(b)\ge d$ holds under full-softmax training and the selected direction has $\ell_2$ norm $\sqrt d$. 
under full-softmax training, momentum rows typically become dense enough that $c^2k(b)\ge d$, in which case the selected direction has $\ell_2$ norm $\sqrt d$. The cap $c$ thus controls the
normalized \emph{shape} of the update within each row, interpolating between
Sign and row-RMS, while a separate scalar $s$ (equivalently, the scaled
feasible set $s\,\mathcal{C}_c$) controls per-row \emph{magnitude} without
changing that shape.

The finite cap trades symmetry for concentration control: $\mathcal{L}_c$ is
equivariant under row and signed-feature permutations,
% (Appendix~\ref{app:equivariance})
but, unlike row-RMS, not under general
right rotations $R\in O(d)$, unavoidably so, since any right-orthogonally equivariant, fixed-energy map that is positively aligned with its input must be exactly row-RMS, with scale invariance following automatically (Appendix~\ref{app:sec_symmetry-nogo})
% since any row-separable, scale-invariant, fixed-energy map that is fully rotation-equivariant must be exactly row-RMS (Appendix~\ref{app:sec_symmetry-nogo}).
Concretely, for dense rows
($c^2k(b)\ge d$) the peak-to-average ratio
$\operatorname{PAPR}(u):=\|u\|_\infty^2/(\|u\|_2^2/d)$ obeys
$\operatorname{PAPR}(\mathcal{L}_c(b))\le c^2$, and optimal linear
objective is nondecreasing in $c$ from Sign to row-RMS
(Appendix~\ref{app:sec_geo-rete-concen}): finite-cap family traces an
alignment--concentration frontier between two endpoints.

As with the other non-hidden-matrix blocks, the tied table accumulates $B_t^E=\mu B_{t-1}^E+G_t^E$ from the merged gradient $G_t^E$. For the tied table, $O_{t,i}=\mathcal{L}_c(\mathrm{row}_i(B^E_t))$ and $E_t=E_{t-1}-\eta_{\mathrm{tie},t}O_t$.
% For the tied table, $O_i=\mathcal{L}_c(\mathrm{row}_i(B^E_t))$ and $E_t=E_{t-1}-\eta_{\mathrm{tie}}O$. 
We implement the scalar tied multiplier as a width-scaled learning rate: with shared schedule multiplier $a_t$,
$\eta_{\mathrm M,t}=\eta_{\mathrm M}^{\max}a_t$, and
\begin{equation}
\label{eq:tied-lr}
\eta_{\mathrm{tie},t}=\eta_{\mathrm M,t}\left(\rho_{\mathrm{out}}/\rho_{\mathrm{hid}}\right)s/d.
\end{equation}
% In our experiments $\eta_{\mathrm M}^{\max}=0.02$ and
% $\rho_{\mathrm{out}}/\rho_{\mathrm{hid}}=3000/50=60$, giving
% $\eta_{\mathrm{tie},t}=1.2\,a_ts/d$; our main setting is $c=3$, $s=0.5$, which
% for all evaluated widths ($d>9$) lies strictly between Sign and row-RMS.

% \newpage
\begin{wrapfigure}[37]{r}{0.50\textwidth}
\vspace{-0.5em}
\fcolorbox{mauvetaupe}{white}{%
\begin{minipage}{\dimexpr\linewidth-2\fboxsep-2\fboxrule\relax}
\vspace{0.35em}

\refstepcounter{algorithm}
{\scriptsize
\noindent\textbf{\textcolor{mauvetaupe}{Algorithm~\thealgorithm:} AF-Muon with tied-table clipped LMO}\par
}
\label{alg:afmuon}

\vspace{0.2em}
{\color{mauvetaupe}\hrule}
\vspace{0.35em}

% \vspace{0.35em}
\scriptsize
\setlength{\tabcolsep}{0pt}
\algrenewcommand\alglinenumber[1]{\textcolor{black}{\tiny #1:}}
%----------------------
\begin{algorithmic}[1]
\Require LR schedule $a_t$, peak Muon LR $\eta_{\mathrm M}^{\max}$, peak vector LR $\eta_{\mathrm{vec}}^{\max}$, momentum $\mu$, matrix weight decay $\lambda_{\mathrm M}$, tied update scale $s$, tied cap $c$, output radius $\rho_{\mathrm{out}}$, hidden radius $\rho_{\mathrm{hid}}$, Newton--Schulz steps $K$, constant $\varepsilon$
\Statex \vspace{-0.5em}
\State Let $\mathcal{P}$ be the set of unique stored parameter tensors, deduplicated by object identity
\State Initialize hidden-matrix momentum buffers $M_0^W\leftarrow 0$
\State Initialize auxiliary momentum buffers $B_0^W\leftarrow 0$
\Statex \vspace{-0.5em}
\For{step $t=1,\dots,T$}
\Statex\vspace{-0.2cm}
  \State $\eta_{\mathrm M,t}\leftarrow \eta_{\mathrm M}^{\max}a_t$;\quad
         $\eta_{\mathrm{vec},t}\leftarrow \eta_{\mathrm{vec}}^{\max}a_t$
  \State Compute gradients $G_t^W\leftarrow\nabla_W\mathcal{L}_t(\theta_{t-1})$
\Statex\vspace{-0.2cm}
  \For{each $W\in\mathcal{P}$}
\Statex\vspace{-0.1cm}
    \If{$G_t^W$ is \texttt{None}}
      \State continue
    \EndIf
\Statex\vspace{-0.1cm}
    \If{$W$ is a hidden matrix in $\mathbb{R}^{m\times n}$}
    \Statex\vspace{-0.2cm}
      % \Comment{$m=$ output dimension, $n=$ input dimension}
      \State $M_t^{W}\leftarrow \mu M_{t-1}^{W}+(1-\mu)G_t^{W}$
      \State $Z_t^{W}\leftarrow (1-\mu)G_t^{W}+\mu M_t^{W}$
      \State $O_t^{W}\leftarrow \mathrm{NS}_K(Z_t^{W})\cdot\sqrt{\max(1,m/n)}$
      \State $W_t\leftarrow (1-\eta_{\mathrm M,t}\lambda_{\mathrm M})W_{t-1}
             -\eta_{\mathrm M,t}O_t^{W}$
\Statex\vspace{-0.1cm}
    \ElsIf{$W\in\mathbb{R}^{V\times d}$ is the tied token embedding/\\ \hspace*{2.9em} LM-head table}
      % \Comment{merged-gradient clipped-row LMO}
      \Statex\vspace{-0.2cm}
      \State \textit{$G_t^W$ is the shared tensor gradient, containing both \\ \hspace*{4.1em} input-lookup and output-softmax contributions}
      \Statex\vspace{-0.2cm}
      \State $B_t^{W}\leftarrow \mu B_{t-1}^{W}+G_t^{W}$
      \State $\eta_{\mathrm{tie},t}\leftarrow
      \eta_{\mathrm M,t}
      \left(\rho_{\mathrm{out}}/\rho_{\mathrm{hid}}\right)s/d$
      \Statex\vspace{-0.2cm}
      \For{each row $i=1,\dots,V$}
        \State $b_i\leftarrow \mathrm{row}_i(B_t^{W})$
        \State $u_i\leftarrow \mathcal{L}_c(b_i)$
      \EndFor
      \Statex\vspace{-0.2cm}
      \State $O_t^{W}\leftarrow U$, where $\mathrm{row}_i(U)=u_i$
      \State $W_t\leftarrow W_{t-1}-\eta_{\mathrm{tie},t}O_t^{W}$
    \Statex\vspace{-0.1cm}
    \ElsIf{$W$ is a one-dimensional auxiliary tensor}
    \Statex\vspace{-0.2cm}
      % \Comment{vector RMS auxiliary LMO: norm gains, QK-norm gains, biases}
      \State $B_t^{W}\leftarrow \mu B_{t-1}^{W}+G_t^{W}$
      \State $O_t^{W}\leftarrow
      B_t^{W}/\max\{\operatorname{RMS}(B_t^{W}),\varepsilon\}$
      \State $W_t\leftarrow W_{t-1}-\eta_{\mathrm{vec},t}O_t^{W}$
    \Statex\vspace{-0.1cm}
    \Else
      \State \textbf{error:} every parameter must be assigned an\\ \hspace{4em} explicit role
      \Statex\vspace{-0.1cm}
    \EndIf
    \Statex\vspace{-0.1cm}

  \EndFor
\EndFor
\State \Return $\theta_T$
\end{algorithmic}

% \vspace{0.35em}
\end{minipage}%
}
% \vspace{-0.5em}
\end{wrapfigure}
% \vspace{0.5em}
In our experiments $\eta_{\mathrm M}^{\max}=0.02$ and
$\rho_{\mathrm{out}}/\rho_{\mathrm{hid}}=3000/50=60$, giving
$\eta_{\mathrm{tie},t}=1.2\,a_ts/d$; our main setting is $c=3$, $s=0.5$, which
for all evaluated widths ($d>9$) lies strictly between Sign and row-RMS.

\paragraph{When does sharing require a new geometry?}
Not always. If one dense matrix is reused at $L$ depths under the
\emph{same} spectral constraint, the intersection of identical balls
collapses back to that ball, so the right update is the ordinary spectral LMO $D_{\|\cdot\|_{\mathrm{op}}}$ applied once to the accumulated gradient, $D_{\|\cdot\|_{\mathrm{op}}}(\sum_\ell G_\ell)$. 
% the ordinary spectral LMO applied once to the accumulated gradient, $D_{\mathrm{sp}}(\sum_\ell G_\ell)$ --- not a per-occurrence intersection. 
The tied vocabulary table differs because its lookup and classifier aliases impose genuinely heterogeneous row-wise and entrywise constraints whose intersection does not collapse: in a four-depth shared-MLP diagnostic, combining per-depth Muon directions after matching spectral radius retains only $74.1\%$ (up-projection) and $83.9\%$ (down-projection) of the linear objective  (Appendix~\ref{app:homogeneous-dense-sharing}).
% retains only $74.1\%$ and $83.9\%$ of the linear objective attained by applying Muon once to the accumulated up- and down-projection gradients (Appendix~\ref{app:homogeneous-dense-sharing}).

% \paragraph{When does sharing require a new geometry?}
% Not always. If one dense matrix is reused at $L$ depths under the
% \emph{same} spectral constraint, the intersection of identical balls
% collapses back to that ball, so the right update is the ordinary spectral LMO $D_{\|\cdot\|_{\mathrm{op}}}$ applied once to the accumulated gradient, $D_{\|\cdot\|_{\mathrm{op}}}(\sum_\ell G_\ell)$. 
% % the ordinary spectral LMO applied once to the accumulated gradient, $D_{\mathrm{sp}}(\sum_\ell G_\ell)$ --- not a per-occurrence intersection. 
% The tied vocabulary table differs because its lookup and classifier aliases impose genuinely heterogeneous row-wise and entrywise constraints whose intersection does not collapse: in a four-depth shared-MLP diagnostic, combining per-depth Muon directions after matching spectral radius retains only $74.1\%$ (up-projection) and $83.9\%$ (down-projection) of the linear objective  (Appendix~\ref{app:homogeneous-dense-sharing}).
\vspace{13pt}
\subsection{Full optimizer}
\label{subsec:algo}
% \vspace{-5pt}
Algorithm~\ref{alg:afmuon} combines these branches into a single AdamW-free optimizer, using the oracle inputs of \S\ref{subsec:lmo}: the Nesterov construction for hidden matrices, and the unnormalized buffer $B_t=\mu B_{t-1}+G_t$ for the tied table and auxiliary tensors. Decoupled weight decay applies only to hidden matrices; the tied table and auxiliary tensors receive none. The row oracle $\mathcal{L}_c$ uses closed-form fast
paths at $c=1$ and $c\ge\sqrt d$; for $1<c<\sqrt d$ we solve
\eqref{eq:tau-general} by per-row bisection in FP32, in row-preserving chunks so that chunking never changes the computed update.
The algorithm is stated for tied-head RoPE decoders, our headline setting; untied heads and learned positional tables are not part of the tied-table claim, and in GPT-style diagnostics with learned positional embeddings we keep WPE on the fixed baseline Sign rule rather than claim a new positional-table optimizer.\\
\\
AF-Muon differs from Hybrid Muon in two ways: no parameter class is delegated to AdamW or carries second-moment state, and the tied table is not treated as a generic auxiliary tensor but receives the clipped intersection LMO above, motivated by its dual input-lookup/output-softmax role.
% \textcolor{portlandorange}{TD: If there was free space, it might be good to add the convergence theorem}
% Convergence guarantees for AF-Muon are established in Appendix~\ref{app:sec-convergenceAFMuon}.
While we establish exact oracle properties and conditional convergence bounds for idealized variants of AF-Muon in Appendix~\ref{app:sec-convergenceAFMuon}, we emphasize that the practical implementation relies on momentum and approximate oracle maps (including finite-step Newton--Schulz and numerical bisection). Consequently, we make no claim of unconditional convergence for this stochastic setting.
% \vspace{-20pt}
\paragraph{Relationship to prior norm-based optimizers. }
AF-Muon remains within the norm-based/LMO optimizer family
\citep{large2024scalable,jordan2024muon,pethick2025scion,riabinin2025gluon}.
We claim no novelty for Muon's spectral update or for the existence of Sign or row-RMS LMOs. Our contribution is the tied-table geometry: when input and output embeddings are tied, the stored tensor is aliased across two roles, and we propose the clipped row-normalizer family \eqref{eq:clipped-row-lmo}, contains the standard endpoints as special cases but exposes finite interior updates, evaluated in
Section~\ref{sec:experiments}.

% Full optimizer summary version
% \begin{algorithm}[t]
% \caption{AF-Muon (AdamW-free)}
% \label{alg:afmuon}
% \begin{algorithmic}[1]
% \Require Muon LR $\eta_{\mathrm M}$, vector LR $\eta_{\mathrm{vec}}$, momentum
% $\mu$, matrix decay $\lambda_{\mathrm M}$, cap $c$, scale $s$, radii
% $\rho_{\mathrm{out}},\rho_{\mathrm{hid}}$, $\varepsilon$
% \For{$t=1,\dots,T$; each unique parameter $W$}
%   \If{$W$ is a hidden matrix}
%     \State $M_t\leftarrow\mu M_{t-1}+(1-\mu)G_t$;\ \ $Z_t\leftarrow(1-\mu)G_t+\mu M_t$
%     \State $W\leftarrow(1-\eta_{\mathrm M}\lambda_{\mathrm M})W-\eta_{\mathrm M}\sqrt{\max(1,m/n)}\,\mathrm{NS}(Z_t)$
%   \ElsIf{$W=E$ is the tied table}
%     \State $B_t\leftarrow\mu B_{t-1}+G_t$;\quad
%            $\operatorname{row}_i(O)\leftarrow\mathcal{L}_c(\operatorname{row}_i(B_t))$
%     \State $E\leftarrow E-\eta_{\mathrm{tie},t}\,O$ \Comment{$\eta_{\mathrm{tie},t}$ from Eq.~\eqref{eq:tielr}}
%   \Else \Comment{scalar / one-dimensional}
%     \State $B_t\leftarrow\mu B_{t-1}+G_t$;\quad
%            $W\leftarrow W-\eta_{\mathrm{vec}}\,B_t/\max\{\operatorname{RMS}(B_t),\varepsilon\}$
%   \EndIf
% \EndFor
% \end{algorithmic}
% \end{algorithm}

% \newpage
\section{Experiments}
\label{sec:experiments}
% \vspace{-7pt}
\subsection{Experimental setup}
\label{subsec:exp-setup}

We evaluate AF-Muon on nine tied-token settings: four decoder-only language models on FineWeb, two shared-vocabulary encoder--decoder configurations, sparse NanoGPT-MoE, ImageGPT-style ImageNet-32 token modeling, and protein language modeling on UniRef50. Unless otherwise stated, each optimizer comparison uses three paired seeds with the same model, tokenizer or token cache, data order, token budget, batching, precision, evaluation shard, and matrix-update protocol. All Muon-family runs use Muon's spectral update for hidden matrices with matrix weight decay $0.1$, FP32 trainable parameters and optimizer state, BF16 autocast,
and gradient clipping at norm $1.0$. Hybrid Muon keeps embeddings, tied tables, normalization parameters, and biases on AdamW, with auxiliary weight decay $0.01$ where used. SCION-style Sign replaces the tied table by the Sign endpoint $(c,s)=(1,1)$. AF-Muon uses the tied-table clipped LMO from Section~\ref{subsec:tied-lmo} with $(c,s)=(3,0.5)$, peak tied-table learning rate $\eta_{\mathrm{tie}}^{\max}=1.2s/d$, and scheduled rate
$\eta_{\mathrm{tie},t}=a_t\eta_{\mathrm{tie}}^{\max}$. For the seven
two-role tied-table settings, the output-to-hidden radius ratio is
$\rho_{\mathrm{out}}/\rho_{\mathrm{hid}}=3000/50=60$~\citep{pethick2025scion}, matching the SCION-style output-layer scaling used for Sign relative to Spectral updates. The T5Gemma2-style experiment uses the main $\rho_{\mathrm{out}}/\rho_{\mathrm{hid}}=60$
setting but for the small fully shared T5-style table, which receives encoder-input,
decoder-input, and output-classifier gradients, we additionally sweep this ratio
and use $\rho_{\mathrm{out}}/\rho_{\mathrm{hid}}=1$; details are reported in
Appendix~\ref{app:t5test_shared_vocab}.
\vspace{-5pt}

\begin{table*}[t]
\centering
\caption{
Main optimizer comparison across tied-token models and domains. Results are
mean $\pm$ SD over three paired seeds. All rows compare matched tied-table
optimizer arms under the same model, data, token budget, precision, batching,
and matrix-update protocol; detailed settings and per-seed results are reported
in Appendix~\ref{sec:details_of_experiments}.
NanoGPT, SmolLM2, Qwen2.5, Llama, T5, MoE, and protein rows report validation
perplexity; ImageGPT reports validation bits/dim.
}
\label{tab:main_results_all}
\scriptsize
\setlength{\tabcolsep}{12.5pt}
\renewcommand{\arraystretch}{0.92}
\begin{tabular}{llccc}
\toprule
Model / setting & Optimizer & Train loss $\downarrow$ & Val. loss $\downarrow$ & Val. PPL / bpd $\downarrow$ \\
\midrule
\multirow{3}{*}{NanoGPT}
& Hybrid Muon & 3.5150 $\pm$ 0.0028 & 3.5177 $\pm$ 0.0032 & 33.71 $\pm$ 0.11 \\
& SCION-style Sign & 3.5125 $\pm$ 0.0107 & 3.5188 $\pm$ 0.0124 & 33.75 $\pm$ 0.42 \\
& AF-Muon & \textbf{3.4898 $\pm$ 0.0011} & \textbf{3.4919 $\pm$ 0.0019} & \textbf{32.85 $\pm$ 0.06} \\
\midrule
\multirow{3}{*}{SmolLM2-135M}
& Hybrid Muon & 3.5088 $\pm$ 0.0028 & 3.5065 $\pm$ 0.0021 & 33.33 $\pm$ 0.07 \\
& SCION-style Sign & 3.4727 $\pm$ 0.0069 & 3.4706 $\pm$ 0.0040 & 32.16 $\pm$ 0.13 \\
& AF-Muon & \textbf{3.4453 $\pm$ 0.0006} & \textbf{3.4446 $\pm$ 0.0036} & \textbf{31.33 $\pm$ 0.11} \\
\midrule
\multirow{3}{*}{Qwen2.5-0.5B}
& Hybrid Muon & 3.6269 $\pm$ 0.0074 & 3.6393 $\pm$ 0.0039 & 38.06 $\pm$ 0.15 \\
& SCION-style Sign & 3.5984 $\pm$ 0.0046 & 3.6109 $\pm$ 0.0031 & 37.00 $\pm$ 0.11 \\
& AF-Muon & \textbf{3.5713 $\pm$ 0.0052} & \textbf{3.5831 $\pm$ 0.0026} & \textbf{35.99 $\pm$ 0.09} \\
\midrule
\multirow{3}{*}{Llama-3.2-1B}
& Hybrid Muon & 3.4932 $\pm$ 0.0027 & 3.5388 $\pm$ 0.0035 & 34.42 $\pm$ 0.12 \\
& SCION-style Sign & 3.4771 $\pm$ 0.0033 & 3.5225 $\pm$ 0.0075 & 33.87 $\pm$ 0.25 \\
& AF-Muon & \textbf{3.4594 $\pm$ 0.0027} & \textbf{3.5056 $\pm$ 0.0048} & \textbf{33.30 $\pm$ 0.16} \\
\midrule
\multirow{3}{*}{Small T5-style}
& Hybrid Muon
& 5.3124 $\pm$ 0.0050
& 5.3374 $\pm$ 0.0102
& 207.98 $\pm$ 2.12 \\
& SCION-style Sign
& 5.1477 $\pm$ 0.0080
& 5.1763 $\pm$ 0.0150
& 177.03 $\pm$ 2.66 \\
& AF-Muon
& \textbf{5.1476 $\pm$ 0.0139}
& \textbf{5.1708 $\pm$ 0.0034}
& \textbf{176.06 $\pm$ 0.61} \\
\midrule
\multirow{3}{*}{T5Gemma2-style}
& Hybrid Muon
& 3.6409 $\pm$ 0.0044
& 3.6548 $\pm$ 0.0041
& 38.66 $\pm$ 0.16 \\
& SCION-style Sign
& 3.5698 $\pm$ 0.0011
& 3.5859 $\pm$ 0.0017
& 36.09 $\pm$ 0.06 \\
& AF-Muon
& \textbf{3.5573 $\pm$ 0.0010}
& \textbf{3.5735 $\pm$ 0.0010}
& \textbf{35.64 $\pm$ 0.04} \\
\midrule
\multirow{3}{*}{NanoGPT-MoE}
& Hybrid Muon & 3.4873 $\pm$ 0.0031 & 3.4902 $\pm$ 0.0019 & 32.79 $\pm$ 0.06 \\
& SCION-style Sign & 3.5106 $\pm$ 0.0033 & 3.5129 $\pm$ 0.0025 & 33.55 $\pm$ 0.08 \\
& AF-Muon & \textbf{3.4662 $\pm$ 0.0011} & \textbf{3.4677 $\pm$ 0.0007} & \textbf{32.06 $\pm$ 0.02} \\
\midrule
\multirow{3}{*}{ImageGPT-style}
& Hybrid Muon & 2.3236 $\pm$ 0.0067 & 2.3346 $\pm$ 0.0049 & 1.1227 $\pm$ 0.0024 \\
& SCION-style Sign & 2.3568 $\pm$ 0.0049 & 2.3712 $\pm$ 0.0055 & 1.1403 $\pm$ 0.0027 \\
& AF-Muon & \textbf{2.3208 $\pm$ 0.0101} & \textbf{2.3243 $\pm$ 0.0055} & \textbf{1.1177 $\pm$ 0.0026} \\
\midrule
\multirow{4}{*}{Protein LM}
& AdamW & 7.951 $\pm$ 0.023 & 7.925 $\pm$ 0.091 & 2772.6 $\pm$ 258.2 \\
& Hybrid Muon & 7.478 $\pm$ 0.013 & 6.874 $\pm$ 0.008 & 967.0 $\pm$ 7.9 \\
& SCION-style Sign & 7.464 $\pm$ 0.015 & 6.816 $\pm$ 0.012 & 912.7 $\pm$ 11.3 \\
& AF-Muon & \textbf{7.426 $\pm$ 0.019} & \textbf{6.760 $\pm$ 0.019} & \textbf{862.3 $\pm$ 16.5} \\
\bottomrule
\end{tabular}
\vspace{-0.8em}
\end{table*}
\subsection{Main decoder-only language modeling results}
\label{subsec:main-results}
Table~\ref{tab:main_results_all} compares Hybrid Muon with
the best current AF-Muon configuration on each model. Across all settings, AF-Muon achieves the best mean validation loss and PPL/BPD. Paired three-seed tests against the strongest non-AF baseline are reported with the per-seed results in Appendix~\ref{sec:details_of_experiments}.

\textbf{NanoGPT.}
NanoGPT provides the cleanest controlled decoder-only test of the aliased sparse--dense tied table. AF-Muon obtains the best validation loss and perplexity across all three paired seeds, showing that replacing the auxiliary AdamW treatment of the shared table with the finite-cap LMO improves the matched Muon matrix-update protocol. long horizon curves and details are provided in Appendix~\ref{app:nanogpt_tied_vocab} Ablations over gradient source, batch size, auxiliary learning rate, vocabulary size, model width, and cap/scale choices are reported in Appendix~\ref{sec:ablation}.

\textbf{SmolLM2-135M.}
SmolLM2-135M tests the same tied-table update in a pretrained-family
decoder-only architecture at larger width and depth. At 2.49B training tokens, AF-Muon achieves the best validation loss and perplexity across all paired seeds, showing that the finite-cap tied-table update remains useful beyond the controlled NanoGPT architecture. Details are reported in Appendix~\ref{app:smollm2_tied_vocab}.

% \textbf{Qwen2.5-0.5B.} Qwen2.5-0.5B tests the tied-table update in a larger decoder-only architecture with a substantially larger shared vocabulary table. At the near-1B evaluation point, AF-Muon improves over the strongest non-AF baseline, SCION-style Sign, by $0.0278\pm0.0022$ paired validation-loss points; the gain over Hybrid Muon is $0.0562\pm0.0016$. The improvement is positive for all three paired seeds. We also continue seed 43 in a long-horizon run, with stability results reported in
% Appendix~\ref{app:qwen25_0p5b_long_horizon}.

\textbf{Qwen2.5-0.5B.}
Qwen2.5-0.5B tests the tied-table update in a larger decoder-only architecture with a substantially larger shared vocabulary table. At the near-1B evaluation point, AF-Muon achieves the best validation loss and perplexity across all paired seeds. A seed-43 long-horizon continuation shows the same stable ordering later in training, with details reported in Appendix~\ref{app:qwen25_0p5b_tied_vocab}, and ~\ref{app:qwen25_0p5b_long_horizon}.

% \begin{table}[ht]
% \centering
% \caption{
% Qwen2.5-0.5B decoder-only language-model training on FineWeb across three
% paired seeds. We report mean $\pm$ SD over seeds at the last evaluation before
% the 1B-token budget, corresponding to 0.918B tokens. More details are given in
% Appendix~\ref{app:qwen25_0p5b_tied_vocab}.
% }
% \label{tab:qwen25_0p5b_tied_vocab}
% \begin{tabular}{lccc}
% \toprule
% Optimizer & Train loss $\downarrow$ & Validation loss $\downarrow$ & Validation PPL $\downarrow$ \\
% \midrule
% Hybrid Muon & 3.6269 $\pm$ 0.0074 & 3.6393 $\pm$ 0.0039 & 38.06 $\pm$ 0.15 \\
% SCION-style Sign & 3.5984 $\pm$ 0.0046 & 3.6109 $\pm$ 0.0031 & 37.00 $\pm$ 0.11 \\
% AF-Muon & \textbf{3.5713 $\pm$ 0.0052} & \textbf{3.5831 $\pm$ 0.0026} & \textbf{35.99 $\pm$ 0.09} \\
% \bottomrule
% \end{tabular}
% \end{table}
% \begin{figure}[ht]
%    \centering
%    \includegraphics[width=\linewidth]{figs/qwen_3seeds/qwen25_0p5b_1b_curves_tail250m.pdf}
%    \caption{
%    Qwen2.5-0.5B training dynamics on FineWeb. Curves show the mean over three
%    paired seeds, and shaded regions denote $\pm 1$ standard error of the mean.
%    The plot shows the short-horizon region used for the main three-seed
%    comparison, omitting the earliest training region for readability. All
%    methods use tied input/output embeddings, FP32 trainable parameters and
%    optimizer state with BF16 autocast, matrix weight decay $0.1$, and the same
%    matched large-batch constant-learning-rate protocol.
%    }
%    \label{fig:qwen25_0p5b_tied_vocab_curves}
% \end{figure}
\begin{figure*}[!t]
\centering
\makebox[\linewidth][c]{%
\includegraphics[
  width=1.0\linewidth,
  trim=0 22 0 8,
  clip
]{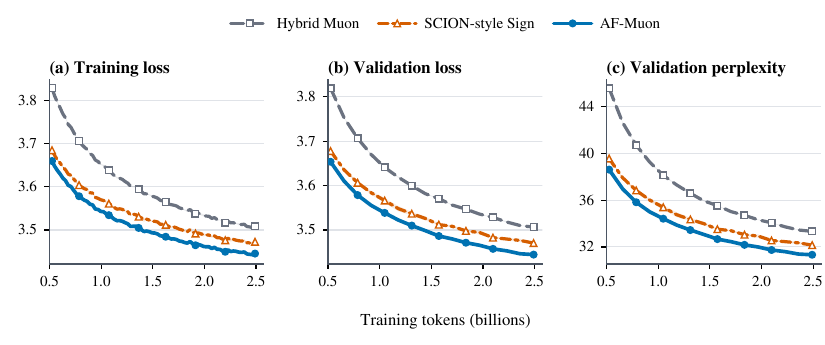}
}

\vspace{-0.6em}

\makebox[\linewidth][c]{%
\includegraphics[
  width=1.0\linewidth,
  trim=0 7 0 38,
  clip
]{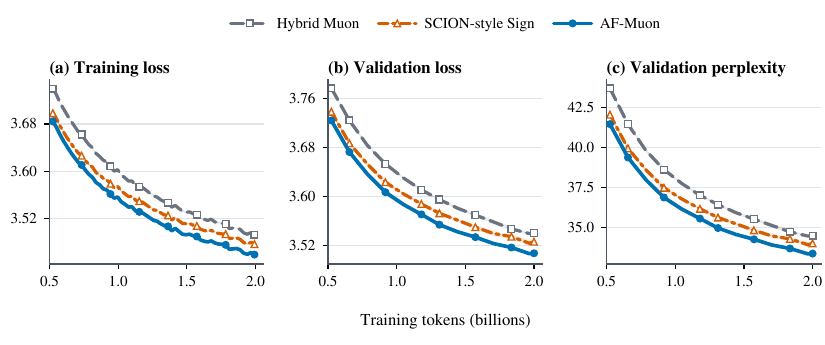}
}

\vspace{-7pt}

\caption{
Decoder-only tied-vocabulary training dynamics on FineWeb for SmolLM2-135M (top) and Llama-3.2-1B (bottom). Each row reports metrics under the matched protocol in Section~\ref{subsec:exp-setup}. Curves show the mean over three paired seeds. Llama model long horizon curves reported in Appendix~\ref{app:llama32_1b_long_horizon}}
\label{fig:decoder_only_tied_vocab_curves}
\end{figure*}

%-----llama-------
% \newpage
% \textbf{Llama-3.2-1B.}
% Llama-3.2-1B is the largest decoder-only tied-vocabulary setting in our main comparison. At the final available evaluation near the 2B-token budget, AF-Muon improves over the strongest non-AF baseline, SCION-style Sign, by $0.0168\pm0.0039$ paired validation-loss points; the gain over Hybrid Muon is $0.0331\pm0.0036$. The improvement is positive for all three paired seeds. We also continue seed 43 to 9.98B tokens as a long-horizon stability check, with results reported in Appendix~\ref{app:llama32_1b_tied_vocab}, and~\ref{app:llama32_1b_long_horizon}.

\textbf{Llama-3.2-1B.}
Llama-3.2-1B is the largest decoder-only tied-vocabulary setting in our main comparison. Near the 2B-token budget, AF-Muon again achieves the best validation loss and perplexity across all paired seeds. A seed-43 continuation to 9.98B tokens provides a longer-horizon stability check and preserves the same ordering;
full results are reported in Appendix~\ref{app:llama32_1b_tied_vocab}, and~\ref{app:llama32_1b_long_horizon}.
\vspace{-5pt}
\subsection{Generalization beyond dense decoder-only language modeling}
\textbf{Encoder--decoder shared vocabulary.}
Motivated by T5Gemma~2~\citep{t5gemma2025}, we use
two experiments to ask whether AF-Muon transfers to three-way shared
vocabularies and whether it remains effective in a larger configuration. In the
small T5-style control (275M tokens, 49k shared vocabulary, AF-Muon and SCION-style Sign both improve over tied Muon. In T5Gemma2-style (786M parameters, 262k shared vocabulary, 750M target tokens), AF-Muon gives the best validation loss in the evaluated seeds. Together, these results show that AdamW-free shared-table optimization transfers to three-way tying. Full results are in Appendix~\ref{app:t5test_shared_vocab}.

\textbf{Sparse NanoGPT-MoE.}
following the MoE scaling direction of recent models such as DeepSeek-V3 and OLMoE~\citep{deepseekv3,olmoe}, and Muon-family optimizers have been used at frontier MoE scale in Kimi K2~\citep{kimiteam2025kimik2}. We replace NanoGPT MLPs with four top-1 routed experts while keeping the vocabulary table tied. 
% Thus routing changes the hidden feed-forward computation, but the tied vocabulary table remains a global token interface. 
At 2.62B tokens, AF-Muon achieves the best validation loss and perplexity; Sign underperforms Muon, indicating that the finite cap, not merely removing AdamW, accounts for the improvement.  Full details are in (Appendix~\ref{app:moe_nanogpt_tied_vocab}).

\begin{figure}[t]
    \centering
    \includegraphics[
        width=\linewidth,
        trim={0 6pt 0 4pt},
        clip
    ]{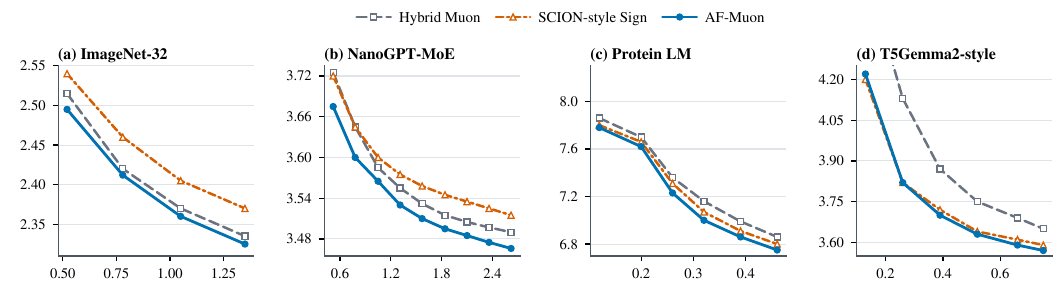}
    \vspace{-5pt}
    \caption{
Validation-loss curves beyond decoder-only language modeling. Each panel shows one tied-token setting with its own training-token scale. AF-Muon reaches the lowest final validation loss in all four settings; full training-loss and perplexity curves are in the related appendix.
}
\vspace{-10pt}
%     \caption{
% Validation-loss curves beyond decoder-only language modeling. Each panel shows one tied-token setting, with the x-axis in billions of training tokens for that experiment. AF-Muon gives the lowest final validation loss in all four settings.
% }
    \label{fig:nondecoder_val_loss_overview}
\end{figure}

\textbf{ImageGPT-style ImageNet-32.}
% Autoregressive modeling of discrete visual tokens is a competitive approach to image generation, from LlamaGen's decoder-only next-token transformer to VAR's next-scale prediction~\citep{sun2024llamagen,tian2024var}. Unified image--text models also tokenize both modalities autoregressively~\citep{chameleon2024,emu3}.
Autoregressive discrete-token modeling underlies modern image and image--text generators, including LlamaGen, VAR, Chameleon, and Emu3~\citep{sun2024llamagen,tian2024var,chameleon2024,emu3}. Motivated by this discrete-token modeling paradigm, we test an ImageGPT-style causal transformer on ImageNet-32 with an intentionally tied input/output color-token table. Images are represented as 1024 RGB554 tokens from a 16k vocabulary, so the shared table again receives sparse lookup and dense classifier gradients. At 1.31B image tokens, AF-Muon achieves the best validation loss and bits/dim, and a fixed top-half completion evaluation gives the same ordering. Full settings and results are reported in Appendix~\ref{app:tied_imagegpt_imagenet32}.

% \textbf{Protein language modeling.}
% Autoregressive protein language models such as ProtGPT2 and ProGen2 model protein sequences with GPT-style next-token objectives
% \citep{ferruz2022protgpt2,nijkamp2022progen2}, so the aliased sparse-lookup/dense-classifier tied-table structure can be tested on a non-linguistic token distribution. Following the GPT-2-based ProtGPT2, we adopt its BPE tokenizer and tie the token embedding / LM head, training a randomly initialized 363M-parameter RoPE decoder on the UniRef50 2021\_04 release~\citep{suzek2015uniref}. At 458.8M training tokens, AF-Muon improves over the strongest non-AF baseline, SCION-style Sign, by $0.0569\pm0.0152$ paired eval-loss points, with a larger gain over Hybrid Muon, and reduces validation perplexity from $912.7$ (Sign) to $862.3$. All Muon-family arms substantially outperform tuned AdamW in this setting, while AF-Muon is best within the Muon-family comparison. Full details, per-seed results, and curves are in Appendix~\ref{app:protein_lm_per_seed}.

\textbf{Protein language modeling.}
Autoregressive protein language models such as ProtGPT2 and ProGen2 model protein sequences with GPT-style next-token objectives
\citep{ferruz2022protgpt2,nijkamp2022progen2}, letting us test the aliased sparse-lookup/dense-classifier tied-table structure on a non-linguistic token distribution. Following ProtGPT2, we use its BPE tokenizer and train a randomly initialized 363M-parameter RoPE decoder with tied token embedding / LM head on the UniRef50 2021\_04 release~\citep{suzek2015uniref}. At 458.8M training
tokens, AF-Muon achieves the best metrics, and all Muon-family arms substantially outperform tuned AdamW. Full details are in Appendix~\ref{app:protein_lm_per_seed}.

\vspace{-5pt}
\subsection{Ablations and robustness}
\label{subsec:ablations}
\vspace{-5pt}
% \textbf{Causal attribution and scope.}
% A $2\times2$ factorial ablation independently replaces the tied-table AdamW update with the finite-cap LMO and the one-dimensional AdamW update with the RMS-LMO. Relative to the AdamW/AdamW auxiliary baseline, changing only the tied table improves validation loss by $0.0133$, changing only the one-dimensional
% parameters improves it by $0.0075$, and changing both improves it by $0.0413$. Thus, the tied-table rule makes an independent contribution, while the non-additive joint gain also reveals an interaction between the two role-specific updates (Appendix~\ref{app:nanogpt_factorial}). A complementary
% gradient-source intervention shows why the tied table requires a mixed-role treatment: the dense output component is dominant, but adding the sparse input component improves validation loss by $0.0239$, and the two components have cosine $0.003$ (Appendix~\ref{app:nanogpt_tied_source_ablation}). Conversely, a
% negative boundary control shows that sharing alone is insufficient: when one dense matrix is reused at multiple depths, every use induces the same spectral geometry, so their intersection collapses to the original spectral ball and Muon should be applied once to the accumulated gradient (Appendix~\ref{app:homogeneous-dense-sharing}).

\textbf{Causal attribution and scope.}
A $2\times2$ factorial ablation independently replaces the tied-table AdamW update with the finite-cap LMO and the one-dimensional AdamW update with the RMS-LMO  (Figure~\ref{fig:app_nanogpt_factorial_curves};Appendix~\ref{app:nanogpt_factorial}). Both substitutions contribute, and the joint improvement is larger than
the sum of the isolated effects, indicating an interaction between the two role-specific updates. A complementary
gradient-source intervention shows why the tied table requires mixed-role treatment: the dense output component dominates, but adding the sparse input component still improves validation, and the two components are nearly orthogonal (cosine $0.003$) (Appendix~\ref{app:nanogpt_tied_source_ablation}). Conversely, a
negative boundary control shows that sharing alone is insufficient: when one dense matrix is reused at multiple depths, every use induces the same spectral geometry, so their intersection collapses to the original spectral ball and Muon
should be applied once to the accumulated gradient (Appendix~\ref{app:homogeneous-dense-sharing}).

\textbf{Geometry versus update scale.}
At fixed $c=3$, every tested scale $s\in[0.125,0.5]$ achieves lower validation loss than the standard SCION-style Sign arm in the single-seed sensitivity sweep, with a broad favorable region around intermediate scales (Appendix~\ref{app:nanogpt_tied_step_scale}). A joint sweep over $c\in\{1,2,3,4,6,10,\infty\}$ and $s\in\{0.5,1\}$ further shows that performance is more sensitive to $s$ than to the precise cap: at $s=0.5$, validation losses span only $0.0069$ across the complete cap range (Appendix~\ref{app:nanogpt_cap_scale_sensitivity}). We therefore do not present $c=3$ as a uniquely optimal post-hoc choice. Rather, it realizes a nondegenerate interior geometry: in this audit it caps only $0.29\%$ of coordinates while affecting $90.3\%$ of rows and attaining $1.253\times$ the same-state Sign linear objective. The sweep also recovers the expected Sign and row-RMS endpoint behavior at $c=1$ and $c=\infty$, respectively.

\textbf{Robustness across scales and protocols.}
The AF-Muon advantage over Sign persists across genuinely re-tokenized 50k, 100k, and 150k vocabularies, with within-tokenizer mean validation-loss improvements of $0.036$, $0.045$, and $0.052$
(Appendix~\ref{app:nanogpt_retokenized_vocab}). It also persists at real hidden widths $d\in\{768,1024,1536\}$, where the three-seed paired improvements are $0.036$, $0.040$, and $0.032$
(Appendix~\ref{app:nanogpt_real_width}). In single-seed protocol sweeps, AF-Muon obtains the lowest observed validation loss at every tested batch size from 32k to 524k tokens per update
(Figure~\ref{fig:app_nanogpt_batch_sensitivity};Appendix~\ref{app:nanogpt_batch_sensitivity}) and at every point in a $16\times$ auxiliary-branch learning-rate grid
(Appendix~\ref{app:nanogpt_auxlr_sensitivity}). We treat these latter sweeps as sensitivity evidence rather than multi-seed uncertainty estimates.

\begin{figure}[t]
    \centering
    \includegraphics[
        width=\linewidth,
        trim={0.0in 0.40in 0.00in 0.05in},
        clip
    ]{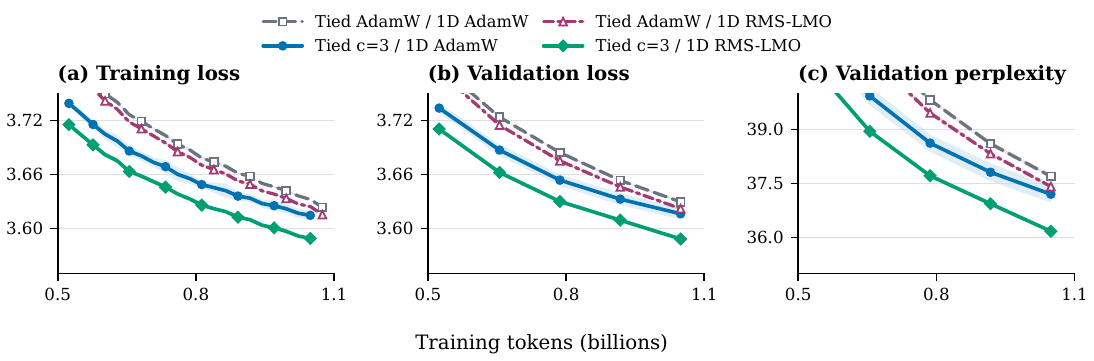}
    \caption{
NanoGPT $2\times2$ factorial ablation over training tokens (billions) for three paired seeds.
}
    \label{fig:app_nanogpt_factorial_curves}
\end{figure}

\begin{figure}[h]
    \centering
    \includegraphics[
  width=\linewidth,
  trim=0.0in 0.38in 0.0in 0.1in,
  clip
]{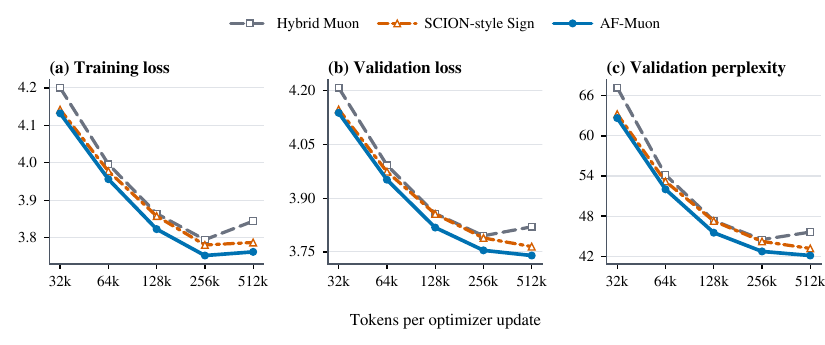}
    \caption{
    NanoGPT batch-size sensitivity at a matched 458.8M-token comparison. The x-axis varies the number of tokens per optimizer update. 
    AF-Muon leads across all tested batch sizes.
    }
    \label{fig:app_nanogpt_batch_sensitivity}
\end{figure}

\vspace{-15pt}
\section{Limitations}
\label{sec:limitations}
\vspace{-7pt}
% AF-Muon is designed for architectures with a genuinely aliased vocabulary table, where lookup and classifier uses induce different gradient structures. The finite-cap construction does not directly prescribe updates for untied input embeddings and output heads, learned positional embeddings, or homogeneous dense parameters reused across depth; such parameters require their own role-appropriate geometries. Weight tying also does not uniquely imply the intersection of max-row-$\ell_2$ and max-entry constraints: we prove that our update is the exact support-aware LMO for this proposed feasible set, not that the set is optimal among all possible role-aware geometries. Our default $c=3,s=0.5$ is a robust interior configuration rather than a universally optimal choice. The sensitivity experiments show broad favorable regions, but tied-step calibration can depend on the sharing topology, as in the fully shared encoder--decoder setting. The row-wise oracle uses chunked FP32 bisection, which reduces persistent optimizer state but does not guarantee lower wall-clock time on every hardware or kernel stack. Finally, our results establish robustness under matched fixed-token protocols; they do not exhaustively tune every optimizer family or test frontier-scale pretraining.
Weight tying does not uniquely imply the intersection of max-row-$\ell_2$ and max-entry constraints: we prove that our update is the exact support-aware LMO for this proposed feasible set, not that the set is optimal among all possible role-aware geometries. The default $c=3$, $s=0.5$ is a robust configuration rather than a universally optimal choice. Our results establish robustness under matched fixed-token protocols; they do not tune every optimizer family or test frontier-scale pretraining.
\vspace{-7pt}
\section{Conclusion}
\label{sec:conclusion}
\vspace{-7pt}

We introduced AF-Muon, an AdamW-free extension of Muon that treats the tied vocabulary table as an \emph{aliased} parameter block---one stored matrix receiving sparse lookup and dense classifier gradients---rather than a generic auxiliary tensor. It keeps Muon's spectral update for hidden matrices, RMS
updates for one-dimensional parameters, and assigns the tied table the exact support-aware LMO of an intersection of max-row-$\ell_2$ and max-entry constraints, removing all second-moment state and cutting optimizer memory by about 20\% at roughly 1\% training-time overhead. Across nine tied-token settings spanning 124M--1B-parameter decoder-only, fully shared encoder--decoder, image-token, protein, and sparse-MoE, trained on text, image, and protein data up to $\sim\!10\text{B}$ tokens, AF-Muon improves fixed-budget validation over Muon and the SCION-style Sign endpoint under matched three-seed protocols, with ablations and sensitivity studies. These results support treating functional aliasing, not stored tensor shape alone, as a first-class consideration in optimizer design.

% \newpage
\subsection*{AI use statement}

We used generative AI tools for manuscript polishing, LaTeX debugging, code debugging, and internal consistency checks. All experiments, results, claims, and final artifacts were reviewed and verified by the authors, who take full responsibility for the content of the paper.

\subsection*{Ethics statement}

This work studies optimizer design for tied-token models and does not involve human subjects, private user data, or the release of a new dataset. Experiments use public or standard research datasets and model families for language, image-token, and protein-sequence modeling. The main ethical considerations are computational cost, reproducibility, and responsible reporting of empirical claims. We therefore report matched-seed comparisons, ablations, implementation details, and limitations, and we avoid claims about downstream deployment or societal impact beyond the optimizer behavior studied here.

\subsection*{Reproducibility statement}

We support reproducibility by reporting the full optimizer specification in Algorithm~\ref{alg:afmuon}, the tied-table geometry and derivations in the appendix, and matched experimental settings in Section~\ref{subsec:exp-setup}. The appendix provides per-seed results, ablations, diagnostics, overhead measurements, and additional training curves for the model families and domains studied in the main paper. We provide source code, configuration files, plotting scripts, and reproduction instructions at
\url{https://github.com/arashlagzian/afmoun}.

% \newpage
\bibliography{iclr2027_conference}
\bibliographystyle{iclr2027_conference}

\newpage
\appendix

\etocdepthtag.toc{appendix}
\etocsettagdepth{main}{none}
\etocsettagdepth{appendix}{subsection}

{
  \hypersetup{linkcolor=smokyblack}
  \tableofcontents
}
% \newpage
\section{Related Work}
\label{sec:related-work}
\paragraph{Hybrid Muon and auxiliary updates.}
Adam and AdamW use coordinate-wise exponential moving averages of first and
second moments and remain strong default optimizers for language-model
pretraining~\citep{kingma2015adam,loshchilov2019decoupled}.  Related methods
modify this trade-off in different ways: Adafactor factorizes matrix second
moments to reduce persistent state~\citep{shazeer2018adafactor}, Shampoo uses
tensor-mode gradient statistics to construct structured
preconditioners~\citep{gupta2018shampoo}, and Lion removes the second moment but
applies a coordinate-wise Sign update to an interpolation of current and past
gradients~\citep{chen2023lion}.  Muon instead maps momentum for selected
two-dimensional hidden weights through a short Newton--Schulz iteration that
approximates the polar factor, i.e., the exact spectral-norm steepest-descent
direction in the converged limit~\citep{jordan2024muon}.  The reference Muon
recipe explicitly leaves input embeddings, output heads, gains, and biases to
AdamW, and subsequent scaling work emphasizes that weight decay and
parameter-class update scales are important at language-model
scale~\citep{liu2025muonscalable}. At frontier scale, Muon underlies trillion-parameter training: Kimi~K2 pretrains a one-trillion-parameter mixture-of-experts model with a Muon-based optimizer and a QK-clipping mechanism that suppresses attention-logit explosion~\citep{kimiteam2025kimik2}, underscoring that the output-side and stability behavior of Muon-style updates remains a live concern at scale. AF-Muon preserves Muon's hidden-matrix
branch but replaces the entire auxiliary AdamW branch with one-buffer,
role-specific updates.

\paragraph{Norm- and LMO-based design.}
The norm-based view treats an update rule as a choice of geometry rather than
as a shape-agnostic transformation of coordinates.  In particular,
\citet{bernstein2024oldnorm} argue that tensors with the same array shape can
require different norms when they implement different functions, such as
linear maps and embedding lookups.  The modular norm recursively constructs an
architecture-level geometry and studies hyperparameter transfer across model
scale~\citep{large2024scalable}, while modular duality assigns operator norms
from layer semantics and derives efficient duality maps for Embed, Linear, and
Conv2D modules~\citep{bernstein2025modular}.  SCION develops stochastic
optimization with norm-ball linear minimization oracles (LMOs), applies them to
unconstrained neural-network training, and studies transferable layer
radii~\citep{pethick2025scion}.  Gluon develops a layerwise LMO framework and
generalized-smoothness model that more closely reflects practical Muon and
SCION implementations~\citep{riabinin2025gluon}.  A unifying view casts Muon, Lion, and related rules as instances of optimistic dual averaging~\citep{pethick2026soda}, and a growing body of analysis establishes convergence for Muon-type methods under non-Euclidean and inexact-oracle assumptions~\citep{convmuon2025,inexactlmo2026}. Recent scaling evidence also finds the output layer to be especially sensitive to its group learning rate and identifies its operator norm as a transferable
quantity~\citep{filatov2025optimal}.  AF-Muon follows this lineage but
separates two choices for the tied table: the cap $c$ selects its normalized LMO geometry,
whereas the radius multiplier $s$ calibrates the magnitude of the completed
direction.

\paragraph{Vocabulary geometry.}
SCION already distinguishes input, hidden, and output layers: for an untied
one-hot input it uses a normalized-column oracle (row normalization in our
$V\times d$ table orientation) and a Sign output oracle, while under
first/last-layer weight sharing it selects Sign for the shared
layer~\citep{pethick2025scion}.  Thus SCION Sign is a direct endpoint
baseline for our setting.  Concurrent symmetry-based work derives equivariant updates for
separate input embeddings and \emph{untied} LM heads, including row-normalized,
one-sided spectral, hybrid, and centered maps; its reported end-to-end stacks
retain AdamW for scalar and vector blocks~\citep{lau2026symmetry}.  These
results strongly support functional-role-specific vocabulary geometry, but they do not prescribe a
joint oracle for one aliased embedding/head tensor.  We do not claim novelty
for Sign or row-RMS themselves.  Our contribution is the exact LMO of a
proposed intersection of max-row-$\ell_2$ and max-entry balls for the merged
tied-table momentum.  It contains SCION Sign and row-RMS as endpoints and
exposes finite-cap directions between them without adding two separately scaled
updates.

\paragraph{Weight tying and mixed gradients.}
Input--output weight tying was proposed independently by
\citet{press2017using} and \citet{inan2017tying} to reduce the
vocabulary-dependent parameter count and often improve language modeling.  The
update analysis of \citet{press2017using} reveals the asymmetry central to our
problem: the lookup path contributes only to rows corresponding to observed
input tokens, whereas the output-softmax path generally contributes to every
vocabulary row.  Consequently, a tied table evolves more like the output
embedding of an untied model than its input embedding.  Later work further
distinguishes the structures learned by input and output
embeddings~\citep{bertolotti2024tying}.  Most recently, \citet{lopardo2026weight} give causal evidence that output-softmax gradients dominate the tied table early in training, biasing the shared matrix toward the output (unembedding) space, and mitigate this representational bias by rescaling the input gradient. This independently corroborates the sparse-lookup/dense-classifier asymmetry central to our setting; AF-Muon addresses the same imbalance from a complementary direction, changing the LMO geometry used to update the aliased tensor rather than reweighting its merged gradient. Together, these studies motivate treating a tied table as neither a pure lookup table nor a pure output classifier.

Prior work has also addressed conflicts in shared embeddings by changing the
parameterization or gradient flow.  Projected and generalized tying relax hard
equality between input and output
representations~\citep{gulordava2018represent,pappas2018beyond}; in the
different setting of generator/discriminator embedding sharing, DeBERTaV3 disentangles gradients
that otherwise produce ``tug-of-war'' dynamics~\citep{he2023debertav3}.
AF-Muon instead preserves standard causal-LM weight tying, keeps the single
autograd-merged lookup and softmax gradient and one momentum buffer, and changes
only the geometry used to update that aliased tensor.

\paragraph{Muon variants.}
MuonAll pursues an AdamW-free recipe by applying Muon to existing
two-dimensional parameters, including embedding and head matrices, and by
diagonalizing one-dimensional tensors before Newton--Schulz; its experiments
focus on supervised fine-tuning rather than from-scratch
pretraining~\citep{page2025muonall}.  Very recent work on Muon representation
geometry systematically studies how native, nearest-square, skinny, and vector
representations alter the polar update~\citep{chang2026muse}.  Other Muon
variants combine orthogonalization with element-wise or neuron-wise
second-moment adaptation~\citep{si2025adamuon,li2025normuon}.  These approaches
either extend a matrix-orthogonalization template to additional stored shapes
or restore adaptive second-moment statistics.  AF-Muon takes a complementary
route: it reserves spectral geometry for hidden linear maps, assigns RMS
geometry to one-dimensional auxiliary tensors, and introduces a native
finite-cap geometry for the tied vocabulary table, without second-moment state
in any branch.

\paragraph{Memory-efficient optimization.}
A complementary line of work targets optimizer \emph{memory} rather than update
geometry. Adafactor factorizes the second moment to sublinear
state~\citep{shazeer2018adafactor}, Adam-mini replaces per-coordinate second
moments with a single rate per parameter block~\citep{zhang2024adammini}, and
GaLore maintains moments in a low-rank gradient
subspace~\citep{zhao2024galore}; minimalist optimizer designs pretrain with
substantially reduced persistent state~\citep{xu2025minimalist}. Muon-family
methods are already lighter than AdamW because they store one momentum buffer
and no second moment~\citep{jordan2024muon,liu2025muonscalable}. Several recent
variants, however, reintroduce second-moment statistics on the orthogonalized
update to improve per-neuron conditioning~\citep{si2025adamuon,li2025normuon},
trading memory back for adaptivity. AF-Muon moves in the opposite direction: it
eliminates the auxiliary AdamW second moment entirely and keeps a single buffer
per tensor in every branch, so the optimizer is second-moment-free throughout
while still assigning role-specific geometry to hidden matrices, the tied table,
and vector parameters.

\vspace{1.5em}
\noindent\fbox{%
    \parbox{\textwidth}{%
    \begin{itemize}[leftmargin=1.5em]
    \item[\ding{114}] Appendix \ref{app:sec_setting} fixes notation and recalls the standard norm-ball LMOs used as building blocks.
    \item[\ding{114}] Appendix \ref{app:sec_aliaspulllback} develops the alias-pullback framework in general: the pullback norm and its intersection ball
    (Theorem \ref{app:the_aliaspullbacknorm}), its dual as an infimal convolution (Theorem \ref{app:the_alias_aware_dual}), and the intersection form of simultaneous role-wise redundancy (Theorem \ref{app:the_redundancy_intersection}).
    \item[\ding{114}] Appendix \ref{app:sec-hardtiedtables} instantiates it for a hard-tied table with two roles, giving the tied norm (\ref{app:eq_tiehardnorm}). 
    \item[\ding{114}] Appendix \ref{app:sec_exactClippedRow} solves the resulting LMO exactly (Theorem \ref{app:the_exactsupportware}). 
    \item[\ding{114}] Appendix \ref{app:sec_geo-rete-concen} computes the dual norm, the objective-retention ratio, and the concentration bound.
    \item[\ding{114}] Appendix \ref{app:sec_symmetry-nogo} determines the exact symmetry class of the oracle and shows that its departure from full orthogonal equivariance is structural.
    \item[\ding{114}] Appendix \ref{app:sec-convergenceAFMuon} collects the conditional convergence statements and delimits their scope. 
    \item[\ding{114}] Appendix \ref{app:three_way_tying} extends the results to three-way tying.
    \end{itemize}
    }
}

%=============================
% First Section
%=============================
\section{Setting, notation, and standard LMO geometry}
\label{app:sec_setting}
\textit{For Appendices \ref{app:sec_setting} through \ref{app:sec-convergenceAFMuon}, we focus exclusively on decoder-only models; in Appendix \ref{app:three_way_tying}, we demonstrate how our results extend to other architectures.}

We consider a decoder-only language model with stored parameters
\[
\theta=(W_1,\ldots,W_L,E,v_1,\ldots,v_K),
\]
where $W_\ell\in \mathbb{R}^{m_l\times n_l}$ are hidden Matrix weights, $E\in\R^{V\times d}$ is a hard-tied token-embedding/LM-head table, and $v_k$ are scalar or one-dimensional auxiliary parameters such as normalization gains and biases. Here $V$ is the vocabulary size and $d$ is the hidden width. In stored coordinates, the table serves two roles:
\begin{align*}
    &i\longmapsto E^\top e_i\in\R^d \qquad \text{(token lookup)}\\
    &h\longmapsto Eh\in\R^V  \qquad \hspace{0.5em}  \text{(output classification)},
\end{align*}
where $e_i$ is a vocabulary basis vector, and both roles refer to the same parameter object.\\
Hard tying requires a single update for the stored table. AF-Muon chooses its update geometry by imposing both a lookup-side row constraint and an output-side coordinate constraint on that update. These constraints are a modeling choice; weight tying alone does not uniquely determine their norms, radii, or learning rates. Weight tying and the distinction between input and output embeddings predate this construction (\cite{press2017using}; \cite{inan2017tying}).

\subsection{Notations}
All spaces are finite-dimensional and real. We use $\ip{x}{y}=x^\top y$ for vectors and $\ip{A}{B}=\operatorname{tr}(A^T B)=\sum_{ij}A_{ij}B_{ij}$ for matrices. Our norm conventions are
\begin{align*}
\begin{array}{ll}
\norm{A}_{F}=(\sum_{ij}A_{ij}^2)^{1/2},&\text{frobenius norm},\\[3pt]
\norm{A}_{\mathrm{op}}=\sigma_{\max}(A),&\text{spectral/operator norm},\\[3pt]
\norm{A}_{*}=\sum_j\sigma_j(A),&\text{nuclear norm},\\[3pt]
\norm{A}_{\infty}=\max_{ij}|A_{ij}|,&\text{entrywise infinity norm},\\[3pt]
\norm{A}_{1}=\sum_{ij}|A_{ij}|,&\text{entrywise one-norm},\\[3pt]
\RMS(x)=\norm{x}_2/\sqrt p\quad(x\in\R^p),&\text{root-mean-square vector norm}.
\end{array}
\end{align*}
RMS is a scaled Euclidean vector norm, not a matrix operator norm. For $U\in\R^{V\times d}$, write $u_i\in\R^d$ for the transpose of its $i$-th row, so $U_{i,:} = u_i^T$. A row-energy bound $\norm{u_i}_2^2\le d$ is equivalent to $\RMS(u_i)\le1$. For a vector $b$, let $\supp b=\{j:b_j\ne0\}$; throughout, $\sign(0)=0$.

For a general norm $N$, its dual is denoted $N^*$. The notation $N_{\mathrm{tie},c}$ is identical to $\norm{\cdot}_{\mathrm{tie},c}$ in the main text. Subscripts such as $N_{\mathrm{tie},c}^*$ refer to dual norms, whereas $\norm{A}_*$ always denotes the nuclear norm.

We also write $[t]_+=\max\{t,0\}$ for the positive part, $\mathbf 1$ for an all-ones vector, $\partial\psi$ for the convex subdifferential of $\psi$, $\delta_C$ for the indicator of a convex set $C$ (zero on $C$, $+\infty$ elsewhere), $\sigma_C(B)=\sup_{U\in C}\ip{B}{U}$ for its support function, $\sigma_j(A)$ for the $j$-th singular value, and $O(d)$ for the orthogonal group.

\subsection{Dual norms and linear minimization oracles}
For a norm $N$ on a finite-dimensional inner-product space, define
\[
N^*(Z)=\max_{N(U)\le1}\ip{Z}{U},\qquad
D_N(Z)\in\argmax_{N(U)\le1}\ip{Z}{U}.
\]
We call $D_N$ the aligned LMO direction; the actual minimization step is $-\eta D_N(Z)$. Compactness of the unit ball ensures attainment. For $\eta>0$,
\[
\ip{Z}{D_N(Z)}=N^*(Z),\qquad
\ip{Z}{-\eta D_N(Z)}=-\eta N^*(Z)<0\quad(Z\ne0).
\]
This identity concerns the oracle input $Z$. If $Z$ is momentum rather than the current gradient, it does not itself establish descent of the current objective. Also, $D_N(Z)\in\partial N^*(Z)$. Positive rescaling leaves the argmax set unchanged; a scale-independent selection satisfies $D_N(\alpha Z)=D_N(Z)$ for $\alpha>0$. We choose $D_N(0)=0$.

The following are standard norm-ball LMOs (\cite{Tyrrell1970convex} and \cite{boyd2004convex}), not new ingredients of AF-Muon.

\paragraph{1. Spectral norm for hidden matrices.} Let $Z=P_r\Sigma_rQ_r^\top$ be a compact SVD containing exactly the $r=\operatorname{rank}(Z)$ positive singular values. Then
\[
\max_{\norm{U}_{\mathrm{op}}\le1}\ip{Z}{U}=\norm{Z}_*,\qquad
D_{\mathrm{op}}(Z)=P_rQ_r^\top.
\]
For $Z=0$, this compact polar factor is zero. The upper bound follows from spectral/nuclear duality, and equality follows by substitution. Rank-deficient inputs can have other maximizers, including components on null singular subspaces; the compact polar factor fixes the zero action there. A full thin-SVD completion need not make this same selection. Muon's finite Newton--Schulz implementation approximates this ideal direction and requires separate feasibility and accuracy control (Appendix \ref{app:sec-convergenceAFMuon}).

\paragraph{2. RMS norm for vectors.} For $u\in\R^p$, let $N_{\mathrm{RMS}}(u)=\norm{u}_2/\sqrt p$. Cauchy--Schwarz gives
\[
N_{\mathrm{RMS}}^*(b)=\sqrt p\norm{b}_2,\qquad
D_{\mathrm{RMS}}(b)=
\begin{cases}b/\RMS(b),&b\ne0,\\0,&b=0.\end{cases}
\]
This optimizer direction follows from the chosen norm; it is distinct from the RMSNorm architectural operation of \cite{rootmeannormZhang2019}.

 \paragraph{3. Max-entry norm.} For matrices $U$, let $N_\infty(U)=\norm{U}_\infty$. The dual norm is the entrywise $\ell_1$     norm, $$N_\infty^*(B)=\norm{B}_1,$$
    and a corresponding maximizing direction is
    $$D_\infty(B)=\sign(B),$$
    applied entrywise. Each nonzero coordinate is maximized by choosing its sign, while zero coordinates are set to zero by convention. This yields the Sign endpoint.
%For $N_\infty(U)=\norm{U}_\infty$,
% \[
% N_\infty^*(B)=\norm{B}_1,\qquad D_\infty(B)=\sign(B)
% \]
% entrywise. Each nonzero coordinate is maximized at its sign; zero coordinates are set to zero by convention. This is the Sign endpoint.

These maps are steepest directions for different geometries, not interchangeable normalizations of a single norm. Norm-dependent layer assignments are studied by \cite{pethick2025scion}, \cite{large2024scalable}, and \cite{bernstein2024oldnorm,bernstein2025modular}. The construction below combines selected role constraints when several roles share one stored tensor.

%=============================
% Second Section
%=============================

\section{The alias-pullback framework}
\label{app:sec_aliaspulllback}
In an untied model, the input table and output head can have distinct updates. Under hard tying, autograd accumulates their contributions into one gradient on one parameter. Separate role constraints must therefore act on the same stored perturbation. Pullbacks, intersections, support functions, and infimal convolution are standard convex-analysis constructions (\cite{Tyrrell1970convex}); here they organize the shared-parameter geometry.

\subsection{Where does the terminology come from?}
A \emph{pullback} transfers a structure on a target space to a source through a map. Here a linear role map $A_r:\mathcal W\to\mathcal X_r$ transfers a norm constraint from the role space to the stored-parameter space.\\
An \emph{alias}, in programming, is another reference to the same underlying object. In this setting, the roles alias one stored tensor: changing that tensor changes every role view through its corresponding $A_r$.\\
The resulting feasible set is an intersection: every role constraint must hold simultaneously. Its dual admits a decomposition into latent role covectors. This decomposition is geometric and need not coincide with the separately accumulated computational-path gradients.

We therefore call the resulting construction the \emph{alias-pullback geometry}: it is the pullback to the shared, aliased parameter space of the geometries imposed by its multiple functional roles. Its unit ball is naturally the intersection of the role-wise feasible sets, so a stored update is feasible only when it is simultaneously feasible in every role. Dually, the merged covector can be represented by an optimal decomposition into latent role covectors, which explains the infimal-convolution expression for the dual norm. Finally, any redundancy of the shared parameter must also be compatible with every role, leading to an intersection of the role-wise redundancy subspaces.

\subsection{Pullback norm and intersection ball}
Let \(R\ge 1\). For each \(r\in\{1,\dots,R\}\), let \((\mathcal X_r,N_r)\) be a normed space and let \(\rho_r>0\) be the corresponding role radius. Define the scaled product norm
\[
N_\oplus(x_1,\dots,x_R)
=
\max_{1\le r\le R}\frac{N_r(x_r)}{\rho_r},
\qquad x_r\in\mathcal X_r.
\]
Given linear role maps \(A_r:\mathcal W\to\mathcal X_r\), define the alias seminorm
\[
N_{\text{alias}}(U)
=
\max_{r}\frac{N_r(A_rU)}{\rho_r},
\qquad U\in\mathcal W.
\]
Equivalently, with $A=(A_1,\dots,A_R):\mathcal W\to \prod_{r=1}^R\mathcal X_r$, we have
\begin{align*}
    N_{\text{alias}} = N_\oplus\circ A.
\end{align*}
Note that $N_{\text{alias}}$ is a semi-norm on \(\mathcal W\); it is a norm precisely when $\bigcap_{r=1}^R \ker A_r=\{0\}$.

% Let $R\ge1$, let each $\mathcal X_r$ carry a norm $N_r$ and radius $\rho_r>0$, and define
% \[
% N_\oplus(X_1,\ldots,X_R)=\max_r\frac{N_r(X_r)}{\rho_r},\qquad
% \alias(U)=\max_r\frac{N_r(A_rU)}{\rho_r}.
% \]
\begin{theorem}[Alias-pullback norm and intersection ball]
\label{app:the_aliaspullbacknorm}
If $$\bigcap_{r=1}^R\ker A_r=\{0\},$$ then $\alias$ is a norm and
\[
\mathcal C_{\mathrm{alias}}=\{U:\alias(U)\le1\}
=\bigcap_{r=1}^R A_r^{-1}\{X:N_r(X)\le\rho_r\}.
\]
\end{theorem}
\begin{proof}
Positive homogeneity and nonnegativity follow immediately from the corresponding properties of the norms $N_r$. Further,
\[
\alias(U+V)\le\max_r\left[\frac{N_r(A_rU)}{\rho_r}+\frac{N_r(A_rV)}{\rho_r}\right]
\le\alias(U)+\alias(V).
\]
If $\alias(U)=0$, then $U$ belongs to every $\ker A_r$, hence $U=0$. The ball identity follows directly from the maximum. Without the common-kernel assumption, this is a semi-norm, becoming a norm on the quotient by $\bigcap_r\ker A_r$.
\end{proof}

\subsection{Dual of the alias geometry}
The dual evaluates a covector on the stored space by finding its cheapest representation through role covectors. Adjoint maps are taken with respect to the specified inner products.

\begin{theorem}[Alias-aware dual norm]
\label{app:the_alias_aware_dual}
    Let $\mathcal W$ and $\mathcal X_r$ be finite-dimensional inner-product spaces, $A_r:\mathcal W\to\mathcal X_r$ linear, $N_r$ norms, and $\rho_r>0$. Assume explicitly that $\bigcap_r\ker A_r=\{0\}$. Then for every $B\in\mathcal W$,
    \[
    N_{\mathrm{alias}}^*(B)=
    \min_{\sum_{r=1}^R A_r^*Z_r=B}\sum_{r=1}^R\rho_r N_r^*(Z_r).
    \]
    The minimum is feasible and attained.
\end{theorem}
\begin{proof}
    Let $\mathcal C_r=\{U:N_r(A_rU)\le\rho_r\}$. Write $\delta_C$ for the convex indicator of $C$ and $\sigma_C(B)=\sup_{U\in C}\ip{B}{U}$ for its support function. Since
    \[
    \delta_{\mathcal C_{\mathrm{alias}}}=\sum_r\delta_{\mathcal C_r},
    \]
    conjugacy gives the infimal-convolution formula. Each $\mathcal C_r$ contains a neighborhood of zero, even when $A_r$ is not injective: continuity of $N_r\circ A_r$ and $\rho_r>0$ ensure this. The common-interior qualification therefore gives exactness and attainment at finite values:
    \begin{align}
    \label{app:eq_infimal-convolution}
    \sigma_{\mathcal C_{\mathrm{alias}}}(B)=
    \min_{\sum_r B_r=B}\sum_r\sigma_{\mathcal C_r}(B_r).
    %\tag{7}
    \end{align}
    For $\mathcal D_r=\{X:N_r(X)\le\rho_r\}$, $\delta_{\mathcal C_r}=\delta_{\mathcal D_r}\circ A_r$. The same continuity qualification at $A_r0=0$ yields
    \begin{align}
    \label{app:eq_infimal-convolution-8}
    \sigma_{\mathcal C_r}(B_r)=
    \inf_{A_r^*Z_r=B_r}\sigma_{\mathcal D_r}(Z_r)
    =\inf_{A_r^*Z_r=B_r}\rho_rN_r^*(Z_r).
    %\tag{8}
    \end{align}
    If the constraint is infeasible, both sides equal $+\infty$, using $\inf\varnothing=+\infty$. Otherwise the infimum is a minimum, since a norm is coercive on the closed affine feasible set. Substitution into (\ref{app:eq_infimal-convolution}) gives the stated aggregate constraint.\\
    For completeness, define $A U=(A_1U,\ldots,A_RU)$. The common-kernel assumption says $A$ is injective, so its adjoint $A^*$ is surjective in finite dimensions. Thus $\sum_r A_r^*Z_r=B$ is feasible for every $B$. The objective $\sum_r\rho_rN_r^*(Z_r)$ is coercive on the product space, ensuring attainment. Finally $\sigma_{\mathcal C_{\mathrm{alias}}}=N_{\mathrm{alias}}^*$ by definition.
\end{proof}
\paragraph{Semi-norm case.} If $K=\bigcap_r\ker A_r\ne\{0\}$, the same support-function identity holds with extended-real values. It is finite precisely for $B\in K^\perp=\operatorname{range}(A^*)$, where it describes the quotient dual. It is not a finite norm on all of $\mathcal W$.

The minimizing $Z_r$ represent an optimal geometric decomposition of a merged gradient or momentum. They are not required to equal the actual gradients produced by the individual roles, and the oracle need not form or store them.

\subsection{Redundancy under aliasing}
Role-wise invariance must be distinguished from invariance of the composed network. Suppose $K_r\subseteq\mathcal X_r$ is a linear subspace of perturbations that leave role $r$ unchanged, with its other arguments held fixed. For exact finite invariance, require this property along every translation in $K_r$ on the domain being considered. If $K_r$ instead denotes a Jacobian kernel, the resulting statement is only infinitesimal and local.

\begin{theorem}[Simultaneous role-wise redundancy is an intersection]
\label{app:the_redundancy_intersection}
For linear role maps $A_r$, the stored perturbations that are redundant in every role are exactly
\[
K_{\mathrm{alias}}:=\{U:A_rU\in K_r\ \text{for all }r\}
=\bigcap_{r=1}^R A_r^{-1}(K_r).
\]
\end{theorem}
\begin{proof}
    A perturbation is role-wise redundant exactly when each role view lies in that role's redundancy subspace. Membership in all preimages is equivalent to membership in their intersection.
\end{proof} 

This is a characterization of \emph{simultaneous role-wise} redundancy. Such redundancy is sufficient for invariance of a network composed from those roles, but it need not exhaust the network's invariances. Downstream maps can discard changes, and effects through several paths can cancel. In particular, $K_{\mathrm{alias}}=\{0\}$ does not prove that the full network has no symmetries, parameter degeneracies, or loss-flat directions.
\paragraph{Quotient geometry requires a further distinction.} If the objective is invariant along a fixed subspace $K$, the norm induced on equivalence classes is
\[
\overline N([U])=\inf_{k\in K}\alias(U+k).
\]
Identify the quotient with $K^\perp$ using the orthogonal projection $P_{K^\perp}$. Its unit ball is the \emph{projection}
\[
P_{K^\perp}\mathcal C_{\mathrm{alias}},
\]
not generally the section $\mathcal C_{\mathrm{alias}}\cap K^\perp$. For $B\in K^\perp$, an oracle over the projected ball has the same optimal linear value as the original ball. An oracle restricted to the section is a different, potentially more restrictive choice unless projection is nonexpansive in $\alias$.

For example, in $\R^2$ let $N(x,y)=\sqrt{x^2+(y-x)^2}$ and $K=\operatorname{span}(e_2)$. The unit ball projects onto $|x|\le1$, whereas its section at $y=0$ is $|x|\le1/\sqrt2$. Thus orthogonal restriction can reduce the optimum even for a covector orthogonal to the redundant subspace.

These distinctions matter for an untied softmax head, whose common-logit-shift invariance is studied in symmetry-compatible optimizer design (\cite{lau2026symmetry}). A direction that is redundant for the head alone can change the lookup role of a tied table. Appendix \ref{app:sec-hardtiedtables} establishes precisely that role-wise obstruction; it makes no claim to classify every symmetry of a complete transformer.

%=============================
% Third Section
%=============================

\section{Hard-tied token tables}
\label{app:sec-hardtiedtables}

\subsection{Gradient structure of the tied table}
\label{app:gradientStructure}
Hard tying makes autograd add all computational-path contributions to the same stored gradient. The following formulas use a summed cross-entropy loss. For a weighted or averaged loss, the corresponding position weights or normalization factors are included in each summand. Lookup and output positions can have different index sets.

\begin{proposition}[Sparse lookup and generically dense output contributions]
\label{app:pro_sparaseinput_dense}
For lookup positions $t$, token indices $i_t$, and upstream embedding gradients $g_{x,t}=\nabla_{x_t}\mathcal L\in\R^d$,
\[
G_{\mathrm{in}}=\sum_{t\in\mathcal T_{\mathrm{in}}}e_{i_t}g_{x,t}^\top.
\]
Its nonzero rows are contained in the set of token indices used by the lookup. For output positions with $z_t=Eh_t$, $p_t=\operatorname{softmax}(z_t)$, and one-hot target $y_t$,
\[
G_{\mathrm{out}}=\sum_{t\in\mathcal T_{\mathrm{out}}}(p_t-y_t)h_t^\top,
\qquad G_E=G_{\mathrm{in}}+G_{\mathrm{out}}.
\]
The output formula is the direct head contribution, holding $h_t$ fixed in this partial derivative. Dependence of $h_t$ on earlier uses of $E$ is already represented by the lookup-path gradients.
\end{proposition}
\begin{proof}
At a lookup, $x_t=E^\top e_{i_t}$ and $dx_t=(dE)^\top e_{i_t}$. Therefore
\begin{align*}
    d\mathcal L_{\mathrm{lookup},t}
=g_{x,t}^\top(dE)^\top e_{i_t}
=\ip{e_{i_t}g_{x,t}^\top}{dE}.
\end{align*}
At the head, $\nabla_{z_t}\ell_t=p_t-y_t$ and $dz_t=(dE)h_t$, so
\begin{align*}
    d\mathcal L_{\mathrm{head},t}
=(p_t-y_t)^\top(dE)h_t
=\ip{(p_t-y_t)h_t^\top}{dE}.
\end{align*}
Summing the differentials over all uses of the stored parameter proves the identity.
\end{proof}

For $V>1$, finite logits, and a full softmax, each $0<p_{t,i}<1$. Thus a nonzero hidden state yields a nonzero vector in every vocabulary row of an individual head contribution. Summed contributions may cancel; individual hidden coordinates may vanish; masking, alternative losses, and numerical underflow can alter support. Consequently full-softmax training does not guarantee that every coordinate of the merged gradient or momentum is nonzero. The support-aware oracle in Appendix \ref{app:sec_exactClippedRow} covers these cases.

\subsection{Why the untied-head quotient disappears}
For an untied output head, adding $\mathbf1a^\top$ adds the common scalar $a^\top h$ to every logit. Softmax probabilities are unchanged. When hidden states range over $\R^d$,
\begin{align}
   K_{\mathrm{out}}=\{\mathbf1a^\top:a\in\R^d\} 
\end{align}
is exactly this head-shift subspace. Restricted hidden-state sets can introduce additional degeneracies. The following result concerns the specified head-shift symmetry and the full lookup map.

\begin{proposition}[Hard tying removes simultaneous head-shift redundancy]
\label{app:pro_remove-shift-redundancy}
Assume the input role map is injective, so $K_{\mathrm{in}}=\{0\}$. The output role, under the softmax loss, has head-shift redundancy: for every $a\in\mathbb R^d$, the perturbation $\mathbf1 a^T$ leaves the output probabilities unchanged to first order, so $\mathbf1 a^\top\in K_{\mathrm{out}}$. Then
\begin{align*}
K_{\mathrm{alias}}
=
K_{\mathrm{in}}\cap K_{\mathrm{out}}
=
\{0\}.
\end{align*}
In particular, a nonzero head-shift direction \(\mathbf1 a^\top\) is redundant for the output role alone but not simultaneously for the two roles.
\end{proposition}
\begin{proof} 
By Theorem~\ref{app:the_redundancy_intersection},
\(K_{\mathrm{alias}}=K_{\mathrm{in}}\cap K_{\mathrm{out}}\). Since
\(K_{\mathrm{in}}=\{0\}\), the intersection is \(\{0\}\). Directly, replacing
\(E\) by \(E+\mathbf1 a^\top\) changes every lookup vector by \(a\), which the
input role detects. Hence such a perturbation is not redundant once the two
roles are tied.
% Apply Theorem \ref{app:the_redundancy_intersection}. Directly, replacing $E$ by $E+\mathbf1a^\top$ changes every lookup vector by $a$, which the input role detects.
\end{proof}

\begin{proposition}[Row-sum identities]
Under the loss convention of Proposition \ref{app:pro_sparaseinput_dense},
\[
\mathbf1^\top G_{\mathrm{out}}=0,\qquad
\mathbf1^\top G_{\mathrm{in}}=\sum_{t\in\mathcal T_{\mathrm{in}}}g_{x,t}^\top.
\]
\end{proposition}
\begin{proof}
For every output position $t \in \mathcal{T}_{\mathrm{out}}$, the probability distribution $p_t$ and target distribution $y_t$ both sum to one, so $\mathbf{1}^\top (p_t - y_t) = 1 - 1 = 0$. For every input lookup, the basis vector satisfies $\mathbf{1}^\top e_{i_t} = 1$. Summing these identities over their respective position sets yields the result.
\end{proof}

The second sum need not vanish. Hence, projecting the merged tied gradient with $\Pi_\perp=I_V-\mathbf1\mathbf1^\top/V$ can discard a legitimate input-role contribution. The same issue applies to accumulated tied momentum. AF-Muon therefore uses the merged, unprojected buffer. This conclusion does not exclude an architecture-specific invariance that must be established separately for the composed network.

\subsection{The tied intersection norm}

It is worth stating why both role maps are the identity, given that the two roles act on $E$ differently. The alias framework constrains a \emph{perturbation} $U=dE$ of the stored table, not the forward map. A lookup perturbs the embedding of token $i$ by $U^\top e_i=u_i$, and the head perturbs the logits at hidden state $h$ by $Uh$; both are linear images of the same $U$ and both views live in the same space of $V\times d$ perturbations. Hence $A_{\mathrm{in}}=A_{\mathrm{out}}=\mathrm{Id}$, and the entire role distinction is carried by the \emph{norms} $N_{\mathrm{in}}$ and $N_{\mathrm{out},c}$ rather than by the maps; the common-kernel condition of Theorem \ref{app:the_aliaspullbacknorm} is then
immediate. 
Appendix \ref{app:sec_aliaspulllback} is nevertheless stated for general $A_r$, since nontrivial role maps arise elsewhere, for instance a factorized or projected head $U\mapsto UP$ or a partially shared table, and since Theorem \ref{app:the_redundancy_intersection} is stated through preimages.

AF-Muon chooses
\begin{align*}
    N_{\mathrm{in}}(U)=\max_i\frac{\norm{u_i}_2}{\sqrt d},\qquad
    N_{\mathrm{out},c}(U)=\frac{\norm{U}_\infty}{c},\quad 1\le c<\infty.
\end{align*}
Each norm bounds the quantity its role feeds downstream.
On the lookup side, a feasible perturbation moves every token embedding by at most
$\RMS(u_i)\le1$: a bounded root-mean-square displacement per token, uniform over the vocabulary
and independent of which tokens the batch happens to contain. On the output side, the logit
perturbation at hidden state $h$ is $Uh$, and for $U\in\mathcal C_c$,
\[
\norm{Uh}_\infty=\max_i|\ip{u_i}{h}|\le
\min\bigl\{\,c\norm{h}_1,\ \sqrt d\,\norm{h}_2\,\bigr\}.
\]
The two bounds are not comparable: the cap supplies the tighter one exactly when
$\norm{h}_1/\norm{h}_2<\sqrt d/c$, that is, when the hidden state is concentrated, which is the
regime where a single coordinate of an update row would otherwise dominate a logit direction.
Neither constraint implies the other in the range $1<c<\sqrt d$, and $c$ sets their relative
strength. Figure \ref{fig:overview} shows the resulting feasible sets.

The diagonal alias map $U\mapsto(U,U)$ gives
\begin{equation}
\label{app:eq_tiehardnorm}
\tie(U)=\max\left\{\max_i\frac{\norm{u_i}_2}{\sqrt d},\frac{\norm{U}_\infty}{c}\right\}.
%\tag{9}
\end{equation}
\begin{proposition}[Hard-tied intersection geometry]
The functional in (\ref{app:eq_tiehardnorm}) is a norm with unit ball
\begin{equation}
\label{app:eq_tiehardball}
\mathcal C_c=\{U:\norm{u_i}_2\le\sqrt d\ \forall i,\ |U_{ij}|\le c\ \forall i,j\}.
%\tag{10}
\end{equation}
\end{proposition}
\begin{proof}
    Apply Theorem \ref{app:the_aliaspullbacknorm} with identity role maps, norms $N_{\mathrm{in}}$ and $\norm{\cdot}_\infty$, and radii $1$ and $c$. The common kernel is zero. Expanding the two constraints gives (\ref{app:eq_tiehardball}).
\end{proof}
The first constraint bounds row energy; the second caps coordinate magnitude. At $c=1$ the row constraint is redundant. For $c\ge\sqrt d$ the coordinate constraint is redundant. We use $c=\infty$ as shorthand for omitting the latter, with $N_{\mathrm{tie},\infty}=N_{\mathrm{in}}$. The cap changes shape only in the nontrivial range $1\le c<\sqrt d$; an external positive scalar changes the overall radius.\\

\section{Exact support-aware clipped-row oracle}
\label{app:sec_exactClippedRow}
The ball \ref{app:eq_tiehardball} is a Cartesian product across vocabulary rows, so it suffices to solve
\begin{equation}
\label{app:eq_optimizationprob}
\max_{u\in\R^d}\ip{b}{u}\quad\text{subject to}\quad\norm{u}_2\le\sqrt d,\; \norm{u}_\infty\le c.
%\tag{11}
\end{equation}
Let $S(b)=\supp b$, $k(b)=|S(b)|$, and $r_c^2(b)=\min\{d,c^2k(b)\}$. For $b\ne0$, choose $\tau_c(b)\ge0$ satisfying
\begin{equation}
\label{app:eq_conditiontau}
\sum_{j=1}^d\min\{c^2,\tau_c(b)^2b_j^2\}=r_c^2(b),
%\tag{12}
\end{equation}
and set $\tau_c(0)=0$. Define
\begin{equation}
\label{app:eq_LC}
[\LC(b)]_j=\sign(b_j)\min\{c,\tau_c(b)|b_j|\}.
%\tag{13}
\end{equation}
When $c^2k(b)\le d$ and $b\ne0$, every $\tau\ge c/\min_{j\in S(b)}|b_j|$ solves \ref{app:eq_conditiontau} and produces the same direction. When $c^2k(b)>d$, the root is unique.
\\
% \begin{definition}[Support-preserving maximizer]
% \label{app:def_support_preserving}
% Call a feasible $u$ for optimization problem \ref{app:eq_optimizationprob} \emph{support-preserving} if $\supp u=\supp b$; equivalently, $u$ places no energy on coordinates where the input vanishes. Uniqueness in the following Theorem \ref{app:the_exactsupportware} is asserted inside this class, and the restriction is not vacuous: when $c^2k(b)<d$ the supported coordinates do not exhaust the Euclidean budget, so any vector supported off $\supp b$ with small enough norm may be added without changing the objective or violating feasibility. All maximizers agree on $\supp b$;
% the canonical selection (\ref{app:eq_LC}) is the one that zeroes the remainder. This convention is what makes $\LC$ single-valued as a map, keeps it odd and positively scale invariant, and
% turns Theorem \ref{app:the-SignedPermutationEquivariance} into an equality of vectors rather than of argmax sets.
% \end{definition}
\begin{definition}[Support-preserving maximizer]
\label{app}
A feasible vector \(u\) for \ref{app:eq_optimizationprob} is called \emph{support-preserving} if $\supp u=\supp b$. Equivalently, $u$ places no energy on coordinates where the input vanishes. \\
When $c^2k(b)<d$, the supported coordinates do not exhaust the Euclidean budget, so any vector supported off $\supp b$ with small enough norm may be added without changing the objective or violating feasibility. All maximizers agree on $\supp b$;
the canonical selection (\ref{app:eq_LC}) is the one that zeroes the remainder. This convention is what makes $\LC$ single-valued as a map, keeps it odd and positively scale invariant.
\end{definition}
\vspace{0.2em}
\begin{theorem}[Exact support-aware clipped-row LMO]
%\label{app:A9}
\label{app:the_exactsupportware}
For every finite $c\ge1$, $\LC(b)$ maximizes objective \ref{app:eq_optimizationprob}. It is the unique support-preserving maximizer and satisfies
\begin{align}
\label{app:eq_supportLC}
&\supp\LC(b)=\supp b, \qquad \norm{\LC(b)}_2^2=\min\{d,c^2k(b)\}.
%\tag{14}
\end{align}
The zero input is interpreted as $\LC(0)=0$. More explicitly,
\begin{equation}
\label{app:eq_LC_sign}
c^2k(b)\le d\quad\Longrightarrow\quad\LC(b)=c\sign(b),
%\tag{15}
\end{equation}
whereas in the radius-limited case,
\begin{equation}
\label{app:eq_radiusLimitLC}
[\LC(b)]_j=\sign(b_j)\min\{c,\tau|b_j|\},\qquad
\sum_j\min\{c^2,\tau^2b_j^2\}=d.
%\tag{16}
\end{equation}
\end{theorem}
\begin{proof}
     For $b=0$, zero is the unique support-preserving feasible vector and is optimal. Let $b\ne0$. Aligning each supported coordinate with $b_j$ cannot decrease the objective and preserves feasibility. Setting coordinates outside $S(b)$ to zero also preserves feasibility and the objective. Write $q_j=|b_j|>0$ on $S(b)$ and $u_j=\sign(b_j)a_j$. The supported problem is
\begin{equation}
\label{app:eq_supportedprob}
\max_a\sum_{j\in S(b)}q_ja_j\quad\text{subject to}\quad
\sum_{j\in S(b)}a_j^2\le d,\quad0\le a_j\le c.
%\tag{17}
\end{equation}
If $c^2k(b)\le d$, the vector $a_j=c$ is feasible and maximizes every positive term individually, proving equation \ref{app:eq_LC_sign}. These supported coordinates are uniquely determined. If $c^2k(b)<d$ and zero coordinates exist, the unrestricted problem can also have off-support components using the remaining radius; the canonical selection excludes them. If $c^2k(b)=d$, no off-support energy remains.

If $c^2k(b)>d$, the Euclidean constraint must be active; otherwise at least one coordinate is below its cap and can be increased slightly to improve the objective while remaining feasible. We establish the unique optimizer for this case next.

Problem \ref{app:eq_supportedprob}, with negated objective, is convex and strictly feasible by taking all supported $a_j$ sufficiently small and positive. KKT conditions are therefore necessary and sufficient. With multipliers $\lambda\ge0$ for $\sum a_j^2\le d$, $\alpha_j\ge0$ for $-a_j\le0$, and $\beta_j\ge0$ for $a_j-c\le0$, stationarity is
\begin{equation}
\label{app:eq_stationarycondition}
-q_j+2\lambda a_j-\alpha_j+\beta_j=0.
%\tag{18}
\end{equation}
No optimal $a_j$ can be zero: then $\beta_j=0$ and \ref{app:eq_stationarycondition} would imply $-q_j-\alpha_j=0$. Thus $a_j>0$ and $\alpha_j=0$. Also $\lambda>0$: if $\lambda=0$, then $\beta_j=q_j>0$ forces all $a_j=c$, contradicting $c^2k(b)>d$.

For $a_j<c$, equation \ref{app:eq_stationarycondition} gives $a_j=q_j/(2\lambda)$. For $a_j=c$, it gives $q_j/(2\lambda)\ge c$. Hence
\begin{equation}
\label{app:eq_aj}
a_j=\min\{c,\tau q_j\},\qquad\tau=(2\lambda)^{-1}>0.
%\tag{19}
\end{equation}
The active radius constraint is \ref{app:eq_radiusLimitLC}. The function $\phi(\tau)=\sum_{j\in S(b)}\min\{c^2,\tau^2q_j^2\}$ is continuous, equals zero at zero, and reaches $c^2k(b)>d$. It is strictly increasing before all terms saturate: if $0\le\tau_1<\tau_2$ and $\phi(\tau_1)<c^2k(b)$, at least one unsaturated term increases. Thus its level-$d$ root exists and is unique. Equations \ref{app:eq_stationarycondition}--\ref{app:eq_aj} and complementary slackness certify optimality. All supported coordinates are positive, and their squared norm is $d$; off-support motion is then impossible. This proves \ref{app:eq_supportLC} and uniqueness.
\end{proof}

\begin{corollary}[Full-table LMO]
 Set $[D_{\mathrm{tie},c}(B)]_{i,:}=\LC(b_i)^\top$. Then
\begin{align*}
    D_{\mathrm{tie},c}(B)\in\argmax_{\tie(U)\le1}\ip{B}{U}. 
\end{align*}
\end{corollary}
\begin{proof}
    Both feasibility and the objective $\ip{B}{U}=\sum_i\ip{b_i}{u_i}$ separate by row.
\end{proof}
\begin{corollary}[Endpoints]
    With $\sign(0)=0$, $L_1(b)=\sign(b)$. For $c\ge\sqrt d$ (including the shorthand $c=\infty$),
    \begin{align*}
        \LC(b)=\begin{cases}\sqrt d\,b/\norm{b}_2=b/\RMS(b),&b\ne0,\\0,&b=0.\end{cases}
    \end{align*}
\end{corollary}
\begin{proof}
    At $c=1$, the box already implies the row-radius bound. At $c\ge\sqrt d$, $\norm{u}_\infty\le\norm{u}_2\le\sqrt d\le c$, so the cap is redundant.\\
    The canonical map is odd and positively scale invariant: $\LC(-b)=-\LC(b)$ and $\LC(\alpha b)=\LC(b)$ for $\alpha>0$.
\end{proof}

\textbf{Numerical implementation.} Use \ref{app:eq_LC_sign} directly for support-limited rows. Otherwise bracket the unique root of \ref{app:eq_radiusLimitLC}, doubling an upper endpoint until $\phi(\tau_{\rm hi})\ge d$, and bisect. Returning the lower endpoint preserves feasibility in exact arithmetic. A finite solve is approximate; floating-point overflow, feasibility, and objective error must be controlled rather than inferred from Theorem \ref{app:the_exactsupportware}. Appendix \ref{app:sec-convergenceAFMuon} states the general framework into which these errors enter.

\paragraph{The tied update in one display.} Collecting the pieces, the stored table is updated from a single merged buffer $B_t$ by
\[
E_{t+1}=E_t-\eta_{\mathrm{tie},t}\,D_{\mathrm{tie},c}(B_t),
\qquad
[D_{\mathrm{tie},c}(B_t)]_{i,:}=\LC(b_{t,i})^\top ,
\]
with $\LC$ as in (\ref{app:eq_LC}). There is no projection, no split of $B_t$ into lookup and head parts, and no second buffer; the role structure enters only through the feasible set $\mathcal C_c$ over which $\LC$ optimizes. Rows with $b_{t,i}=0$ receive no update, which is the support-awareness distinguishing $\LC$ from a row-RMS normalization applied after densification.

%=============================
% Fifth Section
%=============================

\section{Dual geometry, objective retention, and concentration control}
\label{app:sec_geo-rete-concen}
We characterize the support function optimized by the exact oracle, its same-input objective relative to row-RMS, and its support-dependent concentration bound. The convex-duality identities used here are classical; the formulas below specialize them to the chosen tied-table geometry.

\subsection{The tied dual norm as an infimal convolution}
The row norm $N_{\mathrm{in}}$ has dual $N_{\mathrm{in}}^*(X)=\sqrt d\sum_{i=1}^V\norm{x_i}_2$, since its ball is a product of radius-$\sqrt d$ Euclidean balls. The dual of $\norm{\cdot}_\infty/c$ is $c\norm{\cdot}_1$.

\begin{theorem}[Tied dual as infimal convolution]
For $B\in\R^{V\times d}$ and finite $c\ge1$,
\begin{equation}
\label{app:eq_dual_infimal_convolution}
N_{\mathrm{tie},c}^*(B)=\min_{X+Y=B}
\left[\sqrt d\sum_{i=1}^V\norm{x_i}_2+c\norm{Y}_1\right].
%\tag{20}
\end{equation}
\end{theorem}
\begin{proof}
Apply Theorem \ref{app:the_alias_aware_dual} with identity role maps, norms $N_{\mathrm{in}}$ and $\norm{\cdot}_\infty$, and radii $1$ and $c$. Its constraint is $X+Y=B$, and attainment follows from that theorem.
\end{proof}

The tied geometry prices the merged momentum by its cheapest split into a row-structured component \(X\) (group-\(\ell_{2,1}\) cost) and an entrywise-spiky component \(Y\) (\(\ell_1\) cost). This is an implicit geometric decomposition: AF-Muon never stores or solves for \(X, Y\), and does not claim they equal the actual lookup and softmax gradients; the primal clipped-row oracle attains the same optimal value with one merged buffer. The infimal-convolution identity itself is classical (\cite{Tyrrell1970convex}); the novelty is its role as the exact dual of the tied-table geometry.

\subsection{Explicit row support function and objective retention}
Define the row support function 
$$h_c(b) := \max\{\langle b, u\rangle : \|u\|_2 \le \sqrt d, \|u\|_\infty \le c\};$$
by row-Cartesian structure, 
$$N_{\mathrm{tie},c}^*(B) = \sum_{i=1}^{V} h_c(b_i).$$
With the capped active set \(A = \{j : \tau_c(b)|b_j| \ge c\}\), Theorem \ref{app:the_exactsupportware} evaluates it in closed form: in a radius-limited row,
\[
\tau_c(b) = \sqrt{\frac{r_c^2(b) - c^2 |A|}{\sum_{j \notin A} b_j^2}},
\qquad
h_c(b) = c \sum_{j \in A} |b_j| + \tau_c(b) \sum_{j \notin A} b_j^2 ;
\]
in the support-limited case \(c^2 k(b) \le d\), \(h_c(b) = c\,\|b\|_1\). At the endpoints, \(h_1(b) = \|b\|_1\) and \(h_\infty(b) = \sqrt d\, \|b\|_2\), matching the Sign and row-RMS duals.

\paragraph{Determining the active set.} The pair $(A,\tau_c(b))$ above is characterized jointly, not sequentially: $A$ is defined through $\tau_c(b)$ and $\tau_c(b)$ through $A$. In a radius-limited row, they can be identified in closed form by sorting. 
Since $a_j=\min\{c,\tau q_j\}$ in (\ref{app:eq_aj}) saturates exactly when $q_j\ge c/\tau$, the capped coordinates are the ones of largest magnitude. Order the support so that $|b_{(1)}|\ge\cdots\ge|b_{(k)}|$, so that $A=\{(1),\dots,(m)\}$ for some $0\le m<k$, and for each candidate $m$ put
\[
\tau^{(m)}=\sqrt{\frac{d-c^2m}{\sum_{j>m}b_{(j)}^2}} .
\] 
Accept the unique $m$ satisfying $\tau^{(m)}|b_{(m)}|\ge c>\tau^{(m)}|b_{(m+1)}|$, with the conventions $|b_{(0)}|=+\infty$ and $|b_{(k+1)}|=0$. 
This is an exact $O(d\log d)$ alternative to the bisection of Appendix \ref{app:sec_exactClippedRow}; bisection is preferred in the implementation because it batches across rows and avoids the sort.

% Define the same-state \textbf{row-RMS objective retention}
% \[
% R_c(B) := \frac{\langle B,\, D_{\mathrm{tie},c}(B)\rangle}{\langle B,\, D_{\mathrm{tie},\infty}(B)\rangle},
% \]
% whenever the denominator is nonzero.

\begin{proposition}[Objective retention is a dual-norm ratio]
\label{app:prop-objectiveRetentionDualNorm}
For $B\ne0$, the same-input row-RMS retention is
\[
\mathcal R_c(B):=\frac{\ip{B}{D_{\mathrm{tie},c}(B)}}{\ip{B}{D_{\mathrm{tie},\infty}(B)}}
=\frac{N_{\mathrm{tie},c}^*(B)}{\sqrt d\sum_i\norm{b_i}_2}.
\]
\end{proposition}
\begin{proof}
The numerator is the dual norm by the exact LMO identity: 
$$\langle B, D_{\mathrm{tie},c}(B)\rangle = N_{\mathrm{tie},c}^*(B).$$ At \(c = \infty\) the rows decouple into radius-\(\sqrt d\) Euclidean balls, so $$\langle B, D_{\mathrm{tie},\infty}(B)\rangle = \sqrt d \sum_i \|b_i\|_2.$$
\end{proof}

If $c_1\le c_2$, then $\mathcal C_{c_1}\subseteq\mathcal C_{c_2}$, so retention is nondecreasing in $c$. More precisely,
\[
\frac1{\sqrt d}\le\mathcal R_1(B)=\frac{\norm{B}_1}{\sqrt d\sum_i\norm{b_i}_2}
\le\mathcal R_c(B)\le1.
\]
These are fixed-input geometric comparisons, not predictions of future training loss.

\subsection{Support-aware concentration (PAPR) bounds}
For a nonzero row direction, define
\[
\PAPR(u):=\frac{d\norm{u}_\infty^2}{\norm{u}_2^2}.
\]
This ratio is unchanged by nonzero scalar multiplication. It always lies in $[1,d]$; if $u$ has $k$ nonzero entries, it is at least $d/k$.

\textit{Intuition Behind it. }The acronym abbreviates \emph{peak-to-average power ratio}, borrowed from signal processing:
reading $u$ as a length-$d$ signal, $\norm{u}_2^2/d$ is its mean squared coordinate (average power)
and $\norm{u}_\infty^2$ its largest (peak), so $\PAPR(u)$ measures how far one coordinate dominates
the row. A flat row has $\PAPR=1$; a row whose energy sits in a single coordinate has $\PAPR=d$.
This is the quantity the cap is designed to control, since a high-PAPR update row displaces one
embedding coordinate, and correspondingly one logit direction, far more than the row's overall
size suggests.

\begin{proposition}[Support-aware PAPR bound]
\label{app:prop_SupportAwarePAPRBound}
For finite $c\ge1$ and $u=\LC(b)\ne0$,
\begin{equation}
\label{app:ineq_papr}
\PAPR(u)\le\frac{dc^2}{\min\{d,c^2k(b)\}}.
%\tag{21}
\end{equation}
In particular,
\begin{equation}
\label{app:ineq_papr_c}
c^2k(b)\ge d\quad\Longrightarrow\quad\PAPR(u)\le c^2,
%\tag{22}
\end{equation}
where the universal upper bound $d$ can further sharpen this to $\min\{c^2,d\}$. In the support-limited case,
\begin{equation}
\label{app:ineq_papr_sharpen}
c^2k(b)\le d\quad\Longrightarrow\quad\PAPR(u)=\frac{d}{k(b)}.
%\tag{23}
\end{equation}
At $c^2k(b)=d$, the formulas coincide.
\end{proposition}
\begin{proof}
The coordinate constraint gives $\norm{u}_\infty^2\le c^2$ and Theorem \ref{app:the_exactsupportware} gives $\norm{u}_2^2=\min\{d,c^2k(b)\}$. Substitution proves \ref{app:ineq_papr}--\ref{app:ineq_papr_c}. In the support-limited case, precisely $k(b)$ entries have magnitude $c$, giving \ref{app:ineq_papr_sharpen}. The general lower bound follows from $\norm{u}_2^2\le k\norm{u}_\infty^2$.
\end{proof}
The support condition in \ref{app:ineq_papr} is essential. For example, with $d=16$, $c=2$, and $b=e_1$, the canonical update is $2e_1$ and its PAPR is $16$, not at most $c^2=4$. The cap controls absolute coordinate size but cannot create additional supported coordinates. At $c=1$, a dense Sign row has PAPR one; a sparse Sign row has PAPR $d/k(b)$. Row-RMS has
$$
\PAPR(L_\infty(b))=d\frac{\norm{b}_\infty^2}{\norm{b}_2^2},
$$
which reaches $d$ for a one-coordinate input and approaches $d$ for dense inputs dominated by one coordinate.

\textbf{Shape versus scale.} For a positive scalar $s$,
\[
s\mathcal C_c=\{U:\norm{u_i}_2\le s\sqrt d\ \forall i,\ |U_{ij}|\le sc\ \forall i,j\}.
\]
Scaling the completed direction therefore rescales the row radius and cap together, preserving their ratio and PAPR. For a support-sufficient row the scaled energy is $s^2d$; for a support-limited row it is $s^2c^2k(b)$. A uniform scale does not force all nonzero rows to have equal energy when their supports differ.

AF-Muon applies $s$ once through the tied learning rate in main-text equation \ref{eq:tied-lr},
\[
\eta_{\mathrm{tie},t}=\eta_{M,t}\frac{\rho_{\mathrm{out}}}{\rho_{\mathrm{hid}}}\frac{s}{d}.
\]
The width factor and radius ratio are design choices motivated by norm-based scaling and evaluated empirically. Neither their numerical values nor width-independent convergence constants follow from the alias-pullback theorem.

% The support condition is essential: the simplified ceiling \(c^2\) is false for sufficiently sparse rows. At the endpoints, a dense Sign row has PAPR one, whereas row-RMS can have PAPR arbitrarily close to \(d\) as one coordinate dominates.

% The cap \(c\) fixes the normalized feasible-set shape; a scalar \(s\) sets the update magnitude: \(s\,\mathcal{C}_c = \{U : \|u_i\|_2 \le s\sqrt d,\ |U_{ij}| \le sc\}\), so scaling the completed direction rescales row radius and cap together, preserving their ratio. AF-Muon applies \(s\) exactly once, through the tied learning rate \(\eta_{\mathrm{tie},t} = \eta_{M,t} (\rho_{\mathrm{out}}/\rho_{\mathrm{hid}})\, s/d\). The \(1/d\) rule and radius ratios are motivated by SCION/modular-norm scaling and calibrated empirically; they are not derived uniquely from the alias-pullback theorem.

%=============================
% Sixth Section
%=============================

\section{Symmetry and its limitation}
\label{app:sec_symmetry-nogo}

The oracle's symmetry class both certifies its invariances and delimits its relationship to fully rotation-equivariant designs such as the LPRO class of \cite{lau2026symmetry}. The statements here concern equivariance of the oracle as a map on its input. They do not establish that the full network, loss, data, or numerical implementation has the same symmetry. We prove the positive equivariance result, record the exact range of the negative result, and then show that the negative result is structural, not accidental.

\begin{theorem}[Signed-permutation equivariance]
\label{app:the-SignedPermutationEquivariance}
Let \(P \in \mathbb{R}^{V \times V}\) be a vocabulary permutation and \(Q \in \mathbb{R}^{d \times d}\) a signed permutation (hyperoctahedral group \(H_d\)). Then
\[
D_{\mathrm{tie},c}(P B Q^\top) = P\, D_{\mathrm{tie},c}(B)\, Q^\top .
\]
\end{theorem}
\begin{proof}
Row permutations permute the independent row problems. Signed feature permutations preserve Euclidean norms and permute absolute coordinate values, hence preserve the max-entry norm and therefore the feasible set and the Frobenius pairing. They also transport the support set, preserve its size, and leave the radius/root equation unchanged. Thus the canonical support-preserving selection, including zero rows and the support-limited case, commutes with the transformation; \(\operatorname{sign}\) is odd. The displayed equality follows.
\end{proof}

\textit{Exact range of the negative result.} If \(1\le c<\sqrt d\), the oracle is not equivariant under all feature rotations. Because $D_{\mathrm{tie},c}$ acts row by row, it suffices to exhibit the failure for the row map $\LC$; the matrix statement follows by placing the vector in one row of $B$ and zeroing the rest. Put \(q=\mathbf 1/\sqrt d\) and choose \(R\in O(d)\) with \(Re_1=q\). The input $e_1$ has $k=1$, so $c^2k\le d$ and (\ref{app:eq_LC_sign}) applies; the input $q$ has $k=d$, so $c^2k\ge d$ and the row is radius-limited, with uniform entries forcing $\tau_c(q)|q_j|=1\le c$ for every $j$.
Hence
\begin{align*}
    \LC(e_1)=c\,e_1,&\quad
\LC(q)=\mathbf 1=\sqrt d\,q \\
&\Downarrow \\
R\,\LC(e_1) &\ne\ \LC(Re_1) \quad (c\,q\ \ne\ \sqrt d\,q),
\end{align*}
Note where equivariance fails: the two outputs are parallel
but carry different energies, since a support-limited row is allowed only $c^2$ units of squared norm while its dense rotation receives the full budget $d$. For every \(c\ge\sqrt d\), including finite caps, the constraint is redundant, and the oracle equals row-RMS, so full right-\(O(d)\) equivariance holds. When \(d=1\), there is no nontrivial cap range. Thus only the active-cap range is excluded from the left-permutation/right-orthogonal (LPRO) class of \cite{lau2026symmetry}. The next theorem shows this exclusion is structural, not an implementation artifact.

\begin{theorem}[No-go: uniqueness of row-RMS under full orthogonal equivariance]
\label{app:the-NoGO}
Let \(\varphi:\mathbb{R}^d\setminus\{0\}\to\mathbb{R}^d\) satisfy
\begin{enumerate}[topsep=0pt]
    \item right-orthogonal equivariance: \(\varphi(R b)=R\,\varphi(b)\) for every \(R\in O(d)\);
    \item fixed row energy: \(\|\varphi(b)\|_2=\sqrt d\);
    \item positive alignment: \(\langle b,\varphi(b)\rangle>0\).
\end{enumerate}
Then necessarily
\[
\varphi(b)=\sqrt d\,\frac{b}{\|b\|_2},
\]
i.e. \(\varphi\) is row-RMS. Positive scale invariance follows from this conclusion and need not be assumed separately.
\end{theorem}
\begin{proof}
Fix \(b\ne0\). If \(R\in O(d)\) fixes \(b\), equivariance gives
\[
\varphi(b)=\varphi(Rb)=R\varphi(b),
\]
so every orthogonal map fixing \(b\) also fixes \(\varphi(b)\). For \(d\ge2\), a reflection fixing \(b\) and reversing a vector in \(b^\perp\) shows that \(\varphi(b)\) has no component in \(b^\perp\); for \(d=1\), collinearity is immediate. Thus \(\varphi(b)=\alpha(b)b\). Fixed norm gives \(|\alpha(b)|=\sqrt d/\|b\|_2\), and positive alignment selects the positive sign. The positive-scale-invariance statement follows immediately from the displayed formula.
\end{proof}

For a row-separable rule, these conditions uniquely select row-RMS. They do not imply that dropping row separability is the only alternative, nor that signed permutations form a maximal symmetry group for every conceivable concentration-control method. In fact, even without fixed energy, any nonzero positively aligned orthogonally equivariant row map is parallel to \(b\), so its PAPR is
\[
d\,\frac{\|b\|_\infty^2}{\|b\|_2^2},
\]
and cannot have a uniform ceiling below \(d\).

AF-Muon chooses coordinate-dependent control in the active-cap range. Orthogonally invariant constraints, such as an additional spectral constraint, describe different geometries and can couple rows; Theorem \ref{app:the-NoGO} does not characterize such coupled algorithms. Support-limited clipped rows also need not have fixed energy, so Theorem \ref{app:the-NoGO} alone is not a proof of their failure of equivariance; the explicit counterexample above is.

Relatedly, the ``Sign as spectral'' representation of \cite{lau2026symmetry} embeds the vectorized matrix in a large diagonal matrix and preserves only coordinate geometry; AF-Muon's finite cap remains coordinate-aware and is not claimed to convert Sign into a fully spectral or LPRO update.

%=============================
% Seventh Section
%=============================

%\section{Convergence guarantees and their exact scope}
\section{Exact oracle properties and conditional convergence bounds}
\label{app:sec-convergenceAFMuon}

The theory in this appendix characterizes the optimization geometry of AF-Muon and gives conditional descent bounds for idealized variants. We distinguish three regimes:
\begin{enumerate}[leftmargin=*]
\item Exact deterministic oracle: current full gradients, exact LMOs, and exact arithmetic.
\item Inexact oracle: momentum estimation, approximate polar maps, numerical bisection, and stabilization are represented by explicit estimator and oracle-error terms.
\item Practical stochastic training: no unconditional convergence theorem is claimed. A theorem for the stochastic implementation would require additional assumptions controlling stochastic estimator error, momentum lag, oracle error, feasibility violation, and their interaction with the learning-rate schedule.
\end{enumerate}
Thus the results below are exact oracle properties and conditional convergence bounds for idealized variants, not a convergence guarantee for the full implementation in Algorithm~\ref{alg:afmuon}. This scope mirrors the analysis frontier for LMO-based LLM optimizers, where conditional convergence theory is developed by \cite{riabinin2025gluon} and the stochastic/momentum analysis of nonlinear optimizer maps is left open by \cite{lau2026symmetry}. All bounds below concern the stated fixed geometry and smooth objective, and they do not assert dimension-independent constants. Practical stochastic pretraining requires quantitative estimator, oracle, and schedule assumptions; unbiased stochastic gradients alone are insufficient.

\subsection{Product geometry}
Partition parameters into blocks \(x=(x_1,\ldots,x_B)\), block \(b\) with norm \(N_b\) and radius \(r_b>0\). The radii absorb branch learning-rate ratios, tied scale/width factors, and shape factors. Define
\[
N_{\mathcal X}(s)=\max_b \frac{N_b(s_b)}{r_b}.
\]
Then, by independent block-wise maximization,
\[
N_{\mathcal X,*}(g)=\sum_b r_b N_{b,*}(g_b),\qquad
D_{\mathcal X}(g)_b=r_b D_{N_b}(g_b).
\]
Idealized role-specific AF-Muon is thus one LMO of a weighted product norm whenever all branches share the scalar schedule \(a_t\), the same product-norm viewpoint as modular-norm methods \cite{large2024scalable,pethick2025scion}. Time-varying relative branch rates generally require a time-varying-geometry analysis.

\subsection{Exact deterministic descent}
Assume \(f\ge f_{\inf}>-\infty\), \(L>0\), and, on the region used by the iterates,
\[
f(x+s)\le f(x)+\ip{\nabla f(x)}{s}+\frac L2 N_{\mathcal X}(s)^2.
\]
For \(x_{t+1}=x_t-\eta_t D_{\mathcal X}(\nabla f(x_t))\), \(\eta_t\ge0\), and \(\sum_{t<T}\eta_t>0\), the exact LMO identity and feasibility imply
\[
\min_{t<T}N_{\mathcal X,*}(\nabla f(x_t))
\le \frac{\Delta+\frac L2\sum_{t<T}\eta_t^2}{\sum_{t<T}\eta_t},\qquad
\Delta=f(x_0)-f_{\inf}.
\]
Indeed, the one-step decrease is at least \(\eta_t N_{\mathcal X,*}(\nabla f(x_t))-L\eta_t^2/2\); summing proves the bound. If \(\Delta>0\), the horizon-dependent constant step \(\eta=\sqrt{2\Delta/(LT)}\) gives \(\sqrt{2L\Delta/T}\). For diminishing steps,
\[
\sum_t\eta_t=\infty,\quad \sum_t\eta_t^2<\infty
\quad\Longrightarrow\quad
\liminf_t N_{\mathcal X,*}(\nabla f(x_t))=0.
\]
Under the additional norm-PL condition
\[
N_{\mathcal X,*}(\nabla f(x))^2\ge2\mu(f(x)-f_*),\qquad 0<\mu\le L,
\]
the gradient-dependent step \(\eta_t=N_{\mathcal X,*}(\nabla f(x_t))/L\) gives
\[
f(x_{t+1})-f_*\le(1-\mu/L)(f(x_t)-f_*).
\]
This is a different schedule from practical AF-Muon. A generic finite-dimensional smoothness constant need not be width- or vocabulary-independent, so these bounds do not by themselves prove hyperparameter transfer. This is the standard normalized-steepest-descent argument \cite{pethick2025scion,riabinin2025gluon}.

\subsection{Master inexact-oracle theorem}
Let \(g_t=\nabla f(x_t)\), let \(z_t\) be the implemented oracle input on the same gradient scale (e.g. a momentum estimate), and let \(e_t=z_t-g_t\). Assume the implemented direction \(\widehat D_t=\widehat D(z_t)\) satisfies
\[
N_{\mathcal X}(\widehat D_t)\le1,\qquad
\ip{z_t}{\widehat D_t}\ge N_{\mathcal X,*}(z_t)-\delta_t,\quad \delta_t\ge0,
\]
and update \(x_{t+1}=x_t-\eta_t\widehat D_t\). This is the bound through which all AF-Muon approximations enter: momentum lag, stochastic estimation, Newton--Schulz polar error, tied bisection error, and stabilization.

\begin{theorem}[Master inexact-oracle bound]
\label{app:the-MasterInexactOracle}
Under \(L\)-smoothness in \(N_{\mathcal X}\),
\[
f(x_{t+1})\le f(x_t)-\eta_t N_{\mathcal X,*}(g_t)
+2\eta_t N_{\mathcal X,*}(e_t)+\eta_t\delta_t+\frac L2\eta_t^2,
\]
and consequently
\[
\frac{\sum_{t<T}\eta_t N_{\mathcal X,*}(g_t)}{\sum_{t<T}\eta_t}
\le
\frac{\Delta+2\sum_{t<T}\eta_t N_{\mathcal X,*}(e_t)
+\sum_{t<T}\eta_t\delta_t+\frac L2\sum_{t<T}\eta_t^2}{\sum_{t<T}\eta_t}.
\]
\end{theorem}
\begin{proof}
Duality and feasibility give
\[
\ip{e_t}{\widehat D_t}\le N_{\mathcal X,*}(e_t),
\]
so
\[
\ip{g_t}{\widehat D_t}
=\ip{z_t}{\widehat D_t}-\ip{e_t}{\widehat D_t}
\ge N_{\mathcal X,*}(z_t)-\delta_t-N_{\mathcal X,*}(e_t).
\]
The reverse triangle inequality gives \(N_{\mathcal X,*}(z_t)\ge N_{\mathcal X,*}(g_t)-N_{\mathcal X,*}(e_t)\). Combining,
\[
\ip{g_t}{\widehat D_t}\ge N_{\mathcal X,*}(g_t)-2N_{\mathcal X,*}(e_t)-\delta_t.
\]
Insert this into the smoothness inequality with \(s=-\eta_t\widehat D_t\) and telescope.
\end{proof}

For this bound to tend to zero, it suffices that \(\sum_{t<T}\eta_t\to\infty\), \(\sum_{t<T}\eta_t^2/\sum_{t<T}\eta_t\to0\), and both weighted-average error terms vanish. These are sufficient conditions for the bound, not necessary conditions for actual convergence. A persistent error bound yields a residual upper bound; it does not by itself prove that the algorithm must have a nonzero error floor.

\paragraph{Feasibility inflation.}
If instead \(N_{\mathcal X}(\widehat D_t)\le1+\kappa_t\), \(\kappa_t\ge0\), with the same objective-accuracy condition, replace \(2N_{\mathcal X,*}(e_t)\) by \((2+\kappa_t)N_{\mathcal X,*}(e_t)\) and the curvature term by \(L\eta_t^2(1+\kappa_t)^2/2\). A five-step bfloat16 Newton--Schulz map is not automatically covered by a uniform small \(\delta_t\) or \(\kappa_t\); those bounds require evidence for the actual implementation.

\paragraph{Momentum and stabilization.}
For \(B_t=\mu B_{t-1}+g_t\), \(B_{-1}=0\), and \(0\le\mu<1\), define \(S_t=\sum_{j=0}^t\mu^{t-j}\) and \(z_t=B_t/S_t\). The exact scale-invariant oracle satisfies \(D(B_t)=D(z_t)\), and
\[
N_{\mathcal X,*}(z_t-g_t)\le\sum_{j=0}^t\frac{\mu^{t-j}}{S_t}N_{\mathcal X,*}(g_j-g_t).
\]
If \(N_{\mathcal X,*}(g_j-g_t)\le L_g N_{\mathcal X}(x_j-x_t)\) and \(N_{\mathcal X}(x_{k+1}-x_k)\le h\), this is at most \(L_gh\mu/(1-\mu)\). The Lipschitz-gradient assumption is explicit; the one-sided descent inequality alone is not used to assert it. Zero-initialized Nesterov-style combinations can likewise be normalized by their total nonnegative weight. A stabilized implementation may lose exact scale invariance: normalize only in the analysis of its actual direction and measure \(\delta_t\) against that \(z_t\). For example,
\[
\widehat D=b/\max\{\RMS(b),\varepsilon\}
\]
is RMS-feasible with exact deficit
\[
\sqrt p\|b\|_2-\frac{\|b\|_2^2}{\max\{\RMS(b),\varepsilon\}}
\]
relative to input \(b\).

For stochastic gradients the bound holds pathwise whenever its assumptions hold; expectations require integrable error bounds. In general
\[
\mathbb E[D(G_t)\mid\mathcal F_t]\ne D(g_t)
\]
even if \(\mathbb E[G_t\mid\mathcal F_t]=g_t\). Thus unbiased gradients do not give unbiased normalized directions, and with fixed batch and momentum the estimation error generally does not vanish.

\subsection{Decoupled decay as conditional gradient}
The decayed hidden-matrix update
\[
x^+=(1-\eta\lambda)x-\eta rD_N(g)
\]
is not gradient descent on \(f(x)+\frac{\lambda}{2}\|x\|_2^2\); the LMO is applied to \(g\), not to \(g+\lambda x\). It has an exact conditional-gradient reading instead.

\begin{theorem}[Decoupled decay as a Frank--Wolfe step]
\label{app:the-DecoupledDecayFW}
Let \(r>0\), \(\lambda>0\), \(R=r/\lambda\), \(C=\{x:N(x)\le R\}\), \(g=\nabla f(x)\), and \(\gamma=\eta\lambda\in[0,1]\). With \(v(g)=-R D_N(g)\in\arg\min_{v\in C}\ip{g}{v}\),
\[
x^+=(1-\eta\lambda)x-\eta rD_N(g)=(1-\gamma)x+\gamma v(g).
\]
If \(x\in C\), this is a feasible Frank--Wolfe step. If \(x\notin C\), it is the same algebraic update but is not a feasible Frank--Wolfe iterate. Convexity preserves feasibility from \(x_0\in C\). Also
\[
[N(x_{t+1})-R]_+\le(1-\gamma_t)[N(x_t)-R]_+.
\]
Thus distance in norm to the radius bound tends to zero if \(\prod_t(1-\gamma_t)=0\), for example if \(\sum_t\gamma_t=\infty\). Positive but summable steps alone do not ensure attraction.
\end{theorem}
\begin{proof}
The scaled negative oracle minimizes \(\ip{g}{v}\) over \(C\); substitution proves the identity. Convexity preserves feasibility from \(x_0\in C\). For the norm recursion, $N(v(g))\le R$ and $\gamma_t\in[0,1]$ give $N(x_{t+1})\le(1-\gamma_t)N(x_t)+\gamma_tR$, hence $N(x_{t+1})-R\le(1-\gamma_t)\bigl(N(x_t)-R\bigr)$; taking positive parts and using monotonicity of $[\cdot]_+$ together with $1-\gamma_t\ge0$ yields the stated inequality.
\end{proof}

For feasible iterates, the Frank--Wolfe gap
\[
G_C(x)=\ip{\nabla f(x)}{x-v(\nabla f(x))}\ge0.
\]
If \(f\) is \(L\)-smooth in \(N\) on \(C\), let \(D_C=\sup_{x,y\in C}N(x-y)\le2R\). Smoothness gives
\[
f(x_{t+1})\le f(x_t)-\gamma_tG_C(x_t)+\frac L2D_C^2\gamma_t^2,
\]
and
\[
\min_{t<T}G_C(x_t)\le
\frac{f(x_0)-f_{\inf}+\frac L2D_C^2\sum_{t<T}\gamma_t^2}{\sum_{t<T}\gamma_t}.
\]
The denominator must be positive. This stationarity bound follows directly by telescoping and does not require convexity of \(f\). For convex \(f\), it also controls suboptimality through \(f(x)-\min_Cf\le G_C(x)\); classical convex Frank--Wolfe analysis is discussed by \cite{jaggi13frankwolfe}. For a decayed block of a coupled network, this isolated-block bound cannot be applied while ignoring other simultaneous block updates; the following theorem handles them jointly.

\subsection{Hybrid stationarity}
\begin{theorem}[Hybrid stationarity]
\label{app:the-HybridStationarity}
Let \(f\ge f_{\inf}\) satisfy the joint bound
\[
f(x+s)\le f(x)+\sum_b\ip{g_b}{s_b}+\frac12\sum_b L_bN_b(s_b)^2.
\]
Partition stored blocks into decayed blocks \(\mathcal H\) and nondecayed blocks \(\mathcal A\) (including the tied table). Take fixed \(r_b,\bar\eta_b>0\), \(\lambda_h>0\), schedules \(\eta_{b,t}=a_t\bar\eta_b\), \(a_t\ge0\), exact current gradients and LMOs, and \(a_t\bar\eta_h\lambda_h\le1\). Assume \(x_{h,0}\in C_h=\{x_h:N_h(x_h)\le r_h/\lambda_h\}\). Put
\begin{align*}
G_h(x)&=\ip{g_h}{x_h+(r_h/\lambda_h)D_{N_h}(g_h)},\\
S(x)&=\sum_{h\in\mathcal H}\bar\eta_h\lambda_hG_h(x)
+\sum_{a\in\mathcal A}\bar\eta_a r_aN_{a}^*(g_a),
\end{align*}
and 
\begin{align*}
C_0&=\sum_{h\in\mathcal H}L_h(2\bar\eta_h r_h)^2
+\sum_{a\in\mathcal A}L_a(\bar\eta_a r_a)^2.
\end{align*}
Then \(S(x_t)\ge0\) and, whenever \(\sum_{t<T}a_t>0\),
\[
f(x_{t+1})\le f(x_t)-a_tS(x_t)+\frac12C_0a_t^2,
\]
and
\[
\min_{t<T}S(x_t)\le
\frac{f(x_0)-f_{\inf}+\frac12C_0\sum_{t<T}a_t^2}{\sum_{t<T}a_t}.
\]
\end{theorem}
\begin{proof}
Nonnegativity of $S$ is the Frank--Wolfe gap inequality applied blockwise. For $h\in\mathcal H$,
$v_h=-(r_h/\lambda_h)D_{N_h}(g_h)$ minimizes $\ip{g_h}{\cdot}$ over $C_h$, so
$G_h(x)=\ip{g_h}{x_h-v_h}\ge0$ whenever $x_h\in C_h$; feasibility holds for every $t$ by
$x_{h,0}\in C_h$, $a_t\bar\eta_h\lambda_h\le1$, and the convexity argument of
Theorem \ref{app:the-DecoupledDecayFW}. For $a\in\mathcal A$, $N_{a}^*(g_a)\ge0$ by definition of
a dual norm. Hence $S(x_t)\ge0$.\\
For \(h\in\mathcal H\), \(s_h=a_t\bar\eta_h\lambda_h(v_h-x_h)\) and feasibility gives \(N_h(s_h)\le2a_t\bar\eta_h r_h\). Its linear term is \(-a_t\bar\eta_h\lambda_hG_h\). For \(a\in\mathcal A\), \(s_a=-a_t\bar\eta_ar_aD_{N_a}(g_a)\), with norm at most \(a_t\bar\eta_ar_a\) and linear term \(-a_t\bar\eta_ar_aN_{a}^*(g_a)\). Insert both into the joint smoothness bound and telescope.
\end{proof}
This proves a hybrid constrained/unconstrained stationarity statement for exact deterministic updates. It is not gradient descent on \(f+\lambda\|x\|_2^2/2\) and does not by itself cover practical momentum, numerical approximations, or stochastic pretraining with decay. In summary, exact deterministic progress and inexact-oracle bounds hold under explicit smoothness and error assumptions, without implying unconditional convergence of practical stochastic pretraining.

%=============================
% Eighth Section
%=============================

\section{Extension to Encoder-Decoder Architectures (T5)}
\label{app:three_way_tying}

Before detailing the formal equivalence, it is helpful to establish the architectural context of models such as T5. Unlike the decoder-only architectures discussed previously, which employ a two-way weight tying scheme (sharing weights between the input token embedding and the pre-softmax output projection), T5 utilizes an encoder-decoder framework. 

In this three-way tying setting, a single vocabulary table $E \in \mathbb{R}^{V \times d}$ is shared across three distinct operational roles:
\begin{itemize}
    \item Encoder Lookup: Embedding the input sequence tokens into the encoder's continuous space.
    \item Decoder Lookup: Embedding the shifted target sequence tokens into the decoder.
    \item Output Classification: Projecting the final decoder hidden states back into the vocabulary space to generate softmax logits.
\end{itemize}

Intuitively, while this architecture introduces a third path for gradient flow, the encoder and decoder lookup operations are highly symmetric. From a geometric perspective, both act simply as identity row-selectors from $E$, constrained by the exact same bounds and radii. Because they serve identical mathematical roles in shaping the parameter space, the addition of the encoder lookup does not alter the fundamental topology of the optimization landscape. We can formalize this intuition by showing that the three-way alias norm collapses directly into the two-way case.

\begin{proposition}[Equivalence of Three-Way and Two-Way Tying]
\label{prop:three_way_equivalence}
Consider a fully shared T5-style vocabulary table $E \in \mathbb{R}^{V \times d}$ operating across the three roles described above. Because the encoder and decoder lookup roles share identical identity maps, norms, and radii, they constitute a geometrically inert duplicate role. Consequently, the three-way alias norm collapses exactly to the two-way tied intersection norm:
$$N_{\mathrm{alias}}^{\mathrm{T5}}(U) = \max\{N_{\mathrm{in}}(U), N_{\mathrm{out},c}(U)\} = N_{\mathrm{tie},c}(U)$$
Therefore, all geometric, oracle, dual, and symmetry results from the two-way decoder-only setting transfer verbatim. Furthermore, because shared redundancy is monotonic, adding a duplicate lookup role cannot reopen the head-shift redundancy quotient, ensuring $K_{\mathrm{alias}}^{\mathrm{T5}} = \{0\}$ remains true.
\end{proposition}
\begin{proof}
Index the three roles by $r\in\{\mathrm{enc},\mathrm{dec},\mathrm{out}\}$ with role maps
$A_{\mathrm{enc}}=A_{\mathrm{dec}}=A_{\mathrm{out}}$,
%=\mathrm{Id}
norms
$N_{\mathrm{enc}}=N_{\mathrm{dec}}=N_{\mathrm{in}}$ and $N_{\mathrm{out}}=\norm{\cdot}_\infty$,
and radii $\rho_{\mathrm{enc}}=\rho_{\mathrm{dec}}=1$, $\rho_{\mathrm{out}}=c$.
By the definition of the alias norm,
\[
N_{\mathrm{alias}}^{\mathrm{T5}}(U)
=\max\Bigl\{N_{\mathrm{in}}(U),\,N_{\mathrm{in}}(U),\,\tfrac{\norm{U}_\infty}{c}\Bigr\}
=\max\Bigl\{N_{\mathrm{in}}(U),\,\tfrac{\norm{U}_\infty}{c}\Bigr\}
=\tie(U),
\]
since a maximum is unchanged by repeating one of its arguments. The unit ball, the dual of
Theorem \ref{app:the_alias_aware_dual}, and the oracle of Theorem \ref{app:the_exactsupportware}
depend on the roles only through this functional, so each coincides with its two-way counterpart.
For the dual, the duplicated role contributes a third summand
$\rho_{\mathrm{enc}}N_{\mathrm{in}}^*(Z_{\mathrm{enc}})$ whose optimal value is attained at
$Z_{\mathrm{enc}}=0$: replacing $(Z_{\mathrm{enc}},Z_{\mathrm{dec}})$ by
$(0,Z_{\mathrm{enc}}+Z_{\mathrm{dec}})$ preserves the constraint $\sum_rZ_r=B$ and cannot increase
the objective, by the triangle inequality for $N_{\mathrm{in}}^*$ and the equality of the two radii.
For the redundancy claim, Theorem \ref{app:the_redundancy_intersection} gives
\[
K_{\mathrm{alias}}^{\mathrm{T5}}
=K_{\mathrm{enc}}\cap K_{\mathrm{dec}}\cap K_{\mathrm{out}}
\subseteq K_{\mathrm{dec}}\cap K_{\mathrm{out}}
=\{0\},
\]
the last equality by Proposition \ref{app:pro_remove-shift-redundancy}. Adjoining a role can only
shrink an intersection, which is the monotonicity asserted in the statement.
\end{proof}

\subsection{Gradient Structure and Convergence}

While the geometry remains unchanged, the gradient bookkeeping requires a minor structural extension to account for the additional lookup path:

\begin{itemize}[leftmargin=*]
    \item Combined Lookup Gradients: The total input-role gradient is the sum of the encoder and decoder lookup gradients, $G_{\mathrm{in}}^{\mathrm{T5}} = G_{\mathrm{enc}} + G_{\mathrm{dec}}$.
    \item Sparsity: This combined input gradient remains row-sparse and is supported strictly on the union of the encoder and decoder occupied token positions within the batch.
    \item Density: The output classification gradient $G_{\mathrm{out}}$ remains generically dense across all vocabulary rows.
    \item Convergence: Convergence guarantees transfer with identical proofs under standard block reindexing by simply treating the encoder and decoder hidden weights as an enlarged set of decoupled parameter blocks.
\end{itemize}

%----------------------------
% \newpage
\section{Details of Experiments}
\label{sec:details_of_experiments}

\subsection{NanoGPT}
\label{app:nanogpt_tied_vocab}

We include a NanoGPT-style decoder-only experiment to test AF-Muon in a standard
tied input/output embedding language-model setting. The model has 12 layers, 6
attention heads, head dimension 128, hidden size 768, MLP hidden size 3072, and a
50,304-token vocabulary. The input and output vocabulary tables are tied:
% \[
% E_{\mathrm{in}} = W_{\mathrm{out}}^\top .
% \]
\( E_{\mathrm{in}} = W_{\mathrm{out}}^\top \)
Thus the shared table receives sparse input-lookup gradients and dense
output-classifier gradients.

All runs use FP32 trainable parameters and optimizer state with BF16 autocast,
constant learning rates, matrix learning rate $0.02$, auxiliary/vector learning
rate $3\times10^{-4}$, matrix weight decay $0.1$, gradient clipping at norm
$1.0$, batch size 512, context length 1024, and 2.62B training tokens. The
Hybrid Muon baseline applies auxiliary AdamW weight decay $0.01$ to the tied and
one-dimensional auxiliary parameters. AF-Muon and SCION-style Sign use zero
auxiliary/tied weight decay.

% \begin{table}[ht]
% \centering
% \caption{
% Per-seed NanoGPT-style decoder-only results. All runs use tied input/output
% embeddings and 2.67B training tokens.
% }
% \label{tab:app_nanogpt_per_seed}
% \begin{tabular}{lcccc}
% \toprule
% Method & Seed & Train loss $\downarrow$ & Validation loss $\downarrow$ & Validation PPL $\downarrow$ \\
% \midrule
% Hybrid Muon & 43 & 3.5182 & 3.5212 & 33.83 \\
% Hybrid Muon & 44 & 3.5133 & 3.5152 & 33.62 \\
% Hybrid Muon & 45 & 3.5134 & 3.5165 & 33.67 \\
% SCION-style Sign & 43 & 3.5185 & 3.5272 & 34.03 \\
% SCION-style Sign & 44 & 3.5001 & 3.5046 & 33.27 \\
% SCION-style Sign & 45 & 3.5187 & 3.5247 & 33.94 \\
% AF-Muon & 43 & 3.4910 & 3.4936 & 32.91 \\
% AF-Muon & 44 & 3.4896 & 3.4898 & 32.78 \\
% AF-Muon & 45 & 3.4889 & 3.4921 & 32.86 \\
% \bottomrule
% \end{tabular}
% \end{table}

\begin{table}[ht]
\centering
\caption{
Per-seed NanoGPT-style decoder-only results. All runs use tied input/output
embeddings and 2.67B training tokens. The final row reports a paired two-sided $t$-test over validation loss between
AF-Muon and the strongest non-AF baseline by mean validation loss. The
validation-loss entry is the mean paired gain, computed as baseline minus
AF-Muon, so positive values indicate that AF-Muon is better.
}
\label{tab:app_nanogpt_per_seed}
\begin{tabular}{lcccc}
\toprule
Method & Seed & Train loss $\downarrow$ & Validation loss $\downarrow$ & Validation PPL $\downarrow$ \\
\midrule
Hybrid Muon & 43 & 3.5182 & 3.5212 & 33.83 \\
Hybrid Muon & 44 & 3.5133 & 3.5152 & 33.62 \\
Hybrid Muon & 45 & 3.5134 & 3.5165 & 33.67 \\
SCION-style Sign & 43 & 3.5185 & 3.5272 & 34.03 \\
SCION-style Sign & 44 & 3.5001 & 3.5046 & 33.27 \\
SCION-style Sign & 45 & 3.5187 & 3.5247 & 33.94 \\
AF-Muon & 43 & 3.4910 & 3.4936 & 32.91 \\
AF-Muon & 44 & 3.4896 & 3.4898 & 32.78 \\
AF-Muon & 45 & 3.4889 & 3.4921 & 32.86 \\
% \midrule
% AF-Muon vs. Hybrid Muon & paired $t$-test & -- & gain $0.0258$ & $p=0.0013$ \\
\midrule
\multicolumn{2}{l}{AF-Muon vs.\ Hybrid Muon (paired $t$-test)}
 & \multicolumn{3}{c}{validation-loss gain $=0.0258$, $\ p=0.0013$} \\
% \midrule
% AF-Muon vs.\ Hybrid Muon & paired $t$-test & -- & gain $0.0258$,\ $p{=}0.0013$ & -- \\
\bottomrule
\end{tabular}
\end{table}

\begin{table}[ht]
\centering
\caption{
Final NanoGPT tied-vocabulary diagnostics. Hybrid Muon uses auxiliary AdamW
weight decay $0.01$; SCION-style Sign and AF-Muon use zero auxiliary/tied weight
decay.
}
\label{tab:app_nanogpt_diagnostics}
\begin{tabular}{lccccc}
\toprule
Method & Seed & Tied RMS & Tied max & Matrix RMS & Aux WD \\
\midrule
Hybrid Muon & 43 & 0.0427 & 0.2500 & 0.0463 & 0.01 \\
Hybrid Muon & 44 & 0.0429 & 0.2543 & 0.0465 & 0.01 \\
Hybrid Muon & 45 & 0.0431 & 0.2568 & 0.0464 & 0.01 \\
SCION-style Sign & 43 & 0.5528 & 7.8773 & 0.0473 & 0.00 \\
SCION-style Sign & 44 & 0.6238 & 7.8289 & 0.0466 & 0.00 \\
SCION-style Sign & 45 & 0.5548 & 7.8897 & 0.0471 & 0.00 \\
AF-Muon & 43 & 0.3821 & 5.8463 & 0.0469 & 0.00 \\
AF-Muon & 44 & 0.3895 & 9.2696 & 0.0469 & 0.00 \\
AF-Muon & 45 & 0.3849 & 9.4472 & 0.0468 & 0.00 \\
\bottomrule
\end{tabular}
\end{table}

The diagnostics show that AF-Muon improves validation loss while keeping the
shared tied table substantially smaller than the SCION-style Sign endpoint in
RMS norm. Hybrid Muon keeps the tied table much smaller because its auxiliary
AdamW branch includes weight decay, but this additional radial shrinkage does
not close the performance gap. This supports the interpretation that the
finite-cap tied-table geometry is not merely acting as uncontrolled norm growth:
it gives a better tied-table update direction under the matched matrix-update
protocol.

\begin{figure}[t]
    \centering
    \includegraphics[width=\linewidth]{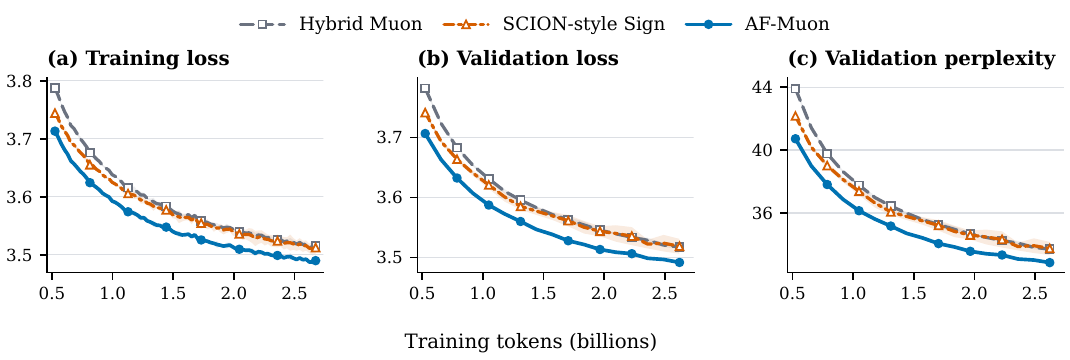}
    \caption{
    NanoGPT-style decoder-only training dynamics on FineWeb. Curves show the mean
    over three paired seeds, and shaded regions denote $\pm 1$ standard error of
    the mean.
    }
    \label{fig:nanogpt_tied_vocab_curves}
\end{figure}

% \newpage
\subsection{SmolLM2-135M}
\label{app:smollm2_tied_vocab}

We include a SmolLM2-135M decoder-only experiment to test AF-Muon in a
pretrained-tokenizer tied input/output embedding language-model setting. The
model has approximately 134.5M trainable parameters and uses the SmolLM2
tokenizer and tied input/output vocabulary table:
% \[
% E_{\mathrm{in}} = W_{\mathrm{out}}^\top .
% \]
\( E_{\mathrm{in}} = W_{\mathrm{out}}^\top \)
Thus the shared table receives sparse input-lookup gradients and dense
output-classifier gradients.

\begin{table}[ht]
\centering
\caption{
Per-seed SmolLM2-135M decoder-only results. All runs use tied input/output
embeddings and 2.49B training tokens. The final row reports a paired two-sided
$t$-test over validation loss between AF-Muon and the strongest non-AF baseline
by mean validation loss; the gain is baseline minus AF-Muon, so positive values
indicate that AF-Muon is better.
}
\label{tab:app_smollm2_per_seed}
\begin{tabular}{lcccc}
\toprule
Method & Seed & Train loss $\downarrow$ & Validation loss $\downarrow$ & Validation PPL $\downarrow$ \\
\midrule
Hybrid Muon & 43 & 3.5118 & 3.5047 & 33.27 \\
Hybrid Muon & 44 & 3.5082 & 3.5058 & 33.31 \\
Hybrid Muon & 45 & 3.5063 & 3.5088 & 33.41 \\
SCION-style Sign & 43 & 3.4777 & 3.4744 & 32.28 \\
SCION-style Sign & 44 & 3.4757 & 3.4663 & 32.02 \\
SCION-style Sign & 45 & 3.4649 & 3.4712 & 32.17 \\
AF-Muon & 43 & 3.4451 & 3.4405 & 31.20 \\
AF-Muon & 44 & 3.4460 & 3.4461 & 31.38 \\
AF-Muon & 45 & 3.4449 & 3.4473 & 31.41 \\
\midrule
\multicolumn{2}{l}{AF-Muon vs.\ SCION-style (paired $t$-test)}
& \multicolumn{3}{c}{validation-loss gain $=0.0260$, $p=0.0239$} \\
\bottomrule
\end{tabular}
\end{table}

\begin{table}[ht]
\centering
\caption{
Final SmolLM2-135M tied-vocabulary diagnostics. Hybrid Muon uses auxiliary/tied
AdamW weight decay $0.01$; SCION-style Sign and AF-Muon use zero auxiliary/tied
weight decay.
}
\label{tab:app_smollm2_diagnostics}
\begin{tabular}{lccccc}
\toprule
Method & Seed & Tied RMS & Tied max & Matrix RMS & Aux WD \\
\midrule
Hybrid Muon & 43 & 0.0549 & 0.3561 & 0.0561 & 0.01 \\
Hybrid Muon & 44 & 0.0552 & 0.3029 & 0.0559 & 0.01 \\
Hybrid Muon & 45 & 0.0551 & 0.3498 & 0.0560 & 0.01 \\
SCION-style Sign & 43 & 0.6611 & 10.0026 & 0.0562 & 0.00 \\
SCION-style Sign & 44 & 0.6524 & 10.0348 & 0.0560 & 0.00 \\
SCION-style Sign & 45 & 0.6668 & 10.0104 & 0.0560 & 0.00 \\
AF-Muon & 43 & 0.4404 & 9.3857 & 0.0566 & 0.00 \\
AF-Muon & 44 & 0.4457 & 12.9820 & 0.0567 & 0.00 \\
AF-Muon & 45 & 0.4415 & 13.6058 & 0.0566 & 0.00 \\
\bottomrule
\end{tabular}
\end{table}

All runs use FP32 trainable parameters and optimizer state with BF16 autocast,
constant learning rates, matrix learning rate $0.02$, auxiliary/vector learning
rate $3\times10^{-4}$, matrix weight decay $0.1$, gradient clipping at norm
$1.0$, batch size 512, context length 1024, and 2.49B training tokens. The
matched large-batch protocol uses micro-batch size 32 and gradient accumulation
16, giving 524k tokens per optimizer update. The Hybrid Muon baseline applies
auxiliary/tied AdamW weight decay $0.01$ to the tied and one-dimensional
auxiliary parameters. AF-Muon and SCION-style Sign use zero auxiliary/tied
weight decay.

The diagnostics show that AF-Muon improves validation loss while keeping the
shared tied table smaller in RMS norm than the SCION-style Sign endpoint. Hybrid
Muon keeps the tied table much smaller because its auxiliary AdamW branch
includes weight decay, but this additional radial shrinkage does not close the
performance gap. This supports the interpretation that the finite-cap tied-table
geometry is not merely acting as uncontrolled norm growth: it gives a better
tied-table update direction under the matched matrix-update protocol.

\subsection{Qwen2.5-0.5B}
\label{app:qwen25_0p5b_tied_vocab}

We include Qwen2.5-0.5B to test AF-Muon on a larger decoder-only architecture
with tied input/output embeddings. The model uses the Qwen2.5-0.5B tokenizer and
a shared input/output vocabulary table,
% \[
% E_{\mathrm{in}} = W_{\mathrm{out}}^\top .
% \]
\( E_{\mathrm{in}} = W_{\mathrm{out}}^\top \)
As in the NanoGPT and SmolLM2 experiments, this shared table receives both sparse
input-lookup gradients and dense output-classifier gradients.

All runs use FP32 trainable parameters and optimizer state with BF16 autocast,
constant learning rates, matrix learning rate $0.02$, auxiliary/vector learning
rate $3\times10^{-4}$, matrix weight decay $0.1$, gradient clipping at norm
$1.0$, context length 1024, and 524,288 tokens per optimizer update. The short
three-seed runs use 1B training tokens. Hybrid Muon applies auxiliary AdamW
weight decay $0.01$ to the tied and one-dimensional auxiliary parameters, while
SCION-style Sign and AF-Muon use zero auxiliary/tied weight decay.

\begin{table}[ht]
\centering
\caption{
Per-seed Qwen2.5-0.5B short-horizon results. Values are reported at the last
evaluation before the 1B-token budget, corresponding to 0.918B training tokens.
The final row reports a paired two-sided $t$-test over validation loss between
AF-Muon and the strongest non-AF baseline by mean validation loss; the gain is
baseline minus AF-Muon, so positive values indicate that AF-Muon is better.
}
\label{tab:app_qwen25_0p5b_per_seed}
\begin{tabular}{lcccc}
\toprule
Method & Seed & Train loss $\downarrow$ & Validation loss $\downarrow$ & Validation PPL $\downarrow$ \\
\midrule
Hybrid Muon & 43 & 3.6186 & 3.6349 & 37.90 \\
Hybrid Muon & 44 & 3.6293 & 3.6406 & 38.12 \\
Hybrid Muon & 45 & 3.6329 & 3.6424 & 38.18 \\
SCION-style Sign & 43 & 3.5935 & 3.6075 & 36.87 \\
SCION-style Sign & 44 & 3.5990 & 3.6133 & 37.09 \\
SCION-style Sign & 45 & 3.6026 & 3.6119 & 37.04 \\
AF-Muon & 43 & 3.5660 & 3.5805 & 35.89 \\
AF-Muon & 44 & 3.5715 & 3.5831 & 35.98 \\
AF-Muon & 45 & 3.5763 & 3.5858 & 36.08 \\
\midrule
\multicolumn{2}{l}{AF-Muon vs.\ SCION-style (paired $t$-test)}
& \multicolumn{3}{c}{validation-loss gain $=0.0278$, $p=0.0020$} \\
\bottomrule
\end{tabular}
\end{table}

At the near-1B evaluation point, AF-Muon achieves the lowest validation loss on
all three paired Qwen2.5-0.5B seeds. The paired improvements are
$0.0562\pm0.0016$ validation-loss points over Hybrid Muon and
$0.0278\pm0.0022$ points over SCION-style Sign. This ordering is also preserved
in the seed-43 long-horizon continuation in
Appendix~\ref{app:qwen25_0p5b_long_horizon}.

\begin{table}[ht]
\centering
\caption{
Final Qwen2.5-0.5B tied-table diagnostics for the 1B short-horizon seed-44 and
seed-45 runs. Hybrid Muon uses auxiliary AdamW weight decay $0.01$;
SCION-style Sign and AF-Muon use zero auxiliary/tied weight decay. Seed-43
checkpoint diagnostics are reported separately for the long-horizon continuation
in Table~\ref{tab:app_qwen25_0p5b_long_horizon_diagnostics}.
}
\label{tab:app_qwen25_0p5b_diagnostics}
\begin{tabular}{lccccc}
\toprule
Method & Seed & Tied RMS & Tied max & Matrix RMS & Aux WD \\
\midrule
Hybrid Muon & 44 & 0.0481 & 0.3676 & 0.0435 & 0.01 \\
Hybrid Muon & 45 & 0.0488 & 0.3256 & 0.0433 & 0.01 \\
SCION-style Sign & 44 & 0.2732 & 2.6204 & 0.0440 & 0.00 \\
SCION-style Sign & 45 & 0.2666 & 2.6204 & 0.0439 & 0.00 \\
AF-Muon & 44 & 0.1997 & 2.6786 & 0.0445 & 0.00 \\
AF-Muon & 45 & 0.1934 & 3.0957 & 0.0444 & 0.00 \\
\bottomrule
\end{tabular}
\end{table}
The short-horizon checkpoint diagnostics show the same qualitative pattern as
the long-horizon Qwen continuation: Hybrid Muon keeps the tied table smallest
because of auxiliary/tied AdamW weight decay, SCION-style Sign produces the
largest tied-table RMS norm, and AF-Muon remains intermediate while achieving
the best validation loss.

\begin{figure}[ht]
   \centering
   \includegraphics[width=\linewidth]{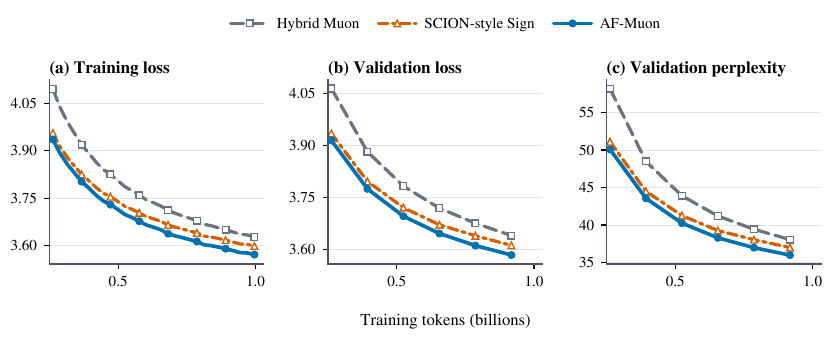}
   \caption{
   Qwen2.5-0.5B training dynamics on FineWeb. Curves show the mean over three
   paired seeds, and shaded regions denote $\pm 1$ standard error of the mean.
   The plot shows the short-horizon region used for the main three-seed
   comparison, omitting the earliest training region for readability. All
   methods use tied input/output embeddings, FP32 trainable parameters and
   optimizer state with BF16 autocast, matrix weight decay $0.1$, and the same
   matched large-batch constant-learning-rate protocol.
   }
   \label{fig:qwen25_0p5b_tied_vocab_curves}
\end{figure}

\subsubsection{Qwen2.5-0.5B Long-horizon Continuation}
\label{app:qwen25_0p5b_long_horizon}

% In the 6.29B-token seed-43 continuation, AF-Muon obtains the lowest final
% validation loss, followed by SCION-style Sign and then Hybrid Muon. AF-Muon
% improves over Hybrid Muon by $0.0316$ validation-loss points and over
% SCION-style Sign by $0.0138$ points at the endpoint. 
In the seed-43 continuation trained to 6.40B tokens, AF-Muon obtains the lowest validation loss at the last scheduled evaluation, 6.29B tokens, followed by SCION-style Sign and then Hybrid Muon. AF-Muon improves over Hybrid Muon by $0.0316$ validation-loss points and over SCION-style Sign by $0.0138$ points at that evaluation point.
The final-checkpoint
diagnostics show the same qualitative tied-table norm ordering observed in the
other tied-vocabulary experiments: Hybrid Muon keeps the tied table smallest
because its auxiliary AdamW branch applies weight decay, SCION-style Sign
produces the largest tied-table norm, and AF-Muon remains intermediate while
obtaining the best endpoint loss. As with the Llama-3.2-1B continuation, this is
a single-seed stability check for the short-horizon three-seed comparison rather
than an independent multi-seed estimate.

% \begin{figure}[ht]
%    \centering
%    \includegraphics[width=\linewidth]{figs/qwen25/qwen25_0p5b_seed43_long_horizon_curves.pdf}
%    \caption{
%    Qwen2.5-0.5B long-horizon seed-43 continuation. The three optimizer arms are
%    trained under the same matched large-batch protocol as the short-horizon
%    runs. This figure is a single-seed stability check, not a replacement for the
%    three-seed short-horizon comparison in Table~\ref{tab:qwen25_0p5b_tied_vocab}.
%    }
%    \label{fig:app_qwen25_0p5b_seed43_long_horizon}
% \end{figure}

\begin{table}[ht]
\centering
\caption{
% Qwen2.5-0.5B long-horizon seed-43 endpoint results. The run continues the same
% seed used in the paired short-horizon comparison to 6.29B training tokens. This
% single-seed continuation is used as a stability check, not as a three-seed
% estimate.
Qwen2.5-0.5B long-horizon seed-43 endpoint results. The run continues the same seed used in the paired short-horizon comparison to 6.40B training tokens; the validation endpoint below is the last scheduled evaluation, at 6.29B tokens. This single-seed continuation is used as a stability check, not as a three-seed
estimate.
}
\label{tab:app_qwen25_0p5b_long_horizon_results}
\begin{tabular}{lccc}
\toprule
Optimizer & Train loss $\downarrow$ & Validation loss $\downarrow$ & Validation PPL $\downarrow$ \\
\midrule
Hybrid Muon & 3.4254 & 3.4274 & 30.80 \\
SCION-style Sign & 3.4068 & 3.4096 & 30.25 \\
AF-Muon & \textbf{3.3922} & \textbf{3.3957} & \textbf{29.84} \\
\bottomrule
\end{tabular}
\end{table}

\begin{table}[ht]
\centering
\caption{
Final Qwen2.5-0.5B long-horizon seed-43 checkpoint diagnostics. The tied-table
and matrix norms are measured from the final model-only checkpoint at 6.40B
training tokens. Hybrid Muon uses auxiliary/tied AdamW weight decay $0.01$;
SCION-style Sign and AF-Muon use zero auxiliary/tied weight decay. All runs use
FP32 trainable parameters and optimizer state with BF16 autocast.
}
\label{tab:app_qwen25_0p5b_long_horizon_diagnostics}
\resizebox{\linewidth}{!}{%
\begin{tabular}{lcccccc}
\toprule
Optimizer & Tied RMS & Tied max & Matrix RMS & Matrix WD & Aux./tied WD & Precision \\
\midrule
Hybrid Muon & 0.0553 & 0.4511 & 0.0439 & 0.1 & 0.01 & FP32 + BF16 \\
SCION-style Sign & 0.5815 & 7.4900 & 0.0450 & 0.1 & 0.00 & FP32 + BF16 \\
AF-Muon & 0.3868 & 6.5167 & 0.0452 & 0.1 & 0.00 & FP32 + BF16 \\
\bottomrule
\end{tabular}%
}
\end{table}

\begin{figure}[ht]
   \centering
   \includegraphics[width=0.9\linewidth]{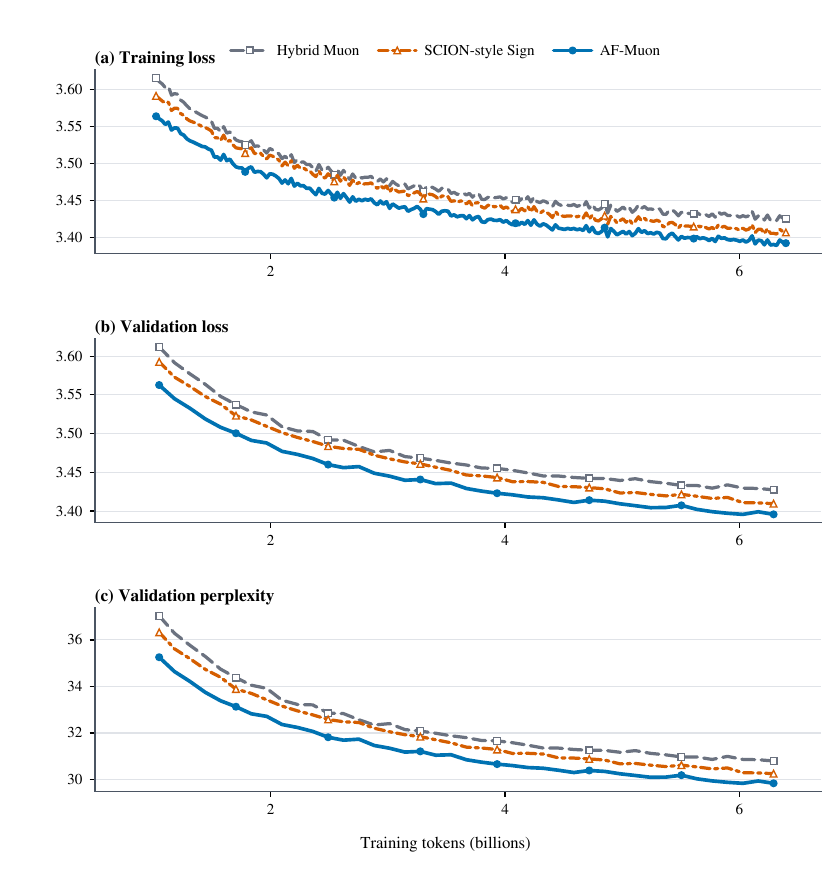}
   \caption{
   Qwen2.5-0.5B long-horizon seed-43 continuation. The three optimizer arms are
   trained under the same matched large-batch protocol as the short-horizon
   runs. The plotted curves omit the earliest training region and begin after
   1B tokens to focus on the long-horizon behavior. This figure is a single-seed stability check, not a replacement for the three-seed short-horizon
   comparison. AF-Muon preserves the best validation loss at the last scheduled evaluation point, 6.29B tokens.
   }
   \label{fig:app_qwen25_0p5b_seed43_long_horizon}
\end{figure}
%----------------------------
\newpage
\subsection{Llama-3.2-1B Decoder-only Results}
\label{app:llama32_1b_tied_vocab}

We include Llama-3.2-1B to test AF-Muon on a larger decoder-only architecture
with a high-capacity tied vocabulary table. The model uses the Llama-3.2-1B
tokenizer and tied input/output embeddings,
% \[
% E_{\mathrm{in}} = W_{\mathrm{out}}^\top .
% \]
\( E_{\mathrm{in}} = W_{\mathrm{out}}^\top \)
As in the other decoder-only language-model experiments, the shared table
receives both sparse input-lookup gradients and dense output-classifier
gradients.

All runs use FP32 trainable parameters and optimizer state with BF16 autocast,
constant learning rates, matrix learning rate $0.02$, auxiliary/vector learning
rate $3\times10^{-4}$, matrix weight decay $0.1$, gradient clipping at norm
$1.0$, context length 1024, and 524,288 tokens per optimizer update. The short
three-seed runs use 2B training tokens. Hybrid Muon applies auxiliary/tied AdamW
weight decay $0.01$ to the tied and one-dimensional auxiliary parameters, while
SCION-style Sign and AF-Muon use zero auxiliary/tied weight decay. The
long-horizon seed-43 continuation uses the same settings and extends training to
9.98B tokens.

\begin{table}[ht]
\centering
\caption{
Per-seed Llama-3.2-1B short-horizon results. Values are reported at the final
available evaluation near the 2B-token budget. Seed 43 is read from the
long-horizon continuation at the last pre-2B evaluation point, 1.966B tokens;
seeds 44 and 45 are read from the 2B short-horizon runs. The final row reports
a paired two-sided $t$-test over validation loss between AF-Muon and the
strongest non-AF baseline by mean validation loss; the gain is baseline minus
AF-Muon, so positive values indicate that AF-Muon is better.
}
\label{tab:app_llama32_1b_per_seed}
\begin{tabular}{lcccc}
\toprule
Method & Seed & Train loss $\downarrow$ & Validation loss $\downarrow$ & Validation PPL $\downarrow$ \\
\midrule
Hybrid Muon & 43 & 3.4927 & 3.5348 & 34.29 \\
Hybrid Muon & 44 & 3.4908 & 3.5400 & 34.47 \\
Hybrid Muon & 45 & 3.4962 & 3.5415 & 34.52 \\
SCION-style Sign & 43 & 3.4755 & 3.5139 & 33.58 \\
SCION-style Sign & 44 & 3.4750 & 3.5280 & 34.06 \\
SCION-style Sign & 45 & 3.4809 & 3.5256 & 33.97 \\
AF-Muon & 43 & 3.4600 & 3.5012 & 33.15 \\
AF-Muon & 44 & 3.4565 & 3.5107 & 33.47 \\
AF-Muon & 45 & 3.4617 & 3.5051 & 33.28 \\
\midrule
\multicolumn{2}{l}{AF-Muon vs.\ SCION-style (paired $t$-test)}
& \multicolumn{3}{c}{validation-loss gain $=0.0168$, $p=0.0176$} \\
\bottomrule
\end{tabular}
\end{table}

\begin{table}[ht]
\centering
\caption{
Final Llama-3.2-1B tied-table diagnostics for the 2B short-horizon seed-44 and
seed-45 runs. Hybrid Muon uses auxiliary/tied AdamW weight decay $0.01$;
SCION-style Sign and AF-Muon use zero auxiliary/tied weight decay. Seed-43
checkpoint diagnostics are reported separately for the long-horizon continuation
in Table~\ref{tab:app_llama32_1b_long_horizon_diagnostics}.
}
\label{tab:app_llama32_1b_diagnostics}
\begin{tabular}{lccccc}
\toprule
Method & Seed & Tied RMS & Tied max & Matrix RMS & Aux WD \\
\midrule
Hybrid Muon & 44 & 0.0311 & 0.3537 & 0.0287 & 0.01 \\
Hybrid Muon & 45 & 0.0310 & 0.2848 & 0.0287 & 0.01 \\
SCION-style Sign & 44 & 0.1690 & 2.2859 & 0.0287 & 0.00 \\
SCION-style Sign & 45 & 0.1718 & 2.3012 & 0.0287 & 0.00 \\
AF-Muon & 44 & 0.1129 & 2.2891 & 0.0290 & 0.00 \\
AF-Muon & 45 & 0.1125 & 1.3577 & 0.0290 & 0.00 \\
\bottomrule
\end{tabular}
\end{table}

The checkpoint diagnostics show the same tied-table norm ordering as in the
Qwen2.5-0.5B runs: Hybrid Muon keeps the tied table smallest because its
auxiliary AdamW branch applies weight decay, SCION-style Sign produces the
largest tied-table RMS norm, and AF-Muon remains intermediate while achieving
the best validation loss.

\subsubsection{Llama-3.2-1B long-horizon continuation}
\label{app:llama32_1b_long_horizon}

To test whether the short-horizon optimizer ordering persists beyond the
three-seed 2B-token budget, we continue seed 43 to 9.98B training tokens. This
run is not used as an independent three-seed estimate; instead, it checks
whether the behavior observed in the paired short-horizon runs remains stable
when one representative seed is trained for a substantially larger token budget.

\begin{table}[ht]
\centering
\caption{
Llama-3.2-1B long-horizon seed-43 endpoint results.
}
\label{tab:app_llama32_1b_long_horizon_results}
\begin{tabular}{lccc}
\toprule
Optimizer & Train loss $\downarrow$ & Validation loss $\downarrow$ & Validation PPL $\downarrow$ \\
\midrule
Hybrid Muon & 3.3848 & 3.4251 & 30.73 \\
SCION-style Sign & 3.3534 & 3.3951 & 29.82 \\
AF-Muon & \textbf{3.3504} & \textbf{3.3912} & \textbf{29.70} \\
\bottomrule
\end{tabular}
\end{table}

\begin{table}[ht]
\centering
\caption{
Final Llama-3.2-1B long-horizon seed-43 diagnostics. All runs use FP32
trainable parameters and optimizer state with BF16 autocast.
}
\label{tab:app_llama32_1b_long_horizon_diagnostics}
\resizebox{\linewidth}{!}{%
\begin{tabular}{lccccccc}
\toprule
Optimizer & Grad norm & Tied RMS & Tied max & Matrix WD & Aux. WD & Tied WD & Precision \\
\midrule
Hybrid Muon & 0.3709 & 0.0368 & 0.3828 & 0.1 & 0.01 & 0.01 & FP32 + BF16 \\
SCION-style Sign & 0.3029 & 0.3656 & 6.1341 & 0.1 & 0.00 & 0.00 & FP32 + BF16 \\
AF-Muon & 0.3084 & 0.2463 & 2.8759 & 0.1 & 0.00 & 0.00 & FP32 + BF16 \\
\bottomrule
\end{tabular}%
}
\end{table}
% In the 9.98B-token seed-43 continuation, AF-Muon preserves the short-horizon
% ordering and achieves the lowest final validation loss. The AF-Muon advantage
% over SCION-style Sign is small at this horizon, while both AdamW-free tied-table
% arms remain clearly ahead of Hybrid Muon. The diagnostics show the same norm
% ordering observed in the smaller tied-vocabulary experiments: Hybrid Muon keeps
% the tied table smallest because its auxiliary AdamW branch applies weight decay,
% SCION-style Sign produces the largest tied-table norm, and AF-Muon remains
% intermediate while obtaining the best endpoint loss.
% \begin{figure}[ht]
%    \centering
%    \includegraphics[width=0.95\linewidth]{figs/llama32p/llama32_1b_seed43_long_horizon_curves_vertical.pdf}
%    \caption{
%    Llama-3.2-1B long-horizon seed-43 continuation. The three optimizer arms are
%    trained under the same matched large-batch protocol as the short-horizon
%    runs. The plotted curves omit the earliest training region and begin after
%    1B tokens to focus on the long-horizon behavior. This figure is a single-seed
%    stability check, not a replacement for the three-seed short-horizon
%    comparison. AF-Muon preserves the best endpoint validation loss at 9.98B
%    tokens, while both AdamW-free tied-table arms remain ahead of Hybrid Muon.
%    }
%    \label{fig:app_llama32_1b_seed43_long_horizon}
% \end{figure}
\begin{figure}[!t]
  \centering
  \includegraphics[
    width=0.98\linewidth,
    height=0.72\textheight,
    keepaspectratio,
    trim=0 7 0 15,
    clip
  ]{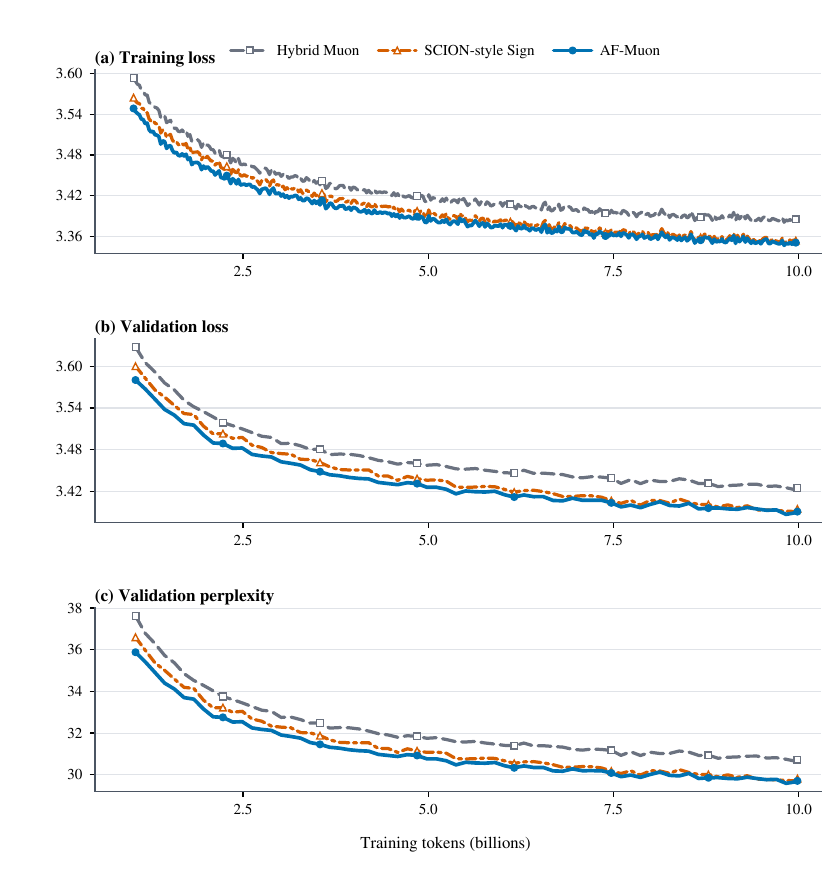}
  \vspace{-0.6em}
  \caption{
  Llama-3.2-1B long-horizon seed-43 continuation. The curves omit the earliest
  training region and begin after 1B tokens to focus on long-horizon behavior.
  This is a single-seed stability check, not a replacement for the three-seed
  short-horizon comparison. AF-Muon preserves the best endpoint validation loss
  at 9.98B tokens, while both AdamW-free tied-table arms remain ahead of Hybrid
  Muon.
  }
  \label{fig:app_llama32_1b_seed43_long_horizon}
\end{figure}

In the 9.98B-token seed-43 continuation shown in
Figure~\ref{fig:app_llama32_1b_seed43_long_horizon}, AF-Muon obtains the lowest
final validation loss, with SCION-style Sign close behind and both AdamW-free
tied-table arms clearly ahead of Hybrid Muon. The final AF-Muon advantage over
SCION-style Sign is small, but the ordering is preserved at the end of the
long-horizon run. We therefore use this continuation as a stability check for
the three-seed short-horizon comparison, rather than as a separate multi-seed
estimate.

\clearpage
\subsection{Shared-vocabulary encoder--decoder experiments}
\label{app:t5test_shared_vocab}

\begin{figure}[!t]
   \centering
   \includegraphics[
     width=\linewidth,
     trim=0 5 0 5,
     clip
   ]{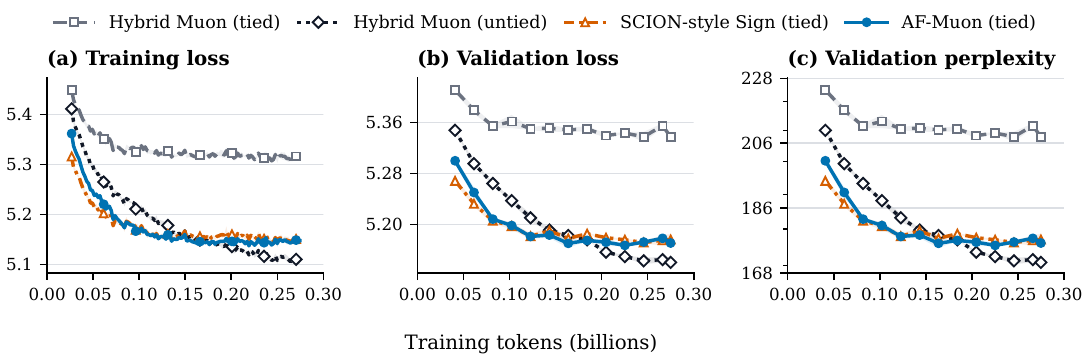}
   \vspace{-10pt}
   \caption{
   Small T5-style training dynamics under shared-vocabulary and untied-vocabulary topologies. Curves show the mean over three seeds. The untied Hybrid Muon control has nearly twice as many parameters, because the encoder embedding, decoder embedding, and output head are separate.
   }
   \vspace{-7pt}
   \label{fig:t5test_shared_vocab_curves}
\end{figure}

% We include an additional encoder--decoder diagnostic to test AF-Muon in a
% shared-parameter topology different from decoder-only language modeling.
% The model is a small T5-style encoder--decoder trained with a
% prefix-to-continuation objective on FineWeb. The encoder input length and decoder
% target length are both 256 tokens. The model has 4 encoder layers, 4 decoder
% layers, hidden size 512, feed-forward size 2048, 8 attention heads, and a
% 49,152-token vocabulary. Unless otherwise stated, all runs use FP32 trainable
% parameters and optimizer state with BF16 autocast, constant learning rates,
% matrix weight decay $0.1$, and no learning-rate schedule.

% In the fully shared setting, the encoder embedding, decoder embedding, and
% output projection use one shared vocabulary table:
% % \[
% % E_{\mathrm{enc}} = E_{\mathrm{dec}} = W_{\mathrm{out}}^\top .
% % \]
% \( E_{\mathrm{enc}} = E_{\mathrm{dec}} = W_{\mathrm{out}}^\top \)
% Thus the shared table receives three structurally different gradient
% contributions: sparse encoder lookup gradients, sparse decoder lookup gradients,
% and dense output-head gradients. This makes the setting a useful diagnostic for
% mixed-role shared parameters.
We include two encoder--decoder experiments to test AF-Muon in a shared-parameter
topology different from decoder-only language modeling. In both cases, one
vocabulary table is used for encoder input, decoder input, and output
classification, so the table receives sparse encoder-lookup gradients, sparse
decoder-lookup gradients, and dense output-head gradients.
They answer complementary questions:
\begin{itemize}[leftmargin=*]
    \item \textbf{Small T5-style:} Does AF-Muon work in a controlled
    encoder--decoder setup, including a tied-versus-untied comparison?
    \item \textbf{T5Gemma2-style:} Does the optimizer recipe transfer to a
    different, larger encoder--decoder configuration?
\end{itemize}
\paragraph{Small T5-style diagnostic.}
The small controlled model is a T5-style encoder--decoder trained with a
prefix-to-continuation objective on FineWeb Table~\ref{tab:app_t5test_per_seed}. The encoder input length and decoder
target length are both 256 tokens. The model has 4 encoder layers, 4 decoder
layers, hidden size 512, feed-forward size 2048, 8 attention heads, and a
49,152-token vocabulary. Unless otherwise stated, runs use FP32 trainable
parameters and optimizer state with BF16 autocast, constant learning rates,
matrix weight decay $0.1$, and no learning-rate schedule. In the fully shared
setting,
\(E_{\mathrm{enc}} = E_{\mathrm{dec}} = W_{\mathrm{out}}^\top\).

\begin{table}[ht]
\centering
\caption{
Single-seed $\rho$ sweep for AF-Muon on small T5-style shared-vocabulary encoder–decoder with a fully shared
vocabulary table. All runs use 110M target tokens, matrix weight decay $0.1$,
and zero auxiliary/tied weight decay. We select $\rho_{\mathrm{out}} /
\rho_{\mathrm{hid}} = 1$ for the main three-seed experiment.
}
\label{tab:app_t5test_afmoun_rho_sweep}
\small
\setlength{\tabcolsep}{3.2pt}
\renewcommand{\arraystretch}{0.95}
\begin{tabular}{lcccccc}
\toprule
Method & $\rho_{\mathrm{out}}/\rho_{\mathrm{hid}}$ & Tied LR & Eval loss $\downarrow$ & Eval PPL $\downarrow$ & Shared RMS & Shared max \\
\midrule
Hybrid Muon & -- & -- & 5.3468 & 209.93 & 0.142 & 3.19 \\
AF-Muon & 1 & $1.9531\times10^{-5}$ & \textbf{5.2652} & \textbf{193.48} & 0.318 & 7.51 \\
AF-Muon & 5 & $9.7656\times10^{-5}$ & 5.3536 & 211.36 & 0.721 & 8.99 \\
AF-Muon & 10 &$1.9531\times10^{-4}$ & 5.4220 & 226.34 & 1.136 & 20.43 \\
AF-Muon & 60 & $1.1719\times10^{-3}$ & 5.4864 & 241.38 & 2.157 & 37.02 \\
\bottomrule
\end{tabular}
\end{table}

\begin{table}[ht]
\centering
\caption{
Per-seed small T5-style shared-vocabulary encoder--decoder results for the
selected shared-vocabulary setting. All runs use 275M target tokens. Hybrid Muon uses auxiliary AdamW weight decay $0.01$; SCION-style Sign and AF-Muon use zero auxiliary/tied weight decay. The table reports paired two-sided $t$-tests over validation loss comparing AF-Muon with Hybrid Muon and with the strongest non-AF baseline -- in the fully shared setting -- by mean validation loss.
}
\label{tab:app_t5test_per_seed}
\begin{tabular}{llccc}
\toprule
Method & Topology & Seed & Validation loss $\downarrow$ & Validation PPL $\downarrow$ \\
\midrule
Hybrid Muon &  & 43 & 5.3372 & 207.94 \\
Hybrid Muon &  & 44 & 5.3272 & 205.86 \\
Hybrid Muon &  & 45 & 5.3476 & 210.11 \\
SCION-style Sign  &  & 43 & 5.1643 & 174.91 \\
SCION-style Sign  & Fully shared & 44 & 5.1714 & 176.17 \\
SCION-style Sign  &  & 45 & 5.1930 & 180.01 \\
AF-Muon  &  & 43 & 5.1700 & 175.91 \\
AF-Muon  &  & 44 & 5.1745 & 176.72 \\
AF-Muon  &  & 45 & 5.1678 & 175.53 \\
\midrule
\multicolumn{2}{l}{AF-Muon vs.\ Hybrid Muon (paired $t$-test)}
& \multicolumn{3}{c}{validation-loss gain $=0.1666$, $p=0.0022$} \\
\multicolumn{2}{l}{AF-Muon vs.\ SCION-style Sign (paired $t$-test)}
& \multicolumn{3}{c}{validation-loss gain $=0.0055$, $p=0.6361$} \\
\midrule
Hybrid Muon &  & 43 & 5.1355 & 169.96 \\
Hybrid Muon & Fully untied & 44 & 5.1419 & 171.05 \\
Hybrid Muon &  & 45 & 5.1437 & 171.35 \\
\bottomrule
\end{tabular}
\end{table}

% The sweep shows that the decoder-only radius ratio used in the causal language
% model experiments is too large for this encoder--decoder shared-vocabulary
% topology. As $\rho_{\mathrm{out}}/\rho_{\mathrm{hid}}$ increases, the shared
% vocabulary table grows rapidly and validation loss degrades. We therefore use
% $\rho_{\mathrm{out}}/\rho_{\mathrm{hid}}=1$ for the main T5test comparison.
The sweep Table~\ref{tab:app_t5test_afmoun_rho_sweep} shows that the decoder-only radius ratio used in the causal language
model experiments is too large for this encoder--decoder topology: as
$\rho_{\mathrm{out}}/\rho_{\mathrm{hid}}$ increases, the shared table grows
rapidly and validation loss degrades. We therefore use
$\rho_{\mathrm{out}}/\rho_{\mathrm{hid}}=1$ for the main small T5-style
comparison.

The fully untied Hybrid Muon control is not parameter matched: it uses separate
encoder, decoder, and output vocabulary tables, increasing the model from 54.6M
to 104.9M parameters. This control isolates the cost of enforcing a single
shared table. Untying improves Hybrid Muon, but within the stricter fully shared
setting AF-Muon is the best tied method by mean validation loss, with only a
small and statistically insignificant margin over SCION-style Sign.

\paragraph{T5Gemma2-style transfer experiment.}
We also run a larger T5Gemma2-style random-init configuration based on
\citet{t5gemma2025}. This experiment tests whether the same optimizer recipe
transfers beyond the small controlled T5-style diagnostic. The run uses the
Hugging Face \texttt{google/t5gemma-2-270m-270m} configuration from random
initialization, a shared 262,208-row vocabulary table, source length 256, target
length 256, and 750M target tokens. The effective target-token batch is 524,288
tokens per update, using microbatch size 16 and 128 gradient-accumulation steps.
All trainable parameters are FP32 with BF16 autocast. We use the main
$\rho_{\mathrm{out}}/\rho_{\mathrm{hid}}=60$ setting, matrix weight decay $0.1$,
Hybrid Muon auxiliary weight decay $0.01$, and zero auxiliary/tied weight decay
for SCION-style Sign and AF-Muon.

\begin{table}[ht]
\centering
\caption{
T5Gemma2-style shared-vocabulary encoder--decoder results at 750M target tokens.
Runs use the random-init \texttt{google/t5gemma-2-270m-270m} configuration with
a shared 262k-row vocabulary table. Per-seed rows report validation metrics;
summary rows report mean $\pm$ SD over paired seeds. Paired tests compare
validation loss against AF-Muon; positive gain means AF-Muon is better.
}
\label{tab:app_t5gemma2_style}
\small
\setlength{\tabcolsep}{4pt}
\begin{tabular}{llccc}
\toprule
Method & Seed/Summary & Train loss $\downarrow$ & Validation loss $\downarrow$ & Validation PPL $\downarrow$ \\
\midrule
Hybrid Muon & 43 & 3.6366 & 3.6501 & 38.48 \\
Hybrid Muon & 44 & 3.6455 & 3.6578 & 38.77 \\
Hybrid Muon & 45 & 3.6407 & 3.6566 & 38.73 \\
Hybrid Muon & Mean $\pm$ SD & 3.6409 $\pm$ 0.0044 & 3.6548 $\pm$ 0.0041 & 38.66 $\pm$ 0.16 \\
\midrule
SCION-style Sign & 43 & 3.5705 & 3.5871 & 36.13 \\
SCION-style Sign & 44 & 3.5703 & 3.5867 & 36.11 \\
SCION-style Sign & 45 & 3.5685 & 3.5840 & 36.02 \\
SCION-style Sign & Mean $\pm$ SD & 3.5698 $\pm$ 0.0011 & 3.5859 $\pm$ 0.0017 & 36.09 $\pm$ 0.06 \\
\midrule
AF-Muon & 43 & 3.5561 & 3.5723 & 35.60 \\
AF-Muon & 44 & 3.5579 & 3.5742 & 35.67 \\
AF-Muon & 45 & 3.5578 & 3.5738 & 35.65 \\
AF-Muon & Mean $\pm$ SD & \textbf{3.5573 $\pm$ 0.0010} & \textbf{3.5735 $\pm$ 0.0010} & \textbf{35.64 $\pm$ 0.04} \\
\midrule
\multicolumn{2}{l}{AF-Muon vs.\ Hybrid Muon}
& \multicolumn{3}{c}{validation-loss gain $=0.0813$, paired $p=4.87\times10^{-4}$} \\
\multicolumn{2}{l}{AF-Muon vs.\ SCION-style Sign}
& \multicolumn{3}{c}{validation-loss gain $=0.0125$, paired $p=0.0112$} \\
\bottomrule
\end{tabular}
\end{table}

In the completed evaluated seeds, AF-Muon obtains the lowest validation loss and
perplexity. The AF-Muon--Sign gap is larger than in the small T5-style
diagnostic, but the two experiments differ in architecture, vocabulary
configuration, and training budget, so this comparison does not isolate
vocabulary size or model size as the cause. The result is therefore best
interpreted as evidence that the AF-Muon tied-table recipe transfers to a larger
encoder--decoder configuration.
\begin{figure}[t]
    \centering
    \includegraphics[width=\linewidth]{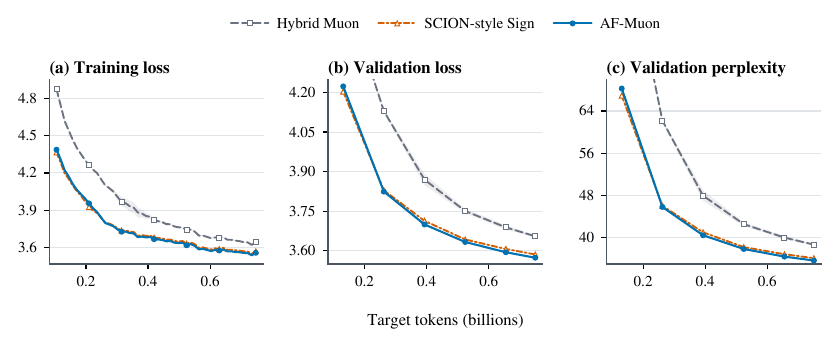}
    \caption{
    T5Gemma2-style shared-vocabulary encoder--decoder training dynamics after 100M target tokens. Curves show the mean over three paired seeds for Hybrid Muon, SCION-style Sign, and AF-Muon.
    }
    \label{fig:app_t5gemma2_style_curves}
\end{figure}
% \clearpage
\subsection{NanoGPT-MoE}
\label{app:moe_nanogpt_tied_vocab}

We include a tied-embedding mixture-of-experts ablation to test whether the
AF-Muon tied-table update remains effective when the decoder-only model uses
sparse conditional computation inside the transformer blocks. Starting from the
NanoGPT-style tied-vocabulary setting, we replace each dense MLP with a top-1
router over four expert MLPs. Each expert uses the same hidden expansion as the
dense NanoGPT MLP. The input and output vocabulary tables remain tied, so the
shared table receives sparse input-lookup gradients and dense output-classifier
gradients:
\[
E_{\mathrm{in}} = W_{\mathrm{out}}^\top .
\]

All runs use the same FineWeb token cache and 50,304-token tokenizer as the
NanoGPT tied-vocabulary experiment. The model has 12 layers, 6 attention heads,
head dimension 128, hidden size 768, four top-1 routed experts per layer, and
MLP hidden size 3072 per expert. We train for 2.62B tokens across three paired
seeds with FP32 trainable parameters and optimizer state, BF16 autocast, batch
size 512, context length 1024, matrix learning rate $0.02$, auxiliary/vector
learning rate $3\times10^{-4}$, matrix weight decay $0.1$, and gradient clipping
at norm $1.0$. The router auxiliary load-balancing coefficient is $0.01$ for all
optimizer arms, and the router $z$-loss coefficient is zero.
\begin{figure}[ht]
    \centering
    \includegraphics[width=\linewidth]{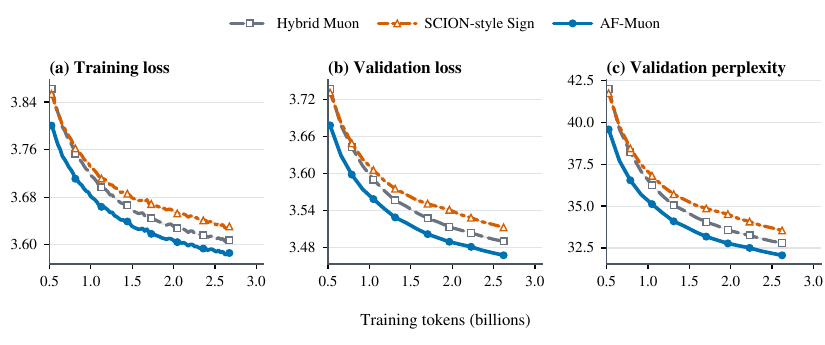}
    \caption{
    Tied-embedding NanoGPT-MoE training dynamics on FineWeb. Curves show the
    mean over three paired seeds, and shaded regions denote $\pm 1$ standard
    error of the mean. All methods use four top-1 routed experts per layer and
    the same router auxiliary load-balancing coefficient.
    }
    \label{fig:moe_nanogpt_tied_vocab_curves}
\end{figure}
The optimizer partition is matched across arms. Matrix parameters, including
attention, expert, and router matrices, use Muon in all three arms. The only
difference is the update applied to the tied table and remaining auxiliary
parameters: Hybrid Muon uses its auxiliary AdamW branch with auxiliary/tied
weight decay $0.01$, SCION-style Sign uses the Sign endpoint with zero
auxiliary/tied weight decay, and AF-Muon uses the finite-cap tied-table update
with zero auxiliary/tied weight decay.

\begin{table}[ht]
\centering
\caption{
Tied-embedding NanoGPT-MoE results on FineWeb for 2.62B tokens across three
paired seeds. We report mean $\pm$ SD over seeds. The model uses four top-1
routed experts per layer and a router auxiliary load-balancing coefficient of
$0.01$.
}
\label{tab:app_moe_nanogpt_summary}
\begin{tabular}{lccc}
\toprule
Optimizer & Train loss $\downarrow$ & Validation loss $\downarrow$ & Validation PPL $\downarrow$ \\
\midrule
Hybrid Muon & 3.4873 $\pm$ 0.0031 & 3.4902 $\pm$ 0.0019 & 32.79 $\pm$ 0.06 \\
SCION-style Sign & 3.5106 $\pm$ 0.0033 & 3.5129 $\pm$ 0.0025 & 33.55 $\pm$ 0.08 \\
AF-Muon & \textbf{3.4662 $\pm$ 0.0011} & \textbf{3.4677 $\pm$ 0.0007} & \textbf{32.06 $\pm$ 0.02} \\
\bottomrule
\end{tabular}
\end{table}

In the tied-embedding MoE setting, AF-Muon consistently achieves the lowest
validation loss. SCION-style Sign is competitive early in training, but Hybrid
Muon overtakes it later; at the final checkpoint AF-Muon improves over both
Hybrid Muon and SCION-style Sign. Quantitatively, AF-Muon improves validation
loss over Hybrid Muon by $0.0224\pm0.0023$ paired loss points and over
SCION-style Sign by $0.0452\pm0.0019$. The router diagnostics show that this
gain is not caused by expert collapse: all methods keep all four experts active,
with router entropy near one and low expert-load coefficient of variation.

\begin{table}[ht]
\centering
\caption{
Per-seed tied-embedding NanoGPT-MoE results. All runs use 2.62B training tokens,
four top-1 routed experts per layer, and the same router auxiliary loss. The
final row reports a paired two-sided $t$-test over validation loss between
AF-Muon and the strongest non-AF baseline by mean validation loss; the gain is
baseline minus AF-Muon, so positive values indicate that AF-Muon is better.
}
\label{tab:app_moe_nanogpt_per_seed}
\begin{tabular}{lccc}
\toprule
Method & Seed & Validation loss $\downarrow$ & Validation PPL $\downarrow$ \\
\midrule
Hybrid Muon & 43 & 3.4883 & 32.73 \\
Hybrid Muon & 44 & 3.4901 & 32.79 \\
Hybrid Muon & 45 & 3.4921 & 32.85 \\
SCION-style Sign & 43 & 3.5112 & 33.49 \\
SCION-style Sign & 44 & 3.5158 & 33.64 \\
SCION-style Sign & 45 & 3.5117 & 33.50 \\
AF-Muon & 43 & 3.4676 & 32.06 \\
AF-Muon & 44 & 3.4685 & 32.09 \\
AF-Muon & 45 & 3.4671 & 32.04 \\
\midrule
\multicolumn{2}{l}{AF-Muon vs.\ Hybrid Muon (paired $t$-test)}
& \multicolumn{2}{c}{validation-loss gain $=0.0224$, $p=0.0034$} \\
\bottomrule
\end{tabular}
\end{table}

\begin{table}[ht]
\centering
\caption{
Final tied-table and router diagnostics for the tied-embedding NanoGPT-MoE
ablation. Values are mean $\pm$ SD over three paired seeds.
}
\label{tab:app_moe_nanogpt_diagnostics}
\resizebox{\linewidth}{!}{%
\begin{tabular}{lcccccc}
\toprule
Method & Tied RMS & Tied max & Router entropy $\uparrow$ & Max expert frac. $\downarrow$ & Router CV $\downarrow$ & Unused experts \\
\midrule
Hybrid Muon & 0.0424 $\pm$ 0.0001 & 0.2656 $\pm$ 0.0120 & 0.9996 $\pm$ 0.0001 & 0.2600 $\pm$ 0.0020 & 0.0290 $\pm$ 0.0056 & 0.00 $\pm$ 0.00 \\
SCION-style Sign & 0.6357 $\pm$ 0.0161 & 7.8549 $\pm$ 0.0381 & 0.9978 $\pm$ 0.0011 & 0.2687 $\pm$ 0.0025 & 0.0604 $\pm$ 0.0096 & 0.00 $\pm$ 0.00 \\
AF-Muon & 0.3736 $\pm$ 0.0044 & 6.1092 $\pm$ 2.0978 & 0.9994 $\pm$ 0.0001 & 0.2626 $\pm$ 0.0009 & 0.0396 $\pm$ 0.0043 & 0.00 $\pm$ 0.00 \\
\bottomrule
\end{tabular}%
}
\end{table}

% \begin{figure}[ht]
%     \centering
%     \includegraphics[width=\linewidth]{figs/nanogpt_moe/moe_nanogpt_50k_top1_e4_curves_tail500m.pdf}
%     \caption{
%     Tied-embedding NanoGPT-MoE training dynamics on FineWeb. Curves show the
%     mean over three paired seeds, and shaded regions denote $\pm 1$ standard
%     error of the mean. All methods use four top-1 routed experts per layer and
%     the same router auxiliary load-balancing coefficient.
%     }
%     \label{fig:moe_nanogpt_tied_vocab_curves}
% \end{figure}

\subsection{Tied ImageGPT-style ImageNet-32}
\label{app:tied_imagegpt_imagenet32}

We include a cross-domain tied-table ablation on autoregressive image modeling
to test whether the AF-Muon tied-embedding behavior also appears outside
language. The model follows an ImageGPT-style causal-transformer setup on
ImageNet-32, but with an intentionally tied input/output token table. This is
therefore a controlled tied-table optimizer probe, not a faithful reproduction
of the original untied ImageGPT setup.

Images are represented as sequences of 1024 discrete color tokens. We use a
deterministic RGB554 tokenizer, which maps each 32-bit RGB pixel to one of
$2^{14}=16{,}384$ color tokens. We add one start-of-sequence token, giving a
total vocabulary size of 16,385. The input sequence is formed by prepending the
start token and shifting the image tokens right, while the output predicts only
the 16,384 color tokens. Thus the color-token rows of the shared table receive
both sparse input-lookup gradients and dense output-classifier gradients, while
the start-token row receives only sparse input gradients.

The model has 8 transformer layers, 8 attention heads, head dimension 64, hidden
size 512, and MLP hidden size 2048. All runs use FP32 trainable parameters and
optimizer state with BF16 autocast, batch size 512, micro-batch size 256,
sequence length 1024, matrix learning rate $0.02$, auxiliary/vector learning
rate $3\times10^{-4}$, matrix weight decay $0.1$, momentum $0.95$, and 1.31B
image tokens. Hybrid Muon uses auxiliary/tied AdamW weight decay $0.01$.
SCION-style Sign and AF-Muon use zero auxiliary/tied weight decay. Because the
RGB554 tokenizer is lossy and task-specific, the absolute bits/dim values are
used only for within-experiment optimizer comparisons.

\begin{table}[ht]
\centering
\caption{
Tied ImageGPT-style ImageNet-32 results with a deterministic RGB554 16k color
vocabulary. Results are mean $\pm$ SD over three paired seeds. All methods use
the same tied input/output color-token table, FP32 trainable parameters and
optimizer state with BF16 autocast, batch size 512, and 1.31B image tokens.
}
\label{tab:tied_imagegpt_imagenet32_summary}
\begin{tabular}{lcccc}
\toprule
Optimizer & Train loss $\downarrow$ & Val. loss $\downarrow$ & Val. bits/dim $\downarrow$ & Val. PPL $\downarrow$ \\
\midrule
Hybrid Muon & 2.3236 $\pm$ 0.0067 & 2.3346 $\pm$ 0.0049 & 1.1227 $\pm$ 0.0024 & 10.33 $\pm$ 0.05 \\
SCION-style Sign & 2.3568 $\pm$ 0.0049 & 2.3712 $\pm$ 0.0055 & 1.1403 $\pm$ 0.0027 & 10.71 $\pm$ 0.06 \\
AF-Muon & \textbf{2.3208 $\pm$ 0.0101} & \textbf{2.3243 $\pm$ 0.0055} & \textbf{1.1177 $\pm$ 0.0026} & \textbf{10.22 $\pm$ 0.06} \\
\bottomrule
\end{tabular}
\end{table}

AF-Muon achieves the best validation loss and bits/dim across the three paired
seeds. Relative to Hybrid Muon, AF-Muon improves validation loss by
$0.0103\pm0.0014$ paired loss points and validation bits/dim by
$0.0049\pm0.0007$. Relative to SCION-style Sign, AF-Muon improves validation
loss by $0.0469\pm0.0064$ and validation bits/dim by $0.0226\pm0.0031$. The
Sign endpoint is substantially worse in this image-token setting, consistent
with the tied-table diagnostics: Sign clips all tied-table rows and has much
larger tail energy, while AF-Muon keeps the finite-cap geometry active.

\begin{table}[ht]
\centering
\caption{
Per-seed tied ImageGPT-style ImageNet-32 results. All runs use 1.31B image
tokens and the same deterministic RGB554 16k tokenizer. The final row reports a
paired two-sided $t$-test over validation loss between AF-Muon and the strongest
non-AF baseline by mean validation loss; the gain is baseline minus AF-Muon, so
positive values indicate that AF-Muon is better.
}
\label{tab:app_tied_imagegpt_imagenet32_per_seed}
\begin{tabular}{lccccc}
\toprule
Method & Seed & Train loss $\downarrow$ & Val. loss $\downarrow$ & Val. bits/dim $\downarrow$ & Val. PPL $\downarrow$ \\
\midrule
Hybrid Muon & 43 & 2.3301 & 2.3363 & 1.1235 & 10.34 \\
Hybrid Muon & 44 & 2.3240 & 2.3290 & 1.1200 & 10.27 \\
Hybrid Muon & 45 & 2.3166 & 2.3384 & 1.1245 & 10.36 \\
SCION-style Sign & 43 & 2.3537 & 2.3671 & 1.1383 & 10.67 \\
SCION-style Sign & 44 & 2.3625 & 2.3690 & 1.1392 & 10.69 \\
SCION-style Sign & 45 & 2.3543 & 2.3775 & 1.1433 & 10.78 \\
AF-Muon & 43 & 2.3300 & 2.3275 & 1.1193 & 10.25 \\
AF-Muon & 44 & 2.3100 & 2.3180 & 1.1147 & 10.16 \\
AF-Muon & 45 & 2.3225 & 2.3273 & 1.1192 & 10.25 \\
\midrule
\multicolumn{2}{l}{AF-Muon vs.\ Hybrid Muon (paired $t$-test)}
& \multicolumn{4}{c}{validation-loss gain $=0.0103$, $p=0.0053$} \\
\bottomrule
\end{tabular}
\end{table}

\begin{table}[ht]
\centering
\caption{
Final tied-table diagnostics for the tied ImageGPT-style ImageNet-32 ablation.
Values are mean $\pm$ SD over three paired seeds. The cosine and objective-ratio
columns compare the AF-Muon tied-table update geometry to the Sign endpoint.
}
\label{tab:app_tied_imagegpt_imagenet32_diagnostics}
\resizebox{\linewidth}{!}{%
\begin{tabular}{lcccccc}
\toprule
Method & Tied RMS & Tied max & Rows clipped & Tail energy & Cos(AF, Sign) & Obj. ratio(AF/Sign) \\
\midrule
Hybrid Muon & 0.0327 $\pm$ 0.0000 & 0.1430 $\pm$ 0.0021 & -- & -- & -- & -- \\
SCION-style Sign & 0.5332 $\pm$ 0.0037 & 4.9483 $\pm$ 0.6251 & 1.000 $\pm$ 0.000 & 0.8012 $\pm$ 0.0005 & -- & -- \\
AF-Muon & 0.3096 $\pm$ 0.0006 & 2.2829 $\pm$ 0.3463 & 0.773 $\pm$ 0.015 & 0.0350 $\pm$ 0.0013 & 0.799 $\pm$ 0.001 & 1.251 $\pm$ 0.000 \\
\bottomrule
\end{tabular}%
}
\end{table}

The diagnostics confirm that the finite-cap tied-table geometry is active in the image-token setting and behaves as it does in language. Hybrid Muon holds the tied table to a very small norm—its auxiliary AdamW branch applies weight decay—but it uses no clipping geometry, so the row-clipping and tail-energy diagnostics do not apply to it. The two LMO arms differ sharply: SCION-style Sign attains the largest tied-table norm and clips every row, concentrating most of the update energy in the clipped tail, whereas AF-Muon maintains an intermediate norm, clips only a selective fraction of rows, and leaves little energy in the tail. Crucially, AF-Muon's update geometry relative to the Sign endpoint cosine 0.80 and objective ratio 1.25 matches its large-batch language-model values almost exactly, indicating that the same finite-cap tied-table structure governs image tokens and text tokens alike. This shared geometry, together with the norm and tail-energy gap between AF-Muon and the Sign endpoint, explains why AF-Muon is both more stable and better-performing than pure Sign in this cross-domain ablation.

\begin{figure}[ht]
    \centering
    \includegraphics[width=\linewidth]{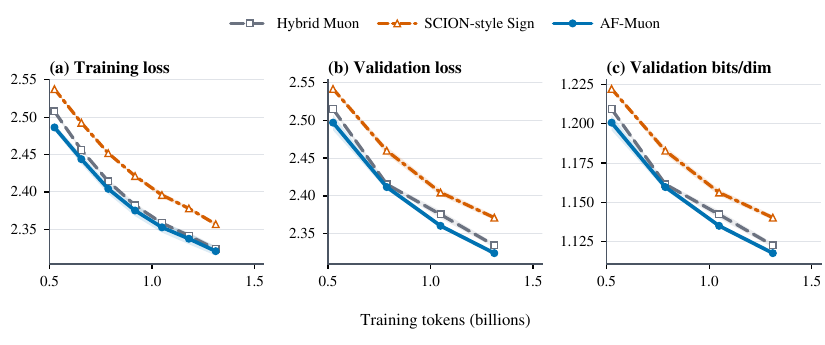}
    \caption{
    Tied ImageGPT-style ImageNet-32 training dynamics with a deterministic
    RGB554 16k color vocabulary. Curves show the mean over three paired seeds,
    and shaded regions denote $\pm 1$ standard error of the mean. To focus on
    the stable late-training regime, the plot starts at step 1000; training
    loss is shown every 250 steps, and validation metrics are shown every
    500 steps. AF-Muon achieves the best final validation loss and validation
    bits/dim, while SCION-style Sign remains worse throughout this late window.
    }
    \label{fig:tied_imagegpt_imagenet32_curves}
\end{figure}

\begin{table}[ht]
\centering
\caption{
Top-half image-completion metrics for the tied ImageGPT-style ImageNet-32
experiment. Metrics are computed by conditioning on the first 512 tokens of
fixed validation images and evaluating the likelihood of the held-out bottom
half. Results are mean $\pm$ SD over three trained seeds; all methods use the
same validation images.
}
\label{tab:app_tied_imagegpt_completion_summary}
\begin{tabular}{lcccc}
\toprule
Optimizer & Completion loss $\downarrow$ & Bits/token $\downarrow$ & Bits/dim $\downarrow$ & PPL $\downarrow$ \\
\midrule
Hybrid Muon & 2.0273 $\pm$ 0.0038 & 2.9248 $\pm$ 0.0054 & 0.9749 $\pm$ 0.0018 & 7.59 $\pm$ 0.03 \\
SCION-style Sign & 2.0511 $\pm$ 0.0105 & 2.9591 $\pm$ 0.0151 & 0.9864 $\pm$ 0.0050 & 7.78 $\pm$ 0.08 \\
AF-Muon & \textbf{2.0117 $\pm$ 0.0076} & \textbf{2.9023 $\pm$ 0.0109} & \textbf{0.9674 $\pm$ 0.0036} & \textbf{7.48 $\pm$ 0.06} \\
\bottomrule
\end{tabular}
\end{table}

\begin{table}[ht]
\centering
\caption{
Per-seed top-half completion results for tied ImageGPT-style ImageNet-32. The
same fixed validation images and completion protocol are used for all methods.
}
\label{tab:app_tied_imagegpt_completion_per_seed}
\begin{tabular}{lccccc}
\toprule
Method & Seed & Completion loss $\downarrow$ & Bits/token $\downarrow$ & Bits/dim $\downarrow$ & PPL $\downarrow$ \\
\midrule
Hybrid Muon & 43 & 2.025113 & 2.921620 & 0.973973 & 7.58 \\
Hybrid Muon & 44 & 2.0316 & 2.9310 & 0.9770 & 7.63 \\
Hybrid Muon & 45 & 2.025116 & 2.921625 & 0.973975 & 7.58 \\
SCION-style Sign & 43 & 2.0456 & 2.9512 & 0.9837 & 7.73 \\
SCION-style Sign & 44 & 2.0631 & 2.9765 & 0.9922 & 7.87 \\
SCION-style Sign & 45 & 2.0444 & 2.9495 & 0.9832 & 7.72 \\
AF-Muon & 43 & 2.0205 & 2.9149 & 0.9716 & 7.54 \\
AF-Muon & 44 & 2.007315 & 2.895943 & 0.965314 & 7.44 \\
AF-Muon & 45 & 2.007310 & 2.895937 & 0.965312 & 7.44 \\
\bottomrule
\end{tabular}
\end{table}

We also evaluate the final checkpoints on a fixed top-half image-completion protocol. For each validation image, the model conditions on the first 512 image tokens and we compute the likelihood of the remaining 512 tokens. This metric does not rely on sampled image quality and provides a paired held-out completion test, where AF-Muon achieves the best score among the trained checkpoints.
% AF-Muon again gives the best completion
% likelihood, improving over Hybrid Muon by $0.0156$ completion-loss points and
% over SCION-style Sign by $0.0394$ points on average.

\begin{figure*}[p]
    \centering
    \begin{subfigure}[t]{0.31\textwidth}
        \centering
        \includegraphics[height=0.87\textheight,keepaspectratio]{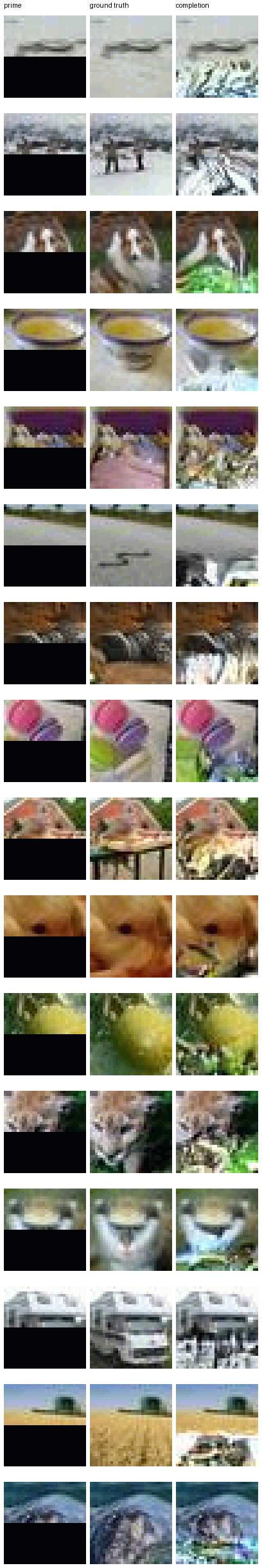}
        \caption{Hybrid Muon}
    \end{subfigure}
    \hfill
    \begin{subfigure}[t]{0.31\textwidth}
        \centering
        \includegraphics[height=0.87\textheight,keepaspectratio]{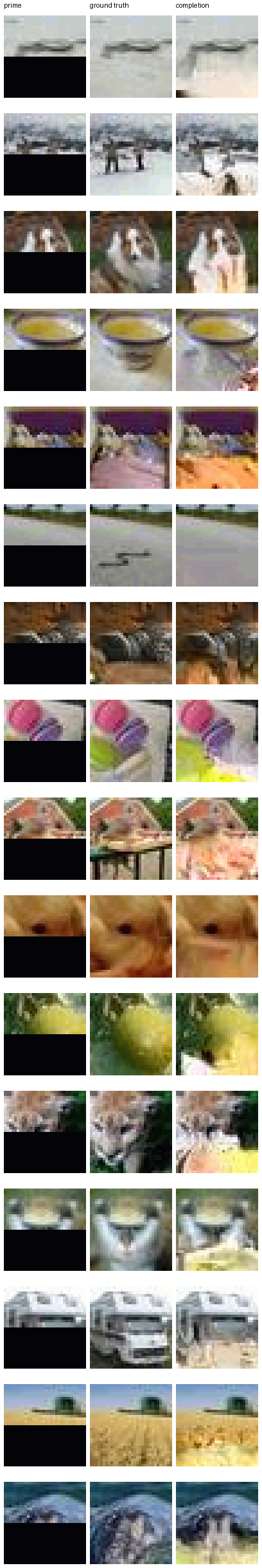}
        \caption{SCION-style Sign}
    \end{subfigure}
    \hfill
    \begin{subfigure}[t]{0.31\textwidth}
        \centering
        \includegraphics[height=0.87\textheight,keepaspectratio]{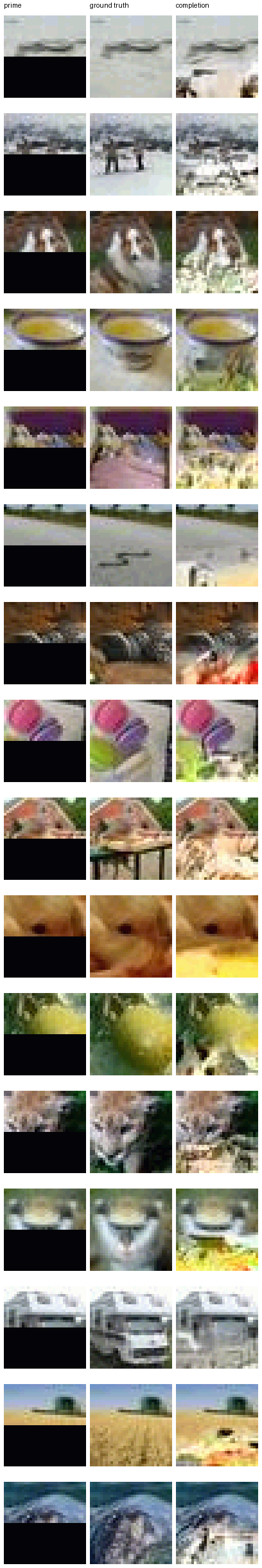}
        \caption{AF-Muon}
    \end{subfigure}
    \caption{
    Qualitative top-half completions from the tied ImageGPT-style ImageNet-32
    seed-43 checkpoints. Each panel uses the same fixed validation images and
    the same sampling seed. Within each panel, columns show the top-half prime,
    the ground-truth full image, and the sampled bottom-half completion. The
    samples are illustrative only; quantitative completion likelihoods are
    reported in Table~\ref{tab:app_tied_imagegpt_completion_summary}.
    }
    \label{fig:app_tied_imagegpt_completions_seed43}
\end{figure*}

\clearpage
\subsection{Protein Language Modeling}
\label{app:protein_lm_per_seed}

\begin{figure}[ht]
\centering
\includegraphics[width=\linewidth]{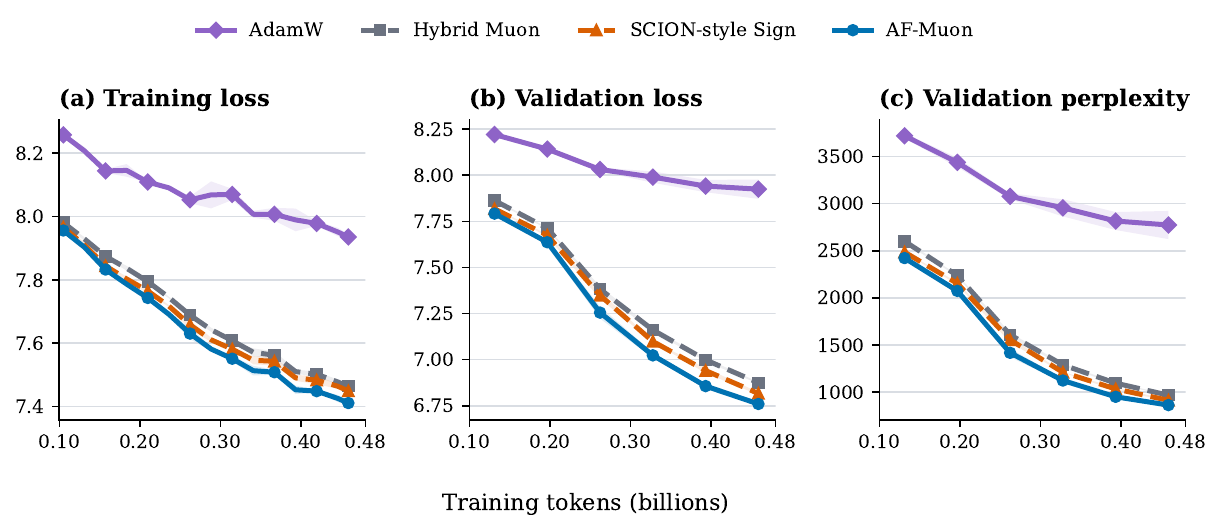}
\caption{
Protein language-modeling training dynamics on UniRef50 / ProtGPT2-BPE with a
363M-parameter tied-vocabulary RoPE decoder. Curves show the mean over three
paired seeds under the matched large-batch protocol, after omitting the first
100M training tokens to focus on the stable training regime. AF-Muon maintains
the lowest training loss, validation loss, and validation perplexity among the
tested optimizers.
}
\label{fig:app_protein_lm_curves}
\end{figure}

\begin{wraptable}{r}{0.40\linewidth}
\vspace{-1.0em}
\centering
\caption{
Seed-43 AdamW LR pilot. All runs use the same 363M RoPE decoder and
500 steps / 32.77M tokens. We use $\eta=6\times10^{-4}$ for the longer AdamW
comparison in table
}
\label{tab:app_protein_adamw_lr_pilot}
\scriptsize
\setlength{\tabcolsep}{3pt}
\begin{tabular}{lcc}
\toprule
AdamW LR & Eval loss $\downarrow$ & Eval PPL $\downarrow$ \\
\midrule
$1\times10^{-4}$ & 8.2669 & 3892.69 \\
$3\times10^{-4}$ & 8.2007 & 3643.67 \\
$6\times10^{-4}$ & \textbf{8.1369} & \textbf{3418.32} \\
$1\times10^{-3}$ & 8.1672 & 3523.36 \\
\bottomrule
\end{tabular}
\vspace{-0.8em}
\end{wraptable}

We evaluate protein-domain transfer on UniRef50 using a ProtGPT2-BPE tokenizer
and a tied-vocabulary 363M-parameter RoPE decoder. All optimizer arms use the
same model, data, precision, batch size, and token budget. The matched rerun
uses FP32 trainable parameters and optimizer state with BF16 autocast, sequence
length 512, micro-batch size 32, gradient accumulation 32, and 524,288 tokens
per optimizer update. We compare all methods at 458.8M training tokens over
three paired seeds. 
% Train loss is taken from the nearest training log entry
% before the final evaluation.
We tune the AdamW control with a short seed-43 learning-rate pilot
(Table~\ref{tab:app_protein_adamw_lr_pilot}), select
$\eta=6\times10^{-4}$, and then use this setting for the matched three-seed
comparison in Table~\ref{tab:app_protein_lm_per_seed}.

% \begin{table}[ht]
% \centering
% \caption{
% Seed-43 AdamW learning-rate pilot for the protein language-modeling control.
% All runs use the same 363M-parameter RoPE decoder, UniRef50 / ProtGPT2-BPE token
% cache, 500 training steps / 32.77M tokens, constant learning rate,
% $\beta_1=0.9$, $\beta_2=0.95$, $\epsilon=10^{-8}$, weight decay $0.1$,
% gradient clipping at norm $1.0$, FP32 trainable parameters and optimizer state
% with BF16 autocast. We use $\eta=6\times10^{-4}$ for the longer AdamW
% comparison.
% }
% \label{tab:app_protein_adamw_lr_pilot}
% \begin{tabular}{lcc}
% \toprule
% AdamW LR & Eval loss $\downarrow$ & Eval PPL $\downarrow$ \\
% \midrule
% $1\times10^{-4}$ & 8.2669 & 3892.69 \\
% $3\times10^{-4}$ & 8.2007 & 3643.67 \\
% $6\times10^{-4}$ & \textbf{8.1369} & \textbf{3418.32} \\
% $1\times10^{-3}$ & 8.1672 & 3523.36 \\
% \bottomrule
% \end{tabular}
% \end{table}
\par\vspace{0.8em}
\noindent
\begin{table}[ht]
\centering
\caption{
Per-seed protein language-modeling results on UniRef50 / ProtGPT2-BPE. All runs
use the matched large-batch protocol and are compared at 458.8M training tokens. The final row reports a paired two-sided $t$-test over eval loss between AF-Muon and the strongest non-AF baseline by mean eval loss; the gain is baseline minus AF-Muon, so positive values indicate that AF-Muon is better.
}
\label{tab:app_protein_lm_per_seed}
\resizebox{\linewidth}{!}{%
\begin{tabular}{lccccccc}
\toprule
Optimizer & Seed & Train step & Eval step & Train loss & Eval loss & Eval ppl & Best loss \\
\midrule
AdamW & 43 & 850 & 875 & 7.9344 & 8.0297 & 3070.73 & 7.8998 \\
AdamW & 44 & 850 & 875 & 7.9773 & 7.8709 & 2619.91 & 7.8709 \\
AdamW & 45 & 850 & 875 & 7.9421 & 7.8737 & 2627.15 & 7.8737 \\
\midrule
Hybrid Muon & 43 & 850 & 875 & 7.4835 & 6.8770 & 969.68 & 6.8770 \\
Hybrid Muon & 44 & 850 & 875 & 7.4868 & 6.8806 & 973.21 & 6.8806 \\
Hybrid Muon & 45 & 850 & 875 & 7.4633 & 6.8650 & 958.14 & 6.8650 \\
\midrule
SCION-style Sign & 43 & 850 & 875 & 7.4729 & 6.8208 & 916.68 & 6.8208 \\
SCION-style Sign & 44 & 850 & 875 & 7.4727 & 6.8260 & 921.48 & 6.8260 \\
SCION-style Sign & 45 & 850 & 875 & 7.4473 & 6.8023 & 899.93 & 6.8023 \\
\midrule
AF-Muon & 43 & 850 & 875 & \textbf{7.4244} & \textbf{6.7469} & \textbf{851.43} & \textbf{6.7469} \\
AF-Muon & 44 & 850 & 875 & \textbf{7.4455} & \textbf{6.7814} & \textbf{881.26} & \textbf{6.7814} \\
AF-Muon & 45 & 850 & 875 & \textbf{7.4077} & \textbf{6.7502} & \textbf{854.27} & \textbf{6.7502} \\
\midrule
\multicolumn{2}{l}{AF-Muon vs.\ SCION-style Sign (paired $t$-test)}
& \multicolumn{6}{c}{eval-loss gain $=0.0569$, $p=0.0231$} \\
\bottomrule
\end{tabular}%
}
\end{table}

AF-Muon gives the lowest validation loss and perplexity for each seed. Averaged
over the three seeds, AF-Muon improves validation loss by 0.1147 over Hybrid
Muon and by 0.0569 over the SCION-style Sign endpoint. This matched rerun
therefore preserves the protein-domain trend under the same large-batch
optimizer protocol used for the main tied-token language-model experiments.

% \clearpage
\subsection{Training Overhead}
\label{app:training_overhead}
% \paniz{I'm not sure but maybe it's better to move this subsection to other section}\\
We measure practical training overhead under the same matched large-batch
protocol used in the main experiments. For each model, the three optimizer arms
are run sequentially on the same GPU, with evaluation, checkpointing, and
diagnostic logging disabled during the timed region. All runs use FP32 trainable
parameters and optimizer state with BF16 autocast, 524k tokens per optimizer
update, matrix weight decay $0.1$, Hybrid Muon auxiliary/tied weight decay
$0.01$, and zero auxiliary/tied weight decay for SCION-style Sign and AF-Muon.
Median step time is computed after discarding the first five steps as warmup,
and peak allocated memory is measured with
\texttt{torch.cuda.max\_memory\_allocated()} after resetting CUDA peak statistics
immediately before the training loop.

The primary memory saving is in optimizer state: SCION-style Sign and AF-Muon
remove the AdamW second-moment buffer from the auxiliary/tied branch, reducing
optimizer-state memory by roughly $18$--$22\%$ relative to Hybrid Muon in these
models. Median step-time overhead of AF-Muon relative to Hybrid Muon is at most
$1.3\%$ across models, and SCION-style Sign is within $0.4\%$; removing the
second-moment buffer from the auxiliary/tied branch therefore costs negligible
throughput. The corresponding reduction in single-GPU peak allocated memory is
smaller because these large-batch runs are activation- and temporary-buffer
dominated. Across models, absolute throughput and peak memory are not directly
comparable because the model size, micro-batch size, and GPU differ; the relevant
comparison is the within-model difference between optimizer arms. NanoGPT,
SmolLM2-135M, and Qwen2.5-0.5B are benchmarked on NVIDIA A100 40GB GPUs, while
Llama-3.2-1B is benchmarked on an NVIDIA H200 GPU.

\begin{table}[ht]
\centering
\small
\caption{
Training-overhead benchmark under the matched large-batch protocol. Optimizer
state memory is computed analytically from the parameter partition and FP32
optimizer states. ``Saved'' reports the optimizer-state reduction relative to
Hybrid Muon within the same model. Peak allocated memory is measured exactly with
\texttt{torch.cuda.max\_memory\_allocated()} after resetting CUDA peak statistics
before the training loop. Step time is the median over 30 training steps after
discarding the first five steps as warmup. The short-run training loss is
included as a sanity check; long-horizon performance is reported in the main
experiment tables.
}
\label{tab:training_overhead}
\resizebox{\linewidth}{!}{%
\begin{tabular}{llccccccc}
\toprule
Model & Optimizer & Opt. state & Saved & Peak alloc. & Step (s) & Tok/s & Rel. thr. & Train loss \\
\midrule
NanoGPT & Hybrid Muon & 0.63 GiB & -- & 18.603 GiB & 6.306 & 83.1k & 1.000$\times$ & 6.5425 \\
NanoGPT & SCION-style Sign & 0.49 GiB & 22.2\% & 18.461 GiB & 6.331 & 82.8k & 0.996$\times$ & 6.2402 \\
NanoGPT & AF-Muon & 0.49 GiB & 22.2\% & 18.461 GiB & 6.386 & 82.1k & 0.987$\times$ & 6.1697 \\
\midrule
SmolLM2-135M & Hybrid Muon & 0.55 GiB & -- & 25.000 GiB & 11.858 & 44.2k & 1.000$\times$ & 6.8592 \\
SmolLM2-135M & SCION-style Sign & 0.45 GiB & 18.2\% & 24.895 GiB & 11.864 & 44.2k & 1.000$\times$ & 6.6022 \\
SmolLM2-135M & AF-Muon & 0.45 GiB & 18.2\% & 24.895 GiB & 11.928 & 44.0k & 0.994$\times$ & 6.5485 \\
\midrule
Qwen2.5-0.5B & Hybrid Muon & 2.35 GiB & -- & 23.260 GiB & 24.193 & 21.7k & 1.000$\times$ & 6.8713 \\
Qwen2.5-0.5B & SCION-style Sign & 1.84 GiB & 21.7\% & 22.752 GiB & 24.245 & 21.6k & 0.998$\times$ & 6.7913 \\
Qwen2.5-0.5B & AF-Muon & 1.84 GiB & 21.7\% & 22.752 GiB & 24.459 & 21.4k & 0.989$\times$ & 6.6253 \\
\midrule
Llama-3.2-1B & Hybrid Muon & 5.58 GiB & -- & 74.766 GiB & 16.330 & 32.1k & 1.000$\times$ & 6.7290 \\
Llama-3.2-1B & SCION-style Sign & 4.60 GiB & 17.6\% & 73.787 GiB & 16.395 & 32.0k & 0.996$\times$ & 6.7475 \\
Llama-3.2-1B & AF-Muon & 4.60 GiB & 17.6\% & 73.787 GiB & 16.486 & 31.8k & 0.990$\times$ & 6.7018 \\
\bottomrule
\end{tabular}%
}
\end{table}
%----------------------------
\newpage
\section{Ablations}
\label{sec:ablation}

\subsection{NanoGPT Factorial Ablation: Tied Table Versus 1D Auxiliary Parameters}
\label{app:nanogpt_factorial}

A potential concern is that AF-Muon changes both the tied vocabulary table and
the remaining one-dimensional auxiliary parameters relative to Hybrid Muon.
Therefore, the improvement over Hybrid Muon might be caused by replacing the
one-dimensional AdamW auxiliary branch rather than by the proposed finite-cap
tied-table geometry. We test this with a $2\times2$ factorial ablation on
NanoGPT.

All runs use the same NanoGPT architecture and training protocol as in the main
NanoGPT experiment: FP32 trainable parameters and optimizer state with BF16
autocast, constant learning rates, matrix weight decay $0.1$, maximum gradient
norm $1.0$, and the same FineWeb token stream. The matrix parameters are always
optimized with Muon. We independently choose the optimizer used for the tied
vocabulary table and the optimizer used for the remaining one-dimensional
parameters:
\[
\{\text{AdamW tied},\ \text{AF-Muon tied }(c=3,s=0.5)\}
\times
\{\text{AdamW 1D},\ \text{RMS-LMO 1D}\}.
\]
All four factorial arms use zero auxiliary weight decay, so the comparison
isolates update geometry rather than adding an extra shrinkage term. (Figure~\ref{fig:app_nanogpt_factorial_curves})

\begin{table}[ht]
\centering
\caption{
NanoGPT $2\times2$ factorial ablation separating the tied vocabulary-table
update from the remaining one-dimensional auxiliary update. All runs use 1.05B
training tokens and three paired seeds.
}
\label{tab:app_nanogpt_factorial}
\begin{tabular}{llcc}
\toprule
Tied table & 1D aux. & Val. loss $\downarrow$ & Val. PPL $\downarrow$ \\
\midrule
AdamW & AdamW & 3.6296 $\pm$ 0.0030 & 37.70 $\pm$ 0.11 \\
AF-Muon & AdamW & 3.6163 $\pm$ 0.0108 & 37.20 $\pm$ 0.40 \\
AdamW & RMS-LMO & 3.6221 $\pm$ 0.0028 & 37.41 $\pm$ 0.11 \\
AF-Muon & RMS-LMO & \textbf{3.5883 $\pm$ 0.0016} & \textbf{36.17 $\pm$ 0.06} \\
\bottomrule
\end{tabular}
\end{table}

% \begin{figure}[ht]
%     \centering
%     \includegraphics[width=\linewidth]{figs/nanogpt_factorial_tied_ablate/fig_nanogpt_factorial_tied_vs_1d_curves_zoom_x1p1.pdf}
%     \caption{
%     NanoGPT factorial ablation curves. Curves show the mean over three paired
%     seeds and shaded regions denote $\pm 1$ standard error of the mean. The
%     figure focuses on the post-warmup/tail region where the four factorial
%     recipes separate most clearly.
%     }
%     \label{fig:app_nanogpt_factorial_curves}
% \end{figure}

The isolated tied-table replacement improves validation loss by $0.0133$ over
the AdamW/AdamW baseline, while the isolated one-dimensional RMS-LMO replacement
improves validation loss by $0.0075$. Replacing both gives a substantially larger
$0.0413$ improvement. This shows that the full AF-Muon result is not merely a
side effect of changing the one-dimensional auxiliary optimizer. The tied-table
finite-cap geometry contributes directly, and the best result appears when it is
combined with RMS-LMO updates for the remaining auxiliary parameters.

%----------------------------
% \newpage
\subsection{Real Vocabulary-Size Re-tokenization Ablation}
\label{app:nanogpt_retokenized_vocab}

This experiment tests whether vocabulary size alone explains why AF-Muon and
the Sign endpoint are closer in some large-vocabulary FP32 runs. We train
matched ByteLevel BPE tokenizers with 50k, 100k, and 150k vocabulary sizes on
the same FineWeb shard and re-tokenize the data for each vocabulary. This is a
real re-tokenization test rather than a padded-vocabulary test: token IDs reach
the top of each vocabulary in every setting.

Because different tokenizers induce different per-token entropy scales, absolute
losses should not be compared across vocabulary sizes. The relevant comparison
is within each vocabulary. Across all three vocabularies, AF-Muon improves over
the SCION-style Sign endpoint: by 0.036 validation-loss points at 50k, 0.045 at
100k, and 0.052 at 150k. Thus the AF-Muon--Sign gap does not shrink as the real
vocabulary grows in this controlled NanoGPT setting.

\begin{table}[ht]
\centering
\caption{
NanoGPT real-vocabulary re-tokenization ablation. We train matched ByteLevel BPE
tokenizers with fix d=768, and 50k, 100k, and 150k vocabulary sizes on the same FineWeb shard,
then train each optimizer for 500M tokens. All runs use FP32 trainable parameters
and optimizer state with BF16 autocast, matrix weight decay $0.1$, zero
auxiliary/tied weight decay, and effective batch size 512. The 50k runs use
microbatch 16; 100k and 150k use microbatch 8 for memory. We report mean $\pm$
sample standard deviation over three paired seeds. related figure \ref{fig:nanogpt_retokenized_vocab}
}
\label{tab:app_nanogpt_retokenized_vocab}
\begin{tabular}{llcccc}
\toprule
Vocab & Optimizer & Params & Train loss $\downarrow$ & Val. loss $\downarrow$ & Val. PPL $\downarrow$ \\
\midrule
\multirow{2}{*}{50k}
& SCION-style Sign & 123.4M & 3.9881 $\pm$ 0.0079 & 3.9260 $\pm$ 0.0014 & 50.70 $\pm$ 0.07 \\
& AF-Muon & 123.4M & \textbf{3.9561 $\pm$ 0.0017} & \textbf{3.8897 $\pm$ 0.0051} & \textbf{48.90 $\pm$ 0.25} \\
\midrule
\multirow{2}{*}{100k}
& SCION-style Sign & 161.8M & 4.0691 $\pm$ 0.0069 & 3.9947 $\pm$ 0.0041 & 54.31 $\pm$ 0.22 \\
& AF-Muon & 161.8M & \textbf{4.0275 $\pm$ 0.0059} & \textbf{3.9500 $\pm$ 0.0061} & \textbf{51.94 $\pm$ 0.32} \\
\midrule
\multirow{2}{*}{150k}
& SCION-style Sign & 200.2M & 4.0881 $\pm$ 0.0108 & 4.0228 $\pm$ 0.0070 & 55.86 $\pm$ 0.39 \\
& AF-Muon & 200.2M & \textbf{4.0405 $\pm$ 0.0055} & \textbf{3.9706 $\pm$ 0.0041} & \textbf{53.02 $\pm$ 0.22} \\
\bottomrule
\end{tabular}
\end{table}

\begin{figure}[ht]
    \centering
    \begin{subfigure}[t]{0.95\linewidth}
        \centering
        \includegraphics[width=0.9\linewidth,
    trim={0.0in 0.05in 0.0in 0.10in},
    clip
]
        {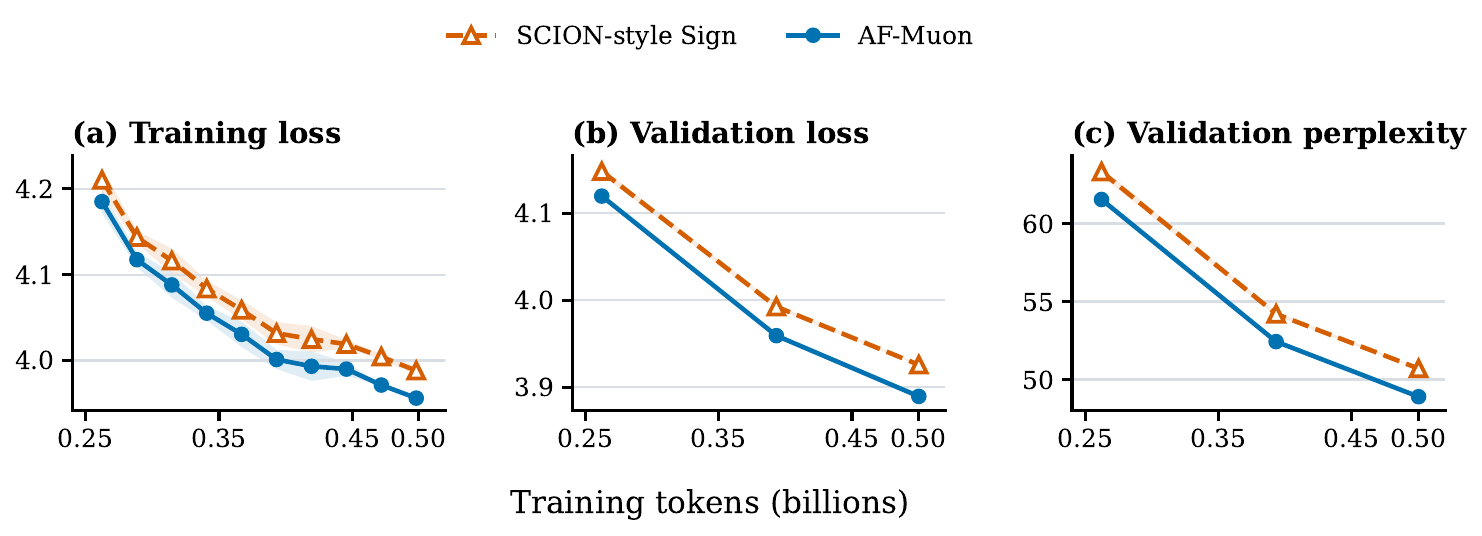}
        \caption{50k vocabulary}
        \label{fig:nanogpt_retok_vocab_50k}
    \end{subfigure}

    \vspace{0.4em}

    \begin{subfigure}[t]{0.95\linewidth}
        \centering
        \includegraphics[width=0.9\linewidth,
    trim={0.0in 0.05in 0.0in 0.6in},
    clip
]{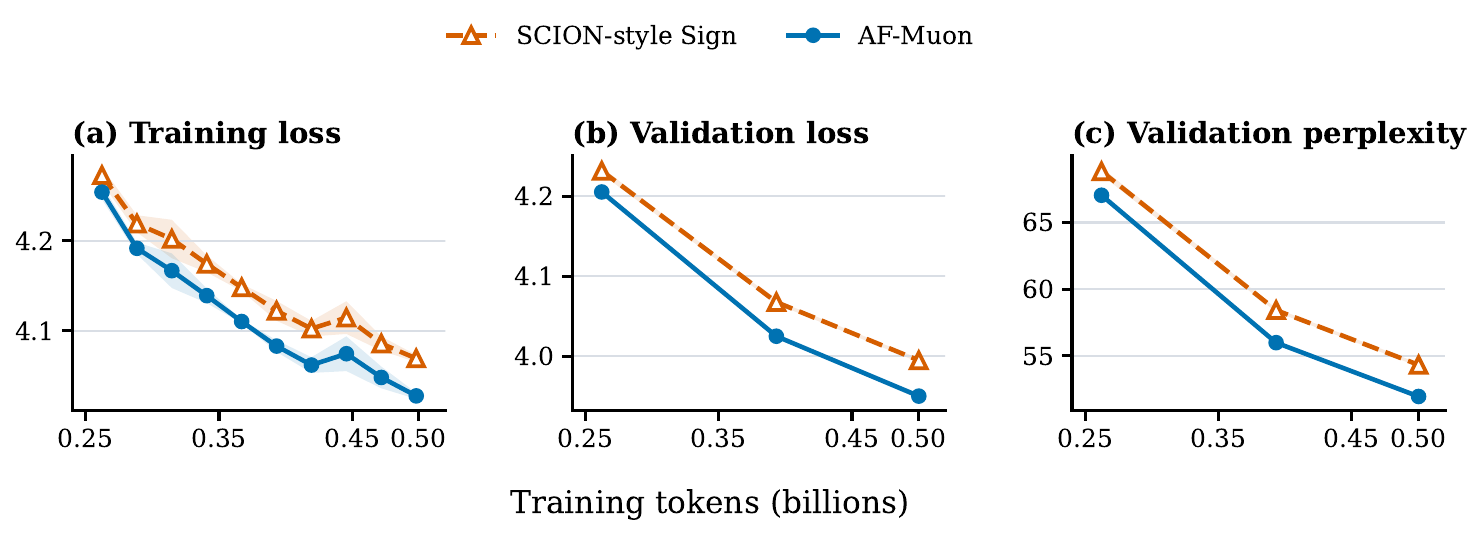}
        \caption{100k vocabulary}
        \label{fig:nanogpt_retok_vocab_100k}
    \end{subfigure}

    \vspace{0.4em}

    \begin{subfigure}[t]{0.95\linewidth}
        \centering
        \includegraphics[width=0.9\linewidth,
    trim={0.0in 0.05in 0.0in 0.6in},
    clip
]{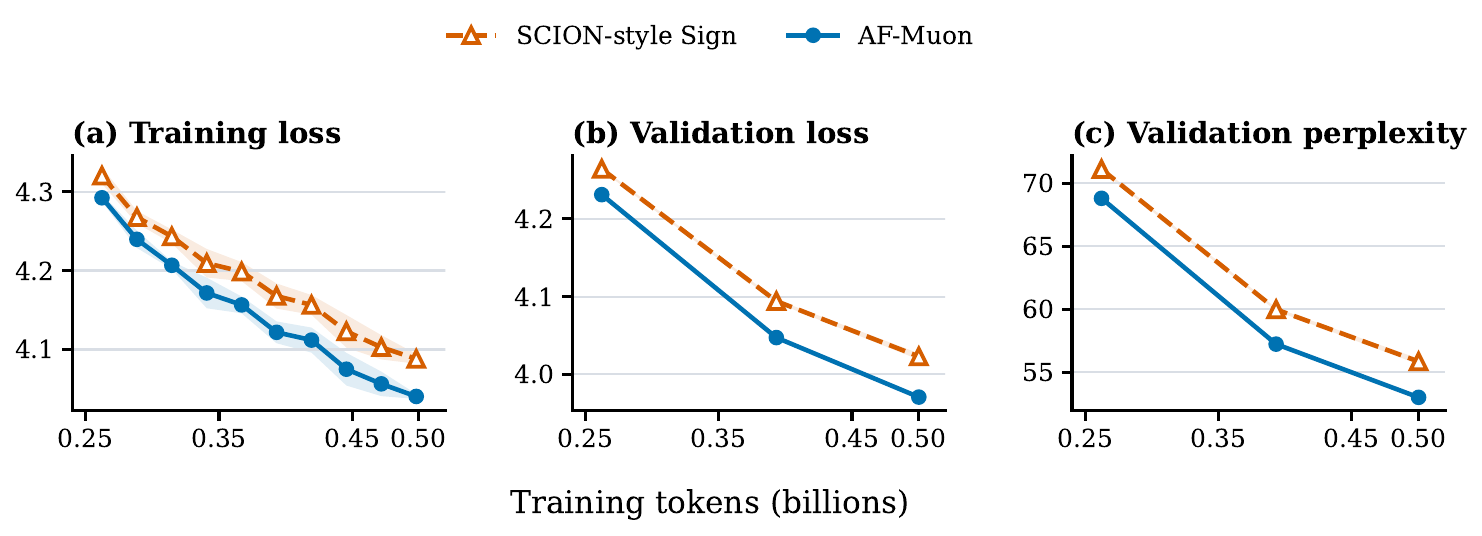}
        \caption{150k vocabulary}
        \label{fig:nanogpt_retok_vocab_150k}
    \end{subfigure}

    \caption{
    Real-vocabulary re-tokenization ablation on NanoGPT. Each row uses a separately
    trained ByteLevel BPE tokenizer with vocabulary size 50k, 100k, or 150k and compares
    AF-Muon with the SCION-style Sign endpoint for 500M training tokens. Curves show the
    mean over three paired seeds, with shaded regions denoting $\pm 1$ standard error.
    Since tokenizers differ across rows, absolute validation losses and perplexities
    should not be compared across vocabulary sizes; the relevant comparison is the
    within-vocabulary AF-Muon versus Sign gap.
    }
    \label{fig:nanogpt_retokenized_vocab}
\end{figure}

%----------------------------
\clearpage
\subsection{Tied-Step Magnitude Ablation}
\label{app:nanogpt_tied_step_scale}

Table~\ref{tab:app_nanogpt_s_scale} reports the NanoGPT tied-scale sweep.
The model architecture, vocabulary, data, precision recipe, matrix weight decay, and token budget are fixed across all rows. The only AF-Muon hyperparameter varied is the tied-table scale $s$, with $c$ fixed to $3$.

Across the tested values, AF-Muon obtains lower validation loss and perplexity than the SCION-style Sign endpoint. The best validation loss in this sweep is obtained at $s=0.25$, while nearby values $s=0.375$ and $s=0.1875$ give similar results. The tied-table RMS decreases as $s$ is reduced, reflecting the smaller
tied-table update scale.
\begin{table}[ht]
\centering
\caption{
NanoGPT AF-Muon tied-scale sweep at fixed model width and vocabulary size.
All runs use a 12-layer NanoGPT model with hidden size $d=768$, a 50k vocabulary,
500M training tokens, FP32 trainable parameters and optimizer state with BF16
autocast, matrix weight decay $0.1$, and zero auxiliary/tied weight decay.
AF-Muon uses $c=3$ and varies the tied scale $s$. The SCION-style Sign endpoint uses $c=1,s=1$ under the same training protocol.
}
\label{tab:app_nanogpt_s_scale}
\begin{tabular}{lcccccc}
\toprule
Method & $s$ & $d_{\mathrm{eff}}$ & Train loss $\downarrow$ & Val. loss $\downarrow$ & Val. PPL $\downarrow$ & Tied RMS \\
\midrule
SCION-style Sign & 1.0000 & -- & 3.9971 & 3.9250 & 50.65 & 0.2147 \\
AF-Muon & 0.5000 & 768 & 3.9558 & 3.8867 & 48.75 & 0.1241 \\
AF-Muon & 0.4300 & 893 & 3.9510 & 3.8799 & 48.42 & 0.1066 \\
AF-Muon & 0.3750 & 1024 & 3.9409 & 3.8697 & 47.93 & 0.0956 \\
AF-Muon & 0.2500 & 1536 & 3.9377 & 3.8687 & 47.88 & 0.0692 \\
AF-Muon & 0.1875 & 2048 & 3.9406 & 3.8720 & 48.04 & 0.0554 \\
AF-Muon & 0.1250 & 3072 & 3.9522 & 3.8856 & 48.70 & 0.0444 \\
\bottomrule
\end{tabular}
\end{table}

\begin{figure}[ht]
    \centering
    \includegraphics[
  width=\linewidth,
  trim=0.0in 0.30in 0.0in 0.15in,
  clip
]{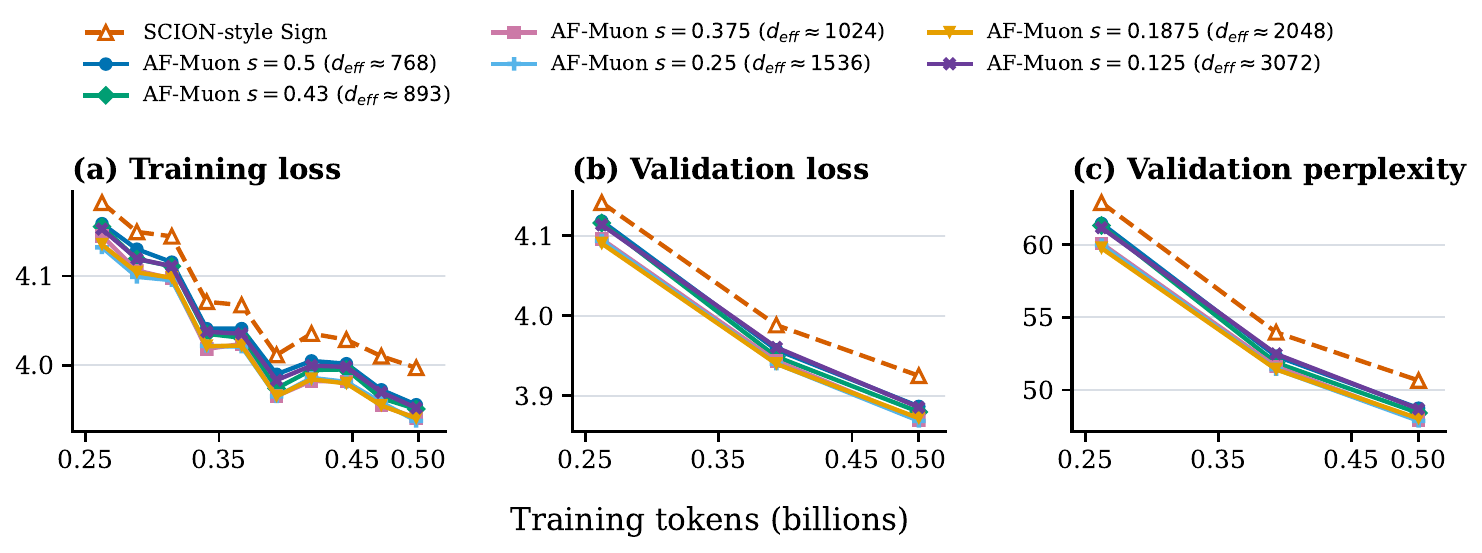}

    \vspace{0.8em}

    \includegraphics[
  width=\linewidth,
  trim=0.0in 0.10in 0.0in 0.15in,
  clip
]{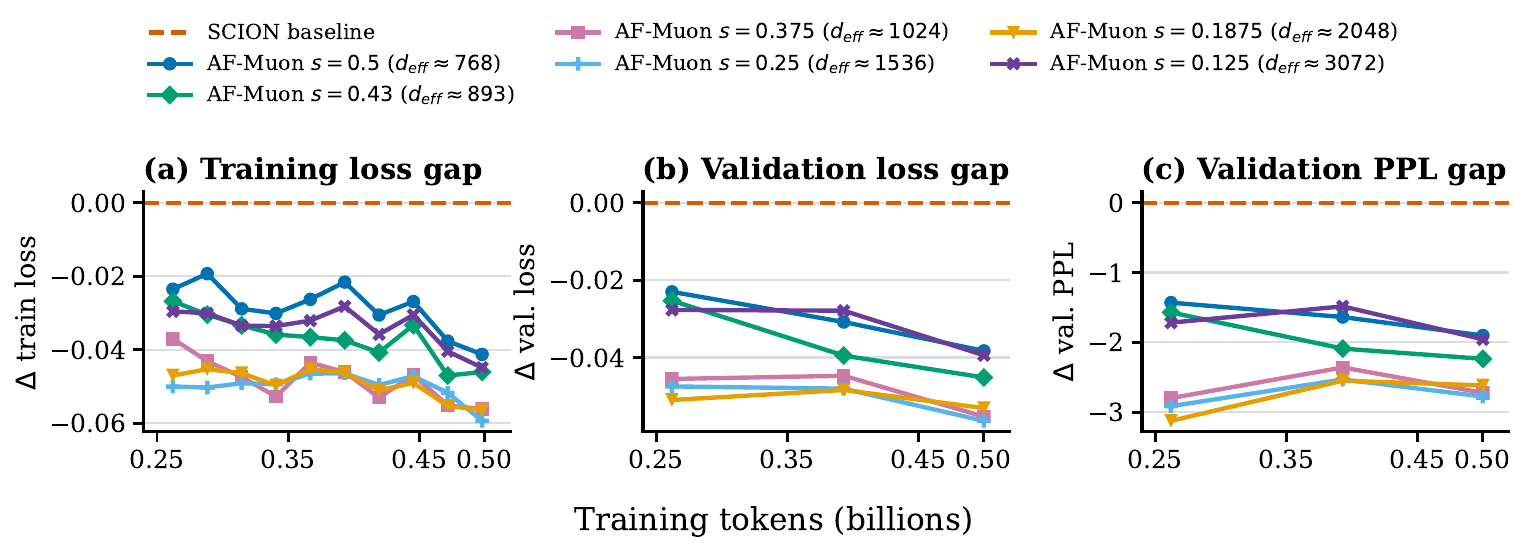}

    \caption{
    NanoGPT tied-scale sweep training dynamics. Top: training loss, validation loss,
    and validation perplexity for SCION-style Sign and AF-Muon with different tied
    scales $s$. Bottom: token-wise AF-Muon minus SCION-style Sign gaps for the same
    metrics. All runs use the same fixed 50k vocabulary, $d=768$ model width,
    500M training-token budget, FP32 trainable parameters and optimizer state with
    BF16 autocast, matrix weight decay $0.1$, and zero auxiliary/tied weight decay.
    }
    \label{fig:app_nanogpt_s_scale_curves}
\end{figure}
%------------------real d
\subsection{NanoGPT Real-Width Ablation}
\label{app:nanogpt_real_width}

We next test whether the AF-Muon advantage over the Sign endpoint persists when
the tied vocabulary rows have larger real hidden dimension. Unlike the tied-scale
diagnostic, this ablation changes the actual model width and tied-row dimension.
All runs use a 50k ByteLevel BPE tokenizer, 500M training tokens, FP32 trainable
parameters and optimizer state with BF16 autocast, matrix weight decay $0.1$,
zero tied/auxiliary weight decay for the LMO methods, and three paired seeds.
AF-Muon uses $c=3,s=0.5$; the SCION-style Sign endpoint uses $c=1,s=1$.

\begin{table}[ht]
\centering
\caption{
Real-width NanoGPT ablation at fixed 50k vocabulary and 500M training tokens.
We report mean $\pm$ sample standard deviation over three paired seeds.
Rows clipped denotes the fraction of tied-table rows for which the finite-cap
constraint is active in the final tied-momentum diagnostic.
}
\label{tab:app_nanogpt_real_width}
\begin{tabular}{llccccc}
\toprule
Width $d$ & Optimizer & Params & Train loss $\downarrow$ & Val. loss $\downarrow$ & Val. PPL $\downarrow$ \\
\midrule
768 & SCION-style Sign & 123.4M & 3.9881 $\pm$ 0.0079 & 3.9260 $\pm$ 0.0014 & 50.70 $\pm$ 0.07 \\
768 & AF-Muon & 123.4M & \textbf{3.9561 $\pm$ 0.0017} & \textbf{3.8897 $\pm$ 0.0051} & \textbf{48.90 $\pm$ 0.25} \\
\midrule
1024 & SCION-style Sign & 202.2M & 3.9205 $\pm$ 0.0035 & 3.8464 $\pm$ 0.0005 & 46.82 $\pm$ 0.02 \\
1024 & AF-Muon & 202.2M & \textbf{3.8822 $\pm$ 0.0026} & \textbf{3.8068 $\pm$ 0.0026} & \textbf{45.00 $\pm$ 0.12} \\
\midrule
1536 & SCION-style Sign & 416.6M & 3.8475 $\pm$ 0.0020 & 3.7573 $\pm$ 0.0041 & 42.83 $\pm$ 0.18 \\
1536 & AF-Muon & 416.6M & \textbf{3.8136 $\pm$ 0.0031} & \textbf{3.7253 $\pm$ 0.0051} & \textbf{41.48 $\pm$ 0.21} \\
\bottomrule
\end{tabular}
\end{table}

AF-Muon remains better than the Sign endpoint at all three real widths. The
paired AF-Muon minus Sign validation-loss gaps are $-0.0363\pm0.0063$ at
$d=768$, $-0.0397\pm0.0023$ at $d=1024$, and $-0.0321\pm0.0085$ at $d=1536$.
Thus increasing real hidden width attenuates neither the sign nor the
consistency of the finite-cap advantage in this NanoGPT setting, although the
row-clipping diagnostic moves closer to the Sign endpoint as width increases.

\begin{table}[ht]
\centering
\caption{
Per-seed real-width NanoGPT results for the width ablation in
Table~\ref{tab:app_nanogpt_real_width}.
}
\label{tab:app_nanogpt_real_width_per_seed}
\begin{tabular}{llcccc}
\toprule
Width $d$ & Optimizer & Seed & Train loss $\downarrow$ & Val. loss $\downarrow$ & Val. PPL $\downarrow$ \\
\midrule
768 & SCION-style Sign & 43 & 3.9971 & 3.9250 & 50.65 \\
768 & SCION-style Sign & 44 & 3.9844 & 3.9253 & 50.67 \\
768 & SCION-style Sign & 45 & 3.9827 & 3.9276 & 50.79 \\
768 & AF-Muon & 43 & 3.9558 & 3.8867 & 48.75 \\
768 & AF-Muon & 44 & 3.9546 & 3.8956 & 49.19 \\
768 & AF-Muon & 45 & 3.9580 & 3.8869 & 48.76 \\
\midrule
1024 & SCION-style Sign & 43 & 3.9235 & 3.8465 & 46.83 \\
1024 & SCION-style Sign & 44 & 3.9214 & 3.8469 & 46.85 \\
1024 & SCION-style Sign & 45 & 3.9166 & 3.8459 & 46.80 \\
1024 & AF-Muon & 43 & 3.8850 & 3.8092 & 45.11 \\
1024 & AF-Muon & 44 & 3.8818 & 3.8071 & 45.02 \\
1024 & AF-Muon & 45 & 3.8799 & 3.8040 & 44.88 \\
\midrule
1536 & SCION-style Sign & 43 & 3.8487 & 3.7555 & 42.75 \\
1536 & SCION-style Sign & 44 & 3.8451 & 3.7544 & 42.71 \\
1536 & SCION-style Sign & 45 & 3.8485 & 3.7621 & 43.04 \\
1536 & AF-Muon & 43 & 3.8168 & 3.7308 & 41.71 \\
1536 & AF-Muon & 44 & 3.8106 & 3.7243 & 41.44 \\
1536 & AF-Muon & 45 & 3.8134 & 3.7207 & 41.29 \\
\bottomrule
\end{tabular}
\end{table}

\begin{table}[ht]
\centering
\caption{
Paired AF-Muon minus SCION-style Sign gaps for the real-width NanoGPT ablation.
Negative values indicate that AF-Muon is better.
}
\label{tab:app_nanogpt_real_width_gaps}
\begin{tabular}{lccc}
\toprule
Width $d$ & Seeds & $\Delta$ Val. loss & $\Delta$ Val. PPL \\
\midrule
768 & 43,44,45 & $-0.0363 \pm 0.0058$ & $-1.81 \pm 0.28$ \\
1024 & 43,44,45 & $-0.0397 \pm 0.0023$ & $-1.82 \pm 0.10$ \\
1536 & 43,44,45 & $-0.0321 \pm 0.0085$ & $-1.35 \pm 0.36$ \\
\bottomrule
\end{tabular}
\end{table}

\begin{figure}[ht]
\centering

\textbf{(a) Hidden width $d=768$}
\includegraphics[
  width=\linewidth,
  trim=0.0in 0.30in 0.0in 0.15in,
  clip
]{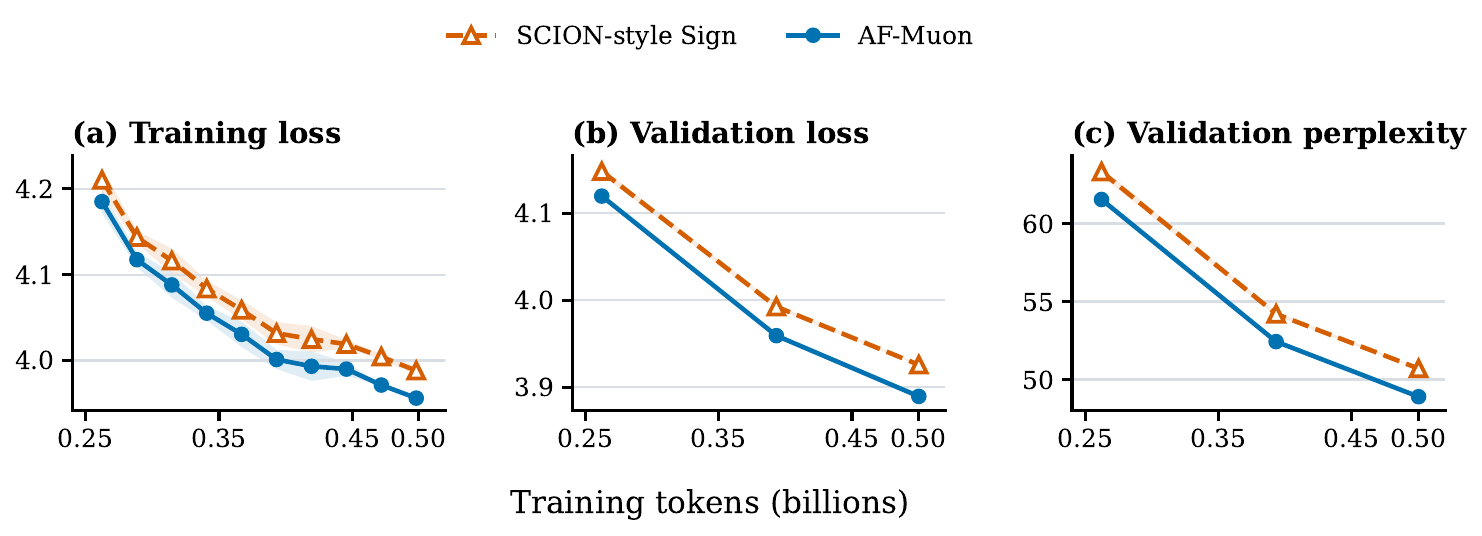}

\vspace{0.6em}
\textbf{(b) Hidden width $d=1024$}
\includegraphics[
  width=\linewidth,
  trim=0.0in 0.30in 0.0in 0.15in,
  clip
]{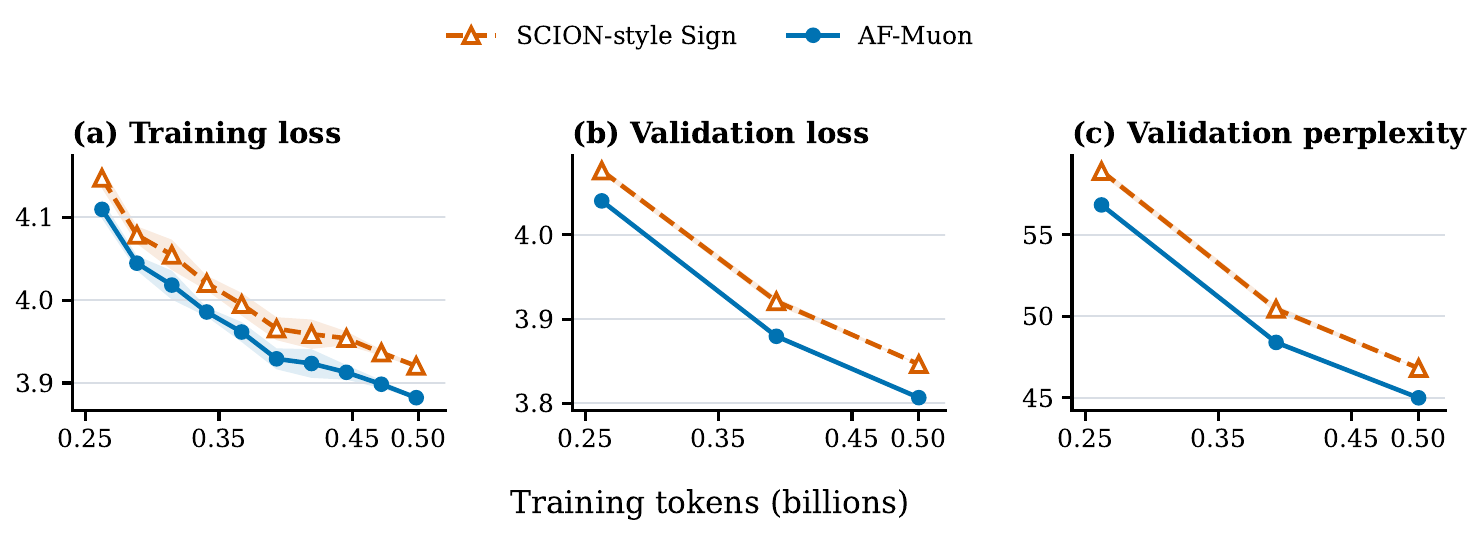}

\vspace{0.6em}
\textbf{(c) Hidden width $d=1536$}
\includegraphics[
  width=\linewidth,
  trim=0.0in 0.10in 0.0in 0.15in,
  clip
]{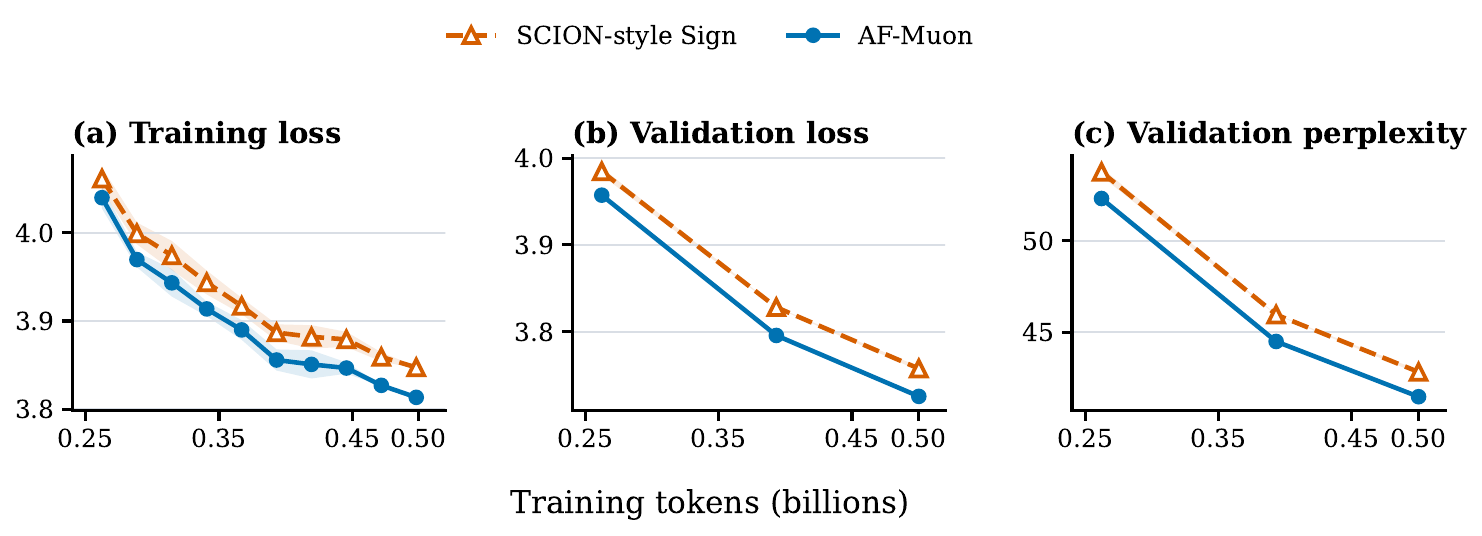}

\caption{
Real-width NanoGPT training dynamics at fixed 50k vocabulary and 500M training tokens.
Each row compares AF-Muon and the SCION-style Sign endpoint for one hidden
width $d \in \{768,1024,1536\}$. Curves show the mean over three paired seeds,
with shaded regions denoting $\pm 1$ standard error.
}
\label{fig:nanogpt_real_width_curves}
\end{figure}

\clearpage
\subsection{Batch-size Sensitivity}
\label{app:nanogpt_batch_sensitivity}

We next test whether the AF-Muon tied-table advantage depends on the particular
large-batch setting used in the main NanoGPT experiment. We use the same
NanoGPT-style tied-embedding model, FineWeb tokenization, optimizer
hyperparameters, FP32 trainable parameters and optimizer state with BF16
autocast, matrix learning rate $0.02$, auxiliary/vector learning rate
$3\times10^{-4}$, matrix weight decay $0.1$, and gradient clipping at norm
$1.0$. We vary only the global batch size, keeping the micro-batch size fixed at 16 and changing gradient accumulation accordingly. All methods are compared at the same token budget of 458.8M tokens.

\begin{table}[ht]
\centering
\caption{
NanoGPT batch-size sensitivity at a matched 458.8M-token comparison point.
All runs use seed 43. Tokens/update is the number of training tokens per
optimizer step.
}
\label{tab:app_nanogpt_batch_sensitivity}
\resizebox{\linewidth}{!}{%
\begin{tabular}{cccccc}
\toprule
Batch & Tokens/update & Hybrid Muon & SCION-style Sign & AF-Muon & AF--Sign \\
\midrule
32  & 32,768  & 4.2078 & 4.1472 & \textbf{4.1378} & -0.0094 \\
64  & 65,536  & 3.9918 & 3.9738 & \textbf{3.9514} & -0.0224 \\
128 & 131,072 & 3.8571 & 3.8559 & \textbf{3.8180} & -0.0379 \\
256 & 262,144 & 3.7950 & 3.7890 & \textbf{3.7544} & -0.0346 \\
512 & 524,288 & 3.8200 & 3.7649 & \textbf{3.7400} & -0.0249 \\
\bottomrule
\end{tabular}%
}
\end{table}

AF-Muon gives the lowest validation loss at every tested batch size, from
32k to 524k tokens per optimizer update. The advantage over the Sign endpoint is smallest at the smallest batch and becomes clearer for the intermediate and large-batch settings. This indicates that the finite-cap tied-table update is not narrowly tuned to the default large-batch protocol: it remains beneficial across a broad range of effective batch sizes under the same matrix-update and
precision settings.
% \begin{figure}[ht]
%     \centering
%     \includegraphics[width=\linewidth]{figs/batch_size_ablation/nanogpt_batch_sensitivity_458m.pdf}
%     \caption{
%     NanoGPT batch-size sensitivity at a matched 458.8M-token comparison point. Each point is a seed-43 run with the same model, data, optimizer
%     hyperparameters, FP32 trainable parameters and optimizer state with BF16
%     autocast, matrix learning rate $0.02$, auxiliary/vector learning rate
%     $3\times10^{-4}$, matrix weight decay $0.1$, and gradient clipping at norm $1.0$. The x-axis varies only the number of tokens per optimizer update. AF-Muon obtains the lowest training loss, validation loss, and validation perplexity across all tested batch sizes.
%     }
%     \label{fig:app_nanogpt_batch_sensitivity}
% \end{figure}

Figure~\ref{fig:app_nanogpt_batch_sensitivity} visualizes the same matched-token comparison. AF-Muon remains below both baselines across the full tested range of global batch sizes, indicating that the finite-cap tied-table update is not narrowly tuned to a single effective batch size.

%--------------------
% \newpage
\subsection{Auxiliary Learning-rate Sensitivity}
\label{app:nanogpt_auxlr_sensitivity}

\paragraph{Auxiliary/fallback learning-rate sensitivity.}
Throughout the main experiments, we use the standard Muon parameter grouping:
hidden matrices are updated by Muon with learning rate $0.02$, while parameters
outside the hidden matrix group, such as embeddings, classifier heads,
normalization scales, and biases, are assigned to an auxiliary AdamW-style branch
with learning rate $3\times 10^{-4}$ in the reference Muon recipe
\citep{jordan2024muon}. The public Muon implementation also describes this
learning rate as being in spectral-norm units for hidden matrices and recommends
AdamW for the non-hidden parameters. We apply this grouping convention across
models, replacing the AdamW treatment of the tied vocabulary table by the
SCION-style Sign endpoint or by AF-Muon where appropriate, while keeping the
hidden-matrix Muon learning rate fixed.

To test whether the NanoGPT comparison depends on a single auxiliary scale,
Table~\ref{tab:app_nanogpt_auxlr_sensitivity} sweeps the auxiliary/fallback
learning rate around and beyond the Muon default. We use the same NanoGPT-style
tied-embedding model and FineWeb tokenization as in
Appendix~\ref{app:nanogpt_tied_vocab}. All runs use seed 43, FP32 trainable
parameters and optimizer state with BF16 autocast, matrix learning rate $0.02$,
matrix weight decay $0.1$, gradient clipping at norm $1.0$, batch size 512,
micro-batch size 16, and 524{,}288 tokens per optimizer update. We vary the
auxiliary/fallback learning rate over
\[
7.5\times10^{-5},\;1.5\times10^{-4},\;3\times10^{-4},\;
6\times10^{-4},\;1.2\times10^{-3},\;2.4\times10^{-3},\;9.6\times10^{-3},
\]
corresponding to $0.25\times$ through $32\times$ the default
$3\times10^{-4}$. For Hybrid Muon, this branch includes the tied vocabulary
table; for the AdamW-free tied-table arms, the tied table is updated by the
Sign or AF-Muon LMO and the sweep tests the remaining auxiliary/vector scale.
All methods are compared at the same 458.8M-token point validation loss.

\begin{table}[ht]
\centering
\caption{
NanoGPT auxiliary/fallback learning-rate sensitivity at a matched 458.8M-token comparison point. All runs use seed 43. The hidden-matrix Muon learning rate is fixed at $0.02$; the auxiliary/fallback learning-rate scale is varied. For Hybrid Muon this branch includes the tied vocabulary table, while the AdamW-free arms use their tied-table LMO and vary the remaining auxiliary/vector scale.
}
\label{tab:app_nanogpt_auxlr_sensitivity}
\resizebox{\linewidth}{!}{%
\begin{tabular}{ccccccc}
\toprule
Aux/tied LR & Mult. & Hybrid Muon & SCION-style Sign & AF-Muon & AF--Sign & AF--Hybrid \\
\midrule
$7.5\times10^{-5}$ & $0.25\times$ & 4.1498 & 3.7743 & \textbf{3.7557} & -0.0185 & -0.3941 \\
$1.5\times10^{-4}$ & $0.5\times$  & 3.9716 & 3.7664 & \textbf{3.7472} & -0.0192 & -0.2244 \\
$3.0\times10^{-4}$ & $1.0\times$  & 3.8182 & 3.7726 & \textbf{3.7337} & -0.0389 & -0.0845 \\
$6.0\times10^{-4}$ & $2.0\times$  & 3.7641 & 3.7522 & \textbf{3.7284} & -0.0238 & -0.0357 \\
$1.2\times10^{-3}$ & $4.0\times$  & 3.7304 & 3.7393 & \textbf{3.7220} & -0.0173 & -0.0084 \\
$2.4\times10^{-3}$ & $8.0\times$  & 3.7233 & 3.7283 & \textbf{3.7208} & -0.0075 & -0.0024 \\
$9.6\times10^{-3}$ & $32.0\times$ & 3.7393 & \textbf{3.7371} & 3.7372 & 0.0001 & -0.0021 \\
\bottomrule
\end{tabular}%
}
\end{table}

The sweep shows that Hybrid Muon is sensitive to the auxiliary AdamW branch: increasing this learning rate substantially narrows the gap relative to AF-Muon.
However, AF-Muon remains competitive across the extended range and attains the best loss at most tested settings before the largest rates degrade all methods.
We therefore view this as a sensitivity diagnostic under literature-motivated Muon defaults, not as a claim of compute-optimal tuning for every optimizer.

% \begin{figure}[ht]
%     \centering
%     \includegraphics[width=\linewidth]{figs/lr_ablation/nanogpt_auxlr_sensitivity.pdf}
%     \caption{
%     NanoGPT auxiliary/tied learning-rate sensitivity at a matched 458.8M-token
%     comparison point. The matrix learning rate is fixed at $0.02$, while the
%     auxiliary/tied learning rate is swept over a $16\times$ range around the
%     default value $3\times10^{-4}$. Each point is a seed-43 run using the same tied-embedding model, data, batch size, precision, matrix weight decay, and gradient clipping protocol as the main NanoGPT experiment. AF-Muon remains best across the full sweep.
%     }
%     \label{fig:app_nanogpt_auxlr_sensitivity}
% \end{figure}

% Figure~\ref{fig:app_nanogpt_auxlr_sensitivity} visualizes the same sweep.
% Hybrid Muon is highly sensitive to reducing the auxiliary AdamW learning rate: at $0.25\times$ and $0.5\times$ the tied/auxiliary branch substantially under-trains. SCION-style Sign is flatter across the sweep, but AF-Muon remains below both baselines at every tested multiplier. This supports the view that the AF-Muon improvement is not caused by a fragile choice of the default auxiliary/tied learning rate.
%-------------------- cap sweep
% \newpage
\subsection{AF-Muon Cap and Scale Sensitivity}
\label{app:nanogpt_cap_scale_sensitivity}
We evaluate sensitivity to the tied-table coordinate cap $c$ and
update scale $s$ using the NanoGPT-style tied-vocabulary model
and large-batch protocol of Appendix~\ref{app:nanogpt_tied_vocab}. We test
$c\in\{1,2,3,4,6,10,\infty\}$ and $s\in\{0.5,1\}$ with three matched seeds per configuration, for a total of 42 training runs. All configurations use the same NanoGPT tied-vocabulary task and are evaluated at optimizer step 750, corresponding to 393.2M training tokens. The endpoints $c=1$ and $c=\infty$ recover Sign and row-RMS, respectively.

\begin{table}[ht]
\centering
\caption{
AF-Muon cap and scale sensitivity on the NanoGPT tied-vocabulary task at the
393.2M-token intermediate comparison point. Results are mean $\pm$ SD over
three paired seeds. The default setting used in the main experiments is $c=3$,
$s=0.5$.
}
\label{tab:app_nanogpt_cap_scale_sensitivity_step750}
\resizebox{\linewidth}{!}{%
\begin{tabular}{ccccccc}
\toprule
Cap $c$ & Scale $s$ & Train loss & Val. loss & Val. PPL & Tied RMS & Tied max \\
\midrule
1 & 0.5 & 3.7880 $\pm$ 0.0208 & 3.7831 $\pm$ 0.0033 & 43.95 $\pm$ 0.15 & 0.1270 $\pm$ 0.0008 & 0.6653 $\pm$ 0.0033 \\
2 & 0.5 & 3.7885 $\pm$ 0.0250 & 3.7826 $\pm$ 0.0021 & 43.93 $\pm$ 0.09 & 0.1444 $\pm$ 0.0004 & 1.1725 $\pm$ 0.0324 \\
3 & 0.5 & \textbf{3.7872 $\pm$ 0.0253} & \textbf{3.7815 $\pm$ 0.0026} & \textbf{43.88 $\pm$ 0.12} & 0.1418 $\pm$ 0.0029 & 1.5712 $\pm$ 0.0923 \\
4 & 0.5 & 3.7892 $\pm$ 0.0241 & 3.7856 $\pm$ 0.0062 & 44.06 $\pm$ 0.27 & 0.1409 $\pm$ 0.0039 & 1.7930 $\pm$ 0.1915 \\
6 & 0.5 & 3.7910 $\pm$ 0.0232 & 3.7845 $\pm$ 0.0005 & 44.01 $\pm$ 0.02 & 0.1408 $\pm$ 0.0017 & 2.1051 $\pm$ 0.2569 \\
10 & 0.5 & 3.7909 $\pm$ 0.0244 & 3.7867 $\pm$ 0.0014 & 44.11 $\pm$ 0.06 & 0.1399 $\pm$ 0.0018 & 2.5313 $\pm$ 0.2197 \\
$\infty$ & 0.5 & 3.7919 $\pm$ 0.0227 & 3.7844 $\pm$ 0.0010 & 44.01 $\pm$ 0.05 & 0.1404 $\pm$ 0.0019 & 2.6954 $\pm$ 0.0666 \\
\midrule
1 & 1 & 3.8168 $\pm$ 0.0301 & 3.8109 $\pm$ 0.0064 & 45.19 $\pm$ 0.29 & 0.2407 $\pm$ 0.0039 & 1.2568 $\pm$ 0.0039 \\
2 & 1 & 3.8250 $\pm$ 0.0243 & 3.8214 $\pm$ 0.0042 & 45.67 $\pm$ 0.19 & 0.2778 $\pm$ 0.0043 & 2.3705 $\pm$ 0.0609 \\
3 & 1 & 3.8288 $\pm$ 0.0221 & 3.8124 $\pm$ 0.0065 & 45.26 $\pm$ 0.30 & 0.2755 $\pm$ 0.0045 & 3.1119 $\pm$ 0.2910 \\
4 & 1 & 3.8291 $\pm$ 0.0289 & 3.8252 $\pm$ 0.0049 & 45.84 $\pm$ 0.23 & 0.2701 $\pm$ 0.0011 & 4.0566 $\pm$ 0.0744 \\
6 & 1 & 3.8285 $\pm$ 0.0213 & 3.8214 $\pm$ 0.0071 & 45.67 $\pm$ 0.32 & 0.2692 $\pm$ 0.0029 & 4.8737 $\pm$ 0.7576 \\
10 & 1 & 3.8199 $\pm$ 0.0220 & 3.8160 $\pm$ 0.0042 & 45.42 $\pm$ 0.19 & 0.2688 $\pm$ 0.0049 & 5.6160 $\pm$ 0.2148 \\
$\infty$ & 1 & 3.8244 $\pm$ 0.0186 & 3.8171 $\pm$ 0.0092 & 45.47 $\pm$ 0.42 & 0.2695 $\pm$ 0.0035 & 7.6845 $\pm$ 0.1437 \\
\bottomrule
\end{tabular}%
}
\end{table}

\begin{table}[ht]
\centering
\caption{
Tied-table diagnostics for the NanoGPT cap and scale sweep at
393.2M training tokens. Results are mean $\pm$ SD over three
matched seeds. Coordinate and row fractions describe cap activity; tail energy is computed from the row-RMS endpoint before applying the finite-cap oracle. Cosine similarity and linear-objective ratio compare the normalized AF-Muon direction with Sign evaluated on the same oracle input within each run.
}
\label{tab:app_nanogpt_cap_scale_diagnostics_step750}
\resizebox{\linewidth}{!}{%
\begin{tabular}{cccccccc}
\toprule
Cap $c$ & Scale $s$ & Matrix RMS & Coord. capped & Rows clipped & Tail energy & Cos(AF, Sign) & Obj. ratio \\
\midrule
1 & 0.5 & 0.0493 $\pm$ 0.0002 & 1.0000 $\pm$ 0.0000 & 1.0000 $\pm$ 0.0000 & 0.8005 $\pm$ 0.0002 & 1.000 $\pm$ 0.000 & 1.000 $\pm$ 0.000 \\
2 & 0.5 & 0.0501 $\pm$ 0.0001 & 0.0579 $\pm$ 0.0000 & 1.0000 $\pm$ 0.0000 & 0.2723 $\pm$ 0.0021 & 0.818 $\pm$ 0.000 & 1.246 $\pm$ 0.000 \\
3 & 0.5 & 0.0500 $\pm$ 0.0001 & 0.0028 $\pm$ 0.0001 & 0.8949 $\pm$ 0.0042 & 0.0303 $\pm$ 0.0014 & 0.799 $\pm$ 0.000 & 1.253 $\pm$ 0.000 \\
4 & 0.5 & 0.0500 $\pm$ 0.0002 & 0.0001 $\pm$ 0.0000 & 0.0484 $\pm$ 0.0050 & 0.0012 $\pm$ 0.0002 & 0.798 $\pm$ 0.000 & 1.253 $\pm$ 0.000 \\
6 & 0.5 & 0.0500 $\pm$ 0.0001 & 0.0000 $\pm$ 0.0000 & 0.0002 $\pm$ 0.0004 & 0.0000 $\pm$ 0.0000 & 0.798 $\pm$ 0.000 & 1.253 $\pm$ 0.000 \\
10 & 0.5 & 0.0499 $\pm$ 0.0001 & 0.0000 $\pm$ 0.0000 & 0.0000 $\pm$ 0.0000 & 0.0000 $\pm$ 0.0000 & 0.798 $\pm$ 0.000 & 1.253 $\pm$ 0.000 \\
$\infty$ & 0.5 & 0.0500 $\pm$ 0.0000 & 0.0000 $\pm$ 0.0000 & 0.0000 $\pm$ 0.0000 & 0.0000 $\pm$ 0.0000 & 0.798 $\pm$ 0.000 & 1.253 $\pm$ 0.000 \\
\midrule
1 & 1 & 0.0557 $\pm$ 0.0003 & 1.0000 $\pm$ 0.0000 & 1.0000 $\pm$ 0.0000 & 0.7988 $\pm$ 0.0009 & 1.000 $\pm$ 0.000 & 1.000 $\pm$ 0.000 \\
2 & 1 & 0.0589 $\pm$ 0.0002 & 0.0588 $\pm$ 0.0003 & 1.0000 $\pm$ 0.0000 & 0.3024 $\pm$ 0.0077 & 0.817 $\pm$ 0.000 & 1.251 $\pm$ 0.002 \\
3 & 1 & 0.0594 $\pm$ 0.0007 & 0.0032 $\pm$ 0.0003 & 0.9195 $\pm$ 0.0180 & 0.0450 $\pm$ 0.0133 & 0.799 $\pm$ 0.001 & 1.255 $\pm$ 0.002 \\
4 & 1 & 0.0594 $\pm$ 0.0002 & 0.0004 $\pm$ 0.0001 & 0.2476 $\pm$ 0.0625 & 0.0196 $\pm$ 0.0093 & 0.797 $\pm$ 0.001 & 1.257 $\pm$ 0.002 \\
6 & 1 & 0.0593 $\pm$ 0.0002 & 0.0000 $\pm$ 0.0000 & 0.0238 $\pm$ 0.0384 & 0.0016 $\pm$ 0.0027 & 0.797 $\pm$ 0.001 & 1.255 $\pm$ 0.002 \\
10 & 1 & 0.0593 $\pm$ 0.0003 & 0.0000 $\pm$ 0.0000 & 0.0000 $\pm$ 0.0000 & 0.0000 $\pm$ 0.0000 & 0.797 $\pm$ 0.000 & 1.254 $\pm$ 0.000 \\
$\infty$ & 1 & 0.0596 $\pm$ 0.0002 & 0.0000 $\pm$ 0.0000 & 0.0000 $\pm$ 0.0000 & 0.0000 $\pm$ 0.0000 & 0.797 $\pm$ 0.000 & 1.254 $\pm$ 0.000 \\
\bottomrule
\end{tabular}%
}
\end{table}
%%%%%%%%%%%%%%%%%%%%

\paragraph{Update scale has the larger effect in this sweep.}
Reducing $s$ from $1$ to $0.5$ lowers mean validation loss for every tested cap.
At $s=0.5$, the mean losses across caps lie in a narrow band. Thus, calibrating
the tied-table update magnitude has a larger effect than selecting the precise
cap within the tested grid. In particular, comparisons between the default
$(c,s)=(3,0.5)$ and the standard Sign endpoint $(1,1)$ reflect changes in both
geometry and scale.

\paragraph{The finite-cap default remains competitive at matched scale.}
The default $(c,s)=(3,0.5)$ achieves the lowest mean validation loss in this
intermediate sweep: $3.7815\pm0.0026$, compared with $3.7831\pm0.0033$ for
scale-matched Sign and $3.7844\pm0.0010$ for scale-matched row-RMS. These margins
are small, so we do not interpret the observed ordering as establishing
statistical superiority. At $s=1$, Sign has the lowest mean loss, further
indicating that no cap is uniformly best across tested scales.

\paragraph{Selective coordinate control motivates the default.}
At the same 393.2M-token audit, $c=3,s=0.5$ caps only
$0.0028\pm0.0001$ of coordinates while affecting
$0.8949\pm0.0042$ of vocabulary rows. The cap therefore acts on a small fraction
of coordinates but across most rows. Its normalized direction attains a
same-state linear-objective ratio of $1.253\pm0.000$ relative to Sign. By
construction, the normalized direction satisfies $\|u\|_\infty\le c$; for rows
where the row constraint is active, this also bounds coordinate concentration.
These properties distinguish the finite-cap update from Sign, which assigns
equal magnitude to nonzero coordinates, and from row-RMS, which lacks the
additional coordinate constraint.

\paragraph{Row-RMS tail-energy diagnostic.}
To characterize the coordinate concentration encountered by the cap, we compute
the stabilized row-RMS endpoint
\[
R_i(B)=
\frac{B_i}{\max\{\operatorname{RMS}(B_i),\varepsilon\}}
\]
and define its tail-energy fraction at threshold $c$ as
\[
E_{\mathrm{tail}}(B;c)
=
\frac{
\sum_{i,j} R_{ij}(B)^2
\mathbf{1}\{|R_{ij}(B)|>c\}
}{
\sum_{i,j} R_{ij}(B)^2
},
\]
with value zero when the denominator is zero. This diagnostic measures the fraction of squared row-RMS endpoint energy carried by coordinates exceeding the proposed cap. It is evaluated before applying the finite-cap oracle and does not measure energy removed, linear-objective retention, or validation
performance.
At threshold $c=1$, approximately $80\%$ of the row-RMS endpoint energy lies in coordinates exceeding the threshold. This describes the pre-oracle row-RMS direction; the completed Sign direction has equal magnitudes on its nonzero coordinates. For the default $(c,s)=(3,0.5)$,
$(3.03\pm0.14)\%$ of the squared row-RMS endpoint energy
lies in coordinates whose absolute values exceed $3$.

Together, these results support $c=3,s=0.5$ as a practical default with selective coordinate control and competitive validation performance. They also limit the claim: the sweep supports the combined scale-calibrated AF-Muon recipe, but does not establish a universal performance advantage for the finite
cap over appropriately scaled endpoints.
%-------------------- Decomposition ablation
% \newpage
\subsection{Sparse-input and Dense-output Gradient Sources}
\label{app:nanogpt_tied_source_ablation}

A tied input/output vocabulary table receives two structurally different
gradient sources: sparse input-lookup gradients from the embedding use, and
dense output-classifier gradients from the language-modeling head. We isolate these sources by running AF-Muon on the same NanoGPT-style tied-embedding model while replacing the tied-table gradient before the optimizer step. The three variants use only the sparse input component, only the dense output component, or the normal sum of both components. All other parameters, including the Transformer matrices and one-dimensional auxiliary parameters, are trained with the standard AF-Muon recipe. All runs use seed 43, FP32 trainable parameters and optimizer state with BF16 autocast, matrix learning rate $0.02$, auxiliary/vector learning rate $3\times10^{-4}$, matrix weight decay $0.1$, zero auxiliary/tied weight decay, gradient clipping at norm $1.0$, batch size 512, micro-batch size 16, and 524,288 tokens per optimizer update. We compare all variants at 458.8M training tokens.

\begin{table}[ht]
\centering
\caption{
NanoGPT tied-table gradient-source ablation at a matched 458.8M-token
comparison point. Only the gradient used for the tied input/output vocabulary table is changed; all other model parameters use the same AF-Muon update.
}
\label{tab:app_nanogpt_tied_source_ablation}
\resizebox{\linewidth}{!}{%
\begin{tabular}{lcccccc}
\toprule
Tied-table gradient & Train loss & Val. loss & Val. PPL & Input rows & Output rows & Input-output cos. \\
\midrule
Sparse input only & 4.0353 & 4.0513 & 57.47 & 0.601 & -- & -- \\
Dense output only & 3.7471 & 3.7634 & 43.10 & -- & 1.000 & -- \\
Sparse input + dense output & \textbf{3.7217} & \textbf{3.7395} & \textbf{42.08} & 0.601 & 1.000 & 0.003 \\
\bottomrule
\end{tabular}%
}
\end{table}

The dense output-classifier gradient is the dominant tied-table signal:
removing it and using only sparse input-lookup gradients substantially worsens validation loss. However, the sparse input signal is not redundant. Using only the dense output component reaches validation loss $3.7634$, while the full tied gradient improves to $3.7395$. The input/output gradient cosine is near zero, showing that the sparse lookup component supplies a distinct direction rather than a scaled copy of the dense classifier gradient.

\begin{figure}[ht]
    \centering
    \includegraphics[width=\linewidth]{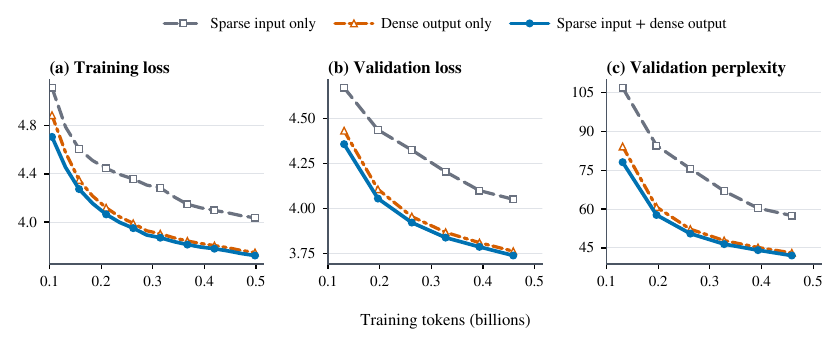}
    \caption{
    Training dynamics for the NanoGPT tied-table gradient-source ablation. Curves show seed-43 runs after the first 100M training tokens. All variants use the same AF-Muon optimizer for non-tied parameters; only the gradient supplied to the tied input/output vocabulary table is changed. The dense output-only variant learns much faster than the sparse input-only variant, but the full sparse-input plus dense-output tied gradient remains best.
    }
    \label{fig:app_nanogpt_tied_source_ablation_curves}
\end{figure}

Figure~\ref{fig:app_nanogpt_tied_source_ablation_curves} shows that the advantage of the full tied gradient persists throughout training: the sparse input-only variant remains far behind, while the full sparse-plus-dense update consistently improves over the dense output-only update.
%--------------------
% \newpage
\subsection{Boundary case: homogeneous dense-parameter reuse}
\label{app:homogeneous-dense-sharing}

AF-Muon is motivated by a parameter tensor that is shared across
\emph{heterogeneous} functional roles. This raises a natural boundary question:
does every form of parameter sharing require an intersection geometry, or is
ordinary Muon already sufficient when the shared uses have the same operator
type?

To examine this distinction, we construct a small diagnostic in which the same
MLP up- and down-projection matrices are reused at four successive depths. This
setting produces several gradient contributions for one stored parameter, but
unlike a tied embedding/output table, every use remains a dense linear map with
the same spectral geometry.

\paragraph{Shared-use gradient decomposition.}

Let a matrix \(W\) be used at \(L\) different depths. Introduce conceptual
copies \(W^{(1)},\ldots,W^{(L)}\) and write the loss before imposing parameter
sharing as
\[
\widetilde{\mathcal L}
\left(
W^{(1)},\ldots,W^{(L)}
\right).
\]
The shared model is obtained by imposing
\[
W^{(1)}=\cdots=W^{(L)}=W.
\]
The gradient on the stored shared parameter is therefore
\begin{equation}
\label{eq:shared-dense-gradient-sum}
G
:=
\nabla_W\mathcal L
=
\sum_{\ell=1}^{L}G_\ell,
\qquad
G_\ell
:=
\left.
\nabla_{W^{(\ell)}}\widetilde{\mathcal L}
\right|_{W^{(1)}=\cdots=W^{(L)}=W}.
\end{equation}
Thus, parameter sharing causes autograd to accumulate the use-specific
gradients \(G_\ell\) into one stored gradient \(G\).

The important distinction is that the individual \(G_\ell\) may differ because
each depth sees different activations and residual states, but all uses of
\(W\) still represent the same type of dense linear operator. Consequently,
they induce the same spectral constraint.

\paragraph{Why the intersection collapses.}

Let the feasible update set for each dense use be the same spectral ball
\[
\mathcal C_{\mathrm{sp}}
=
\left\{
U:\|U\|_{\mathrm{op}}\le r
\right\}.
\]
If every occurrence of the shared matrix has this geometry, then intersecting
the role constraints produces no new feasible set:
\begin{equation}
\label{eq:homogeneous-intersection-collapse}
\bigcap_{\ell=1}^{L}\mathcal C_{\mathrm{sp}}
=
\mathcal C_{\mathrm{sp}}.
\end{equation}
Moreover, because one update \(U\) is applied to every occurrence of \(W\), its
first-order effect is
\[
\sum_{\ell=1}^{L}\langle G_\ell,U\rangle
=
\left\langle
\sum_{\ell=1}^{L}G_\ell,U
\right\rangle
=
\langle G,U\rangle.
\]
The correct shared-parameter LMO is therefore
\begin{equation}
\label{eq:shared-dense-lmo}
D_{\mathrm{sp}}(G)
\in
\arg\max_{U\in\mathcal C_{\mathrm{sp}}}
\left\langle
\sum_{\ell=1}^{L}G_\ell,U
\right\rangle.
\end{equation}
In the exact spectral case, \(D_{\mathrm{sp}}(G)\) is the polar factor of the
accumulated gradient \(G\). In the implementation, it is approximated by the
same Newton--Schulz map used by Muon.
% This is different from independently computing one Muon direction per use and
% then summing:
% \[
% V_{\mathrm{raw}}
% =
% \sum_{\ell=1}^{L}D_{\mathrm{sp}}(G_\ell).
% \]
% Because the LMO is nonlinear,
% \[
% D_{\mathrm{sp}}
% \left(
% \sum_{\ell=1}^{L}G_\ell
% \right)
% \neq
% \sum_{\ell=1}^{L}D_{\mathrm{sp}}(G_\ell)
% \]
% in general.

This is different from independently computing one Muon direction per use and
then summing, $V_{\mathrm{raw}}=\sum_{\ell=1}^{L}D_{\mathrm{sp}}(G_\ell)$.
Because the LMO is nonlinear,
$D_{\mathrm{sp}}\!\left(\sum_{\ell=1}^{L}G_\ell\right)
\neq \sum_{\ell=1}^{L}D_{\mathrm{sp}}(G_\ell)$ in general.
Furthermore, \(V_{\mathrm{raw}}\) generally lies outside the
single shared spectral ball because its operator norm can grow with \(L\).
Consequently, the raw difference between these two expressions is not a fair
same-radius comparison.

We instead rescale the summed per-use direction to the spectral radius of the
aggregate Muon direction:
\begin{equation}
\label{eq:equal-radius-rolewise-muon}
\widetilde V
=
\frac{
\left\|D_{\mathrm{sp}}(G)\right\|_{\mathrm{op}}
}{
\left\|V_{\mathrm{raw}}\right\|_{\mathrm{op}}
}
V_{\mathrm{raw}}.
\end{equation}
We then measure its equal-radius linear-objective retention:
\begin{equation}
\label{eq:shared-dense-objective-retention}
R_{\mathrm{obj}}
=
\frac{
\langle G,\widetilde V\rangle
}{
\langle G,D_{\mathrm{sp}}(G)\rangle
}.
\end{equation}
For an exact LMO and exact operator norms, \(R_{\mathrm{obj}}\le 1\). A value
below one measures the linearized objective lost by applying the nonlinear
spectral map separately to each use before combining the directions.

\paragraph{Diagnostic setup.}

We use a four-depth causal diagnostic model with hidden size \(512\),
feed-forward size \(2048\), and a vocabulary of \(49{,}152\) tokens. The same
MLP up-projection and down-projection matrices are reused at every depth, while
the normalization parameters remain depth-specific. The input embedding,
positional embedding, and output head are untied so that the experiment does
not contain the tied-vocabulary mechanism studied in the main paper.

The model contains \(52.6\)M parameters and is trained on \(5.0\)M predicted
tokens from FineWeb. We use FP32 parameters and optimizer state, no autocast,
Muon learning rate \(0.02\), auxiliary learning rate \(3\times10^{-4}\),
matrix weight decay \(0.1\), and auxiliary AdamW weight decay \(0.01\). Only
the shared dense matrices are relevant to this diagnostic. Because Hybrid Muon
and AF-Muon use the same Muon branch for these matrices, this is a diagnostic
of the common hidden-matrix update rather than a comparison between complete
optimizers.

For the final audit batch, we replace the shared matrix by \(L\) equal-valued
differentiable copies and compute the use-specific gradients
\(G_1,\ldots,G_L\). We validate the decomposition by comparing
\(\sum_\ell G_\ell\) with the gradient obtained from the genuinely shared
forward pass.

We report the pairwise depth-gradient cosine
\[
C_{\ell k}
=
\frac{
\langle G_\ell,G_k\rangle
}{
\|G_\ell\|_F\|G_k\|_F
},
\]
the coefficient of variation of the depth-gradient norms, and the cancellation
ratio
\begin{equation}
\label{eq:shared-dense-cancellation}
\kappa
=
\frac{
\left\|\sum_{\ell=1}^{L}G_\ell\right\|_F
}{
\sum_{\ell=1}^{L}\|G_\ell\|_F
}.
\end{equation}
Here \(\kappa=1\) indicates perfectly aligned contributions, whereas smaller
values indicate increasing cancellation or directional diversity.

\begin{table}[t]
\centering
\small
\caption{
\textbf{Homogeneous dense-sharing diagnostic.}
One up-projection and one down-projection are reused at four depths.
Statistics are measured at the final audit after \(5.0\)M predicted tokens.
Objective retention compares the equal-spectral-radius sum of per-depth Muon
directions with Muon applied once to the accumulated gradient. The experiment
uses one seed and is intended as a boundary diagnostic rather than a
performance comparison.
}
\label{tab:shared_dense_boundary}
\resizebox{\linewidth}{!}{
\begin{tabular}{lcccccccc}
\toprule
Matrix
& Mean pair cosine
& Min. pair cosine
& Negative pairs
& Norm CV
& Cancellation \(\kappa\)
& Direction cosine
& Objective retained
& Split error \\
\midrule
Up projection
& \(0.3396\)
& \(0.0507\)
& \(0.0\%\)
& \(0.3247\)
& \(0.6803\)
& \(0.6785\)
& \(74.08\%\)
& \(5.22\times10^{-8}\) \\
Down projection
& \(0.5014\)
& \(0.2251\)
& \(0.0\%\)
& \(0.2800\)
& \(0.7697\)
& \(0.7764\)
& \(83.85\%\)
& \(5.11\times10^{-8}\) \\
\bottomrule
\end{tabular}}
\end{table}

\paragraph{Results.}

The split-gradient reconstruction is numerically exact to floating-point
precision. The shared and split computations have identical loss and logits,
and the relative gradient reconstruction errors are
\(5.22\times10^{-8}\) and \(5.11\times10^{-8}\) for the up and down matrices,
respectively. No nonfinite values are observed.

The use-specific gradients are not identical. Their mean pairwise cosines are
\(0.3396\) for the up projection and \(0.5014\) for the down projection, and
their norm coefficients of variation are \(0.3247\) and \(0.2800\).
Nevertheless, every pairwise cosine remains positive: the minimum cosines are
\(0.0507\) and \(0.2251\), and the fraction of negative depth pairs is zero for
both matrices. The cancellation ratios are \(0.6803\) and \(0.7697\), indicating
moderate directional diversity but no severe destructive conflict.

Muon applied to the accumulated gradient differs substantially from the raw sum
of per-depth Muon directions, with direction cosines \(0.6785\) and \(0.7764\).
This difference is expected from the nonlinearity of the spectral LMO and is
not evidence of an optimizer failure. After matching spectral radius, the
per-depth-then-sum construction retains only \(74.08\%\) of the aggregate
Muon linear objective for the up projection and \(83.85\%\) for the down
projection. Thus, applying Muon once to the accumulated gradient is materially
better aligned with the shared parameter's first-order objective than applying
Muon independently to each occurrence and subsequently combining the
directions.

Training itself remains stable throughout the diagnostic. Validation loss
decreases from approximately \(10.91\) at initialization to \(6.35\) after
\(5.0\)M tokens, the final reported gradient norm is \(0.71\), and no
nonfinite parameters, gradients, or optimizer states are observed. We do not
interpret the absolute loss as a competitive language-modeling result because
the model is an intentionally restricted shared-MLP diagnostic without
attention or cross-token mixing.

\paragraph{Boundary with tied-table sharing.}

This experiment identifies an important boundary for role-aware optimizer
design. Repeated parameter use is not by itself sufficient to require a new
LMO. In the shared dense-layer setting, different use sites generate different
gradients, but every use induces the same spectral update geometry. The
intersection of their feasible sets therefore collapses to the original
spectral ball, and the standard Muon rule applied to the accumulated gradient
is the appropriate shared-parameter LMO.

The tied embedding/output table studied in the main paper lies on the other
side of this boundary. Its input uses act as vocabulary-indexed lookup rows,
whereas its output use acts as a dense vocabulary classifier. These roles
motivate distinct row-wise and entrywise controls:
\[
\mathcal C_{\mathrm{tie},c}
=
\left\{
U:
\max_i\frac{\|u_i\|_2}{\sqrt d}\le 1,\quad
\max_{i,j}\frac{|U_{ij}|}{c}\le 1
\right\}.
\]
Their intersection does not collapse and produces the finite-cap clipped-row
LMO. The relevant distinction is therefore not merely whether a tensor is
aliased, but whether its aliases impose heterogeneous update geometries:
\[
\boxed{
\text{homogeneous reuse}
\;\Longrightarrow\;
\text{accumulate gradients and apply the existing LMO},
}
\]
whereas
\[
\boxed{
\text{heterogeneous role sharing}
\;\Longrightarrow\;
\text{construct a joint feasible geometry}.
}
\]

The dense-sharing diagnostic therefore serves as a negative boundary control:
it shows why AF-Muon's new tied-table geometry is specialized to heterogeneous
lookup/classifier aliasing rather than being a generic modification for every
shared parameter.

\end{document}